\documentclass[11pt]{article}

\usepackage[T1]{fontenc}
\usepackage[utf8]{inputenc}
\usepackage{amsmath}        
\usepackage{amsthm}         
\IfFileExists{newtxmath.sty}{
  \IfFileExists{newpxtext.sty}{\usepackage{newpxtext}}{}
  \usepackage[stix2,bigdelims,timesmathacc]{newtxmath}  
  \DeclareMathAlphabet{\mathcal}{OMS}{cmsy}{m}{n}       
}{
  \usepackage{lmodern}
  \usepackage{amssymb,amsfonts}
}
\usepackage{iftex}          
\ifPDFTeX                   
  \usepackage[activate={true,nocompatibility},final]{microtype}
\else                       
  \usepackage[protrusion=true,final]{microtype}
\fi

\usepackage[margin=1in]{geometry}

\usepackage{mathrsfs}       
\usepackage{mathtools}      
\usepackage{bm}             
\usepackage{nicefrac}
\allowdisplaybreaks         

\usepackage{thm-restate}    

\usepackage{graphicx}
\graphicspath{{figures/}{./}}                       
\usepackage{array}                      
\usepackage{booktabs}                               
\usepackage{multirow,makecell,array}
\usepackage[font=small,labelfont=bf]{caption}       
\usepackage{subcaption}                             
\usepackage{algorithm}
\usepackage{algpseudocode}                          

\usepackage{enumitem}
\setlist[itemize]{leftmargin=2.2em,itemsep=2pt,topsep=2pt}
\setlist[enumerate]{leftmargin=2.2em,itemsep=2pt,topsep=2pt}

\usepackage{xcolor}
\usepackage{colortbl}                               
\definecolor{LinkColor}{rgb}{0.10,0.40,0.75}        
\definecolor{CiteColor}{rgb}{0.70,0.25,0.20}        
\definecolor{UrlColor} {rgb}{0.20,0.50,0.50}        
\definecolor{TodoColor}{rgb}{0.80,0.30,0.10}        
\definecolor{TabHighlight}{rgb}{0.85,0.93,1.00}     

\usepackage{tikz}
\usetikzlibrary{positioning,calc,arrows.meta}

\usepackage{url}
\usepackage{hyperref}
\hypersetup{
  colorlinks=true,
  linkcolor=LinkColor,
  citecolor=CiteColor,
  urlcolor=UrlColor,
  breaklinks=true,
  bookmarksnumbered=true,
}
\usepackage{bookmark}                               
\usepackage[capitalise,nameinlink,noabbrev,sort&compress]{cleveref}  

\numberwithin{equation}{section}

\newif\ifdraft \draftfalse
\ifdraft
  \usepackage{lineno}\linenumbers
  \newcommand{\todo}[1]{\textcolor{TodoColor}{\textbf{[TODO:}~#1\textbf{]}}}
\else
  \newcommand{\todo}[1]{}                            
\fi

\newcommand{\N}{\mathbb{N}}
\newcommand{\Z}{\mathbb{Z}}

\newcommand{\E}{\mathbb{E}}                       
\renewcommand{\P}{\mathbb{P}}                     

\newcommand{\eps}{\varepsilon}

\newcommand{\defeq}{\coloneqq}                    

\DeclarePairedDelimiterX{\inner}[2]{\langle}{\rangle}{#1,#2}   

\newcommand{\cX}{\mathcal{X}}
\newcommand{\cH}{\mathcal{H}}
\newcommand{\cF}{\mathcal{F}}
\newcommand{\cZ}{\mathcal{Z}}
\newcommand{\cD}{\mathcal{D}}
\newcommand{\dN}{d_{\mathrm N}}
\newcommand{\dDS}{d_{\mathrm{DS}}}
\newcommand{\dNk}[1]{d_{\mathrm N}^{#1}}
\newcommand{\dDSk}[1]{d_{\mathrm{DS}}^{#1}}
\newcommand{\ListRisk}{L^{\mathrm{list}}}
\newcommand{\Lstarr}{\Lstar_{r}}
\newcommand{\dRE}{d_{\mathrm{RE}}}
\newcommand{\Risk}{L}
\newcommand{\Lstar}{L^\star}
\newcommand{\Rhat}{\widehat L}
\newcommand{\ind}{\mathbf{1}}
\newcommand{\MenuLoss}{L_{\mu}}
\newcommand{\RMC}{\mathsf{RMC}}
\newcommand{\MW}{\mathsf{MW}}
\newcommand{\LSCS}{\mathsf{LSCS}}
\newcommand{\CompCover}{\mathsf{Cover}}
\newcommand{\cW}{\mathcal{W}}
\newcommand{\cQ}{\mathcal{Q}}
\newcommand{\cR}{\mathcal{R}}
\newcommand{\cU}{\mathcal{U}}
\DeclareMathOperator{\Nat}{Nat}
\DeclareMathOperator{\DS}{DS}
\DeclareMathOperator{\TV}{TV}
\DeclareMathOperator{\kl}{kl}

\theoremstyle{plain}
\newtheorem{theorem}{Theorem}[section]

\newtheorem{lemma}[theorem]{Lemma}
\newtheorem{corollary}[theorem]{Corollary}

\theoremstyle{definition}
\newtheorem{definition}[theorem]{Definition}
\newtheorem{assumption}[theorem]{Assumption}
\newtheorem{example}[theorem]{Example}

\theoremstyle{remark}
\newtheorem{remark}[theorem]{Remark}

\crefname{theorem}{Theorem}{Theorems}          \Crefname{theorem}{Theorem}{Theorems}
\crefname{proposition}{Proposition}{Propositions}
\Crefname{proposition}{Proposition}{Propositions}
\crefname{lemma}{Lemma}{Lemmas}                \Crefname{lemma}{Lemma}{Lemmas}
\crefname{corollary}{Corollary}{Corollaries}   \Crefname{corollary}{Corollary}{Corollaries}
\crefname{conjecture}{Conjecture}{Conjectures} \Crefname{conjecture}{Conjecture}{Conjectures}
\crefname{fact}{Fact}{Facts}                   \Crefname{fact}{Fact}{Facts}
\crefname{definition}{Definition}{Definitions} \Crefname{definition}{Definition}{Definitions}
\crefname{assumption}{Assumption}{Assumptions} \Crefname{assumption}{Assumption}{Assumptions}
\crefname{example}{Example}{Examples}          \Crefname{example}{Example}{Examples}
\crefname{problem}{Problem}{Problems}          \Crefname{problem}{Problem}{Problems}
\crefname{remark}{Remark}{Remarks}             \Crefname{remark}{Remark}{Remarks}
\crefname{claim}{Claim}{Claims}                \Crefname{claim}{Claim}{Claims}
\crefname{algorithm}{Algorithm}{Algorithms}    \Crefname{algorithm}{Algorithm}{Algorithms}

\title{Optimistic Rates for Multiclass PAC Learning}
\author{%
  Xiaoyu Li\textsuperscript{1}\hspace{2.2em}
  Andi Han\textsuperscript{2}\hspace{2.2em}
  Jiaojiao Jiang\textsuperscript{1}\hspace{2.2em}
  Junbin Gao\textsuperscript{2}
  \\[0.9em]
  \small
  \begin{tabular}{@{}l@{\hspace{2.2em}}l@{}}
    \textsuperscript{1}University of New South Wales
      & \texttt{\{xiaoyu.li2,jiaojiao.jiang\}@unsw.edu.au}\\
    \textsuperscript{2}University of Sydney
      & \texttt{\{andi.han,junbin.gao\}@sydney.edu.au}
  \end{tabular}%
}
\date{}

\hypersetup{
  pdftitle={Optimistic Rates for Multiclass PAC Learning},
  pdfauthor={Xiaoyu Li, Andi Han, Jiaojiao Jiang, Junbin Gao}
}

\begin{document}
\maketitle

\begin{abstract}
Worst-case multiclass bounds do not become smaller when the best classifier
is already nearly correct: what is missing is an optimistic rate, a
guarantee whose fluctuation scales with the oracle risk itself.  For a
class of Natarajan dimension \(\dN\) and
Daniely--Shalev-Shwartz dimension \(\dDS\), the optimal excess risk is
known at the two endpoints (\(\dDS/n\) realizable,
\(\sqrt{\dN/n}+\dDS/n\) agnostic
\cite{hanneke2024improved,cohen2025natarajan,pabbaraju2026optimal}) and
open in between.  We close the gap: at every fixed oracle risk
\(\Lstar\), the optimal excess risk is
\[
 \widetilde\Theta\!\left(
 \sqrt{\frac{\Lstar\dN}{n}}+\frac{\dDS}{n}
 \right),
\]
uniformly in the alphabet size, attained by a learner that knows neither
\(\Lstar\) nor the confidence level.%
\renewcommand{\thefootnote}{\fnsymbol{footnote}}%
\footnote{The main theorems are
machine-checked in Lean~4; \cref{app:lean} records what is verified and
in which form, and the development is available at
\url{https://github.com/xiaoyulics/multiclass-pac-learning}.}%
\setcounter{footnote}{0}%
The upper bound composes the cover--menu--compression architecture of
\cite{cohen2025natarajan}, at the realizable rate of
\cite{pabbaraju2026optimal}, with a new comparator-facing relative
compression theorem: a size-\(k\) compression rule that empirically
dominates a comparator \(h\) has population risk at most
\(L(h)+O(\sqrt{L(h)\Gamma}+\Gamma)\) with
\(\Gamma=(k\log n+\log(1/\delta))/n\), without stability; this transfers
the comparison principle of the sharp binary theory
\cite{mathiasen2026optimal} while discarding its Boolean-cube geometry,
which does not lift to multiclass labels.  The
lower bound forces both terms using one class and one distribution at
every fixed \(\Lstar\), by a pair-Assouad scheme calibrated to \(\Lstar\)
and a fiber argument on the pseudo-cubes underlying the
Natarajan-versus-DS separation of \cite{brukhim2022characterization}.
Both theorems extend to list learning: against the best \(r\)-tuple of
hypotheses, the same architecture and the same two engines yield an
optimistic rate and a lower bound of the same shape, forcing the
fluctuation term that \cite{pabbaraju2026optimal} expected to be
necessary against list comparators, and removing the factor \(r\) from
the known realizable list lower bound.

\end{abstract}

\clearpage
{\hypersetup{linkcolor=black}\tableofcontents}
\clearpage

\section{Introduction}
\label{sec:intro}

The statistical difficulty of classification should depend on how much error
is left in the best available hypothesis.  When that hypothesis is nearly
correct, the learner is not facing the same problem as in the fully agnostic
regime: most observations carry no disagreement with the comparator, and the
fluctuation term ought to reflect this.  Guarantees with this
property---interpolating between the realizable fast rate and the agnostic
slow rate according to the level of realizability---are called
\emph{optimistic rates} \cite{panchenko2002some,srebro2010optimistic}.  The requirement is
easiest to see at the boundary.  A class with VC dimension \(d\) admits the
fast rate \(d/n\) when it is exactly realizable, and a class whose best
hypothesis errs on one observation in a trillion is statistically
indistinguishable from a realizable one at any practical sample size; yet
the worst-case agnostic bound charges the same slow rate \(\sqrt{d/n}\)
for both \(\Lstar=10^{-12}\) and \(\Lstar=\tfrac13\).  An optimistic
rate must instead degrade continuously: at \(\Lstar=10^{-12}\), the
localized fluctuation \(\sqrt{\Lstar d/n}\) remains within a factor of
two of the realizable rate for every \(n\le10^{12}d\).  In binary classification the desired
statement is now precise.  If the class has VC dimension \(d\), the
distribution-free fixed-risk benchmark is
\[
 \Risk(\widehat h)-\Lstar
 \asymp
 \sqrt{\frac{\Lstar(d+\log(1/\delta))}{n}}
 +\frac{d+\log(1/\delta)}{n}.
\tag{1.1}\label{eq:binary-target}
\]
The first term is the agnostic fluctuation at the actual oracle risk; the
second is the realizable remainder; lower bounds of this localized form are
classical \cite{devroye1995lower}.  \cite{hanneke2024revisiting} showed
that empirical risk minimization, and any other proper learner, misses the
benchmark by a factor that grows as \(\Lstar\) decreases, and gave a
learner attaining it over most of the range of \(\Lstar\).
\cite{asilis2025smallerror} determined the fine structure of the remaining
small-error regime, including lower bounds that exhibit the localized term
as unavoidable.  \cite{mathiasen2026optimal} then constructed a
parameter-free learner attaining \eqref{eq:binary-target} up to universal
constants.  Their
proof uses a localized orientation of the Boolean cube and, crucially, a
high-probability comparison that preserves coefficient one on the oracle
risk.  All of this line is binary; it is the multiclass version of the
question that concerns us here.

The multiclass problem is not obtained by changing the alphabet in this
argument.  It has two capacity parameters with different jobs.  The
Daniely--Shalev-Shwartz dimension \(\dDS\) \cite{daniely2014optimal}
characterizes realizable learnability \cite{brukhim2022characterization},
whereas the Natarajan dimension \(\dN\) \cite{natarajan1989learning}
controls the leading agnostic term.  The cover--menu--compression reduction of
\cite{cohen2025natarajan} established the two-parameter structure with a
\(\dDS^{3/2}\) coverage term; \cite{pabbaraju2026optimal} then resolved the
realizable sample complexity at \(\dDS\), and Corollary~1.2 of that paper
records the resulting worst-case rate
\[
 \widetilde O\!\left(
 \frac{\dDS}{\eps}+
 \frac{\dN+\log(1/\delta)}{\eps^2}
 \right).
\tag{1.2}\label{eq:prior-sample}
\]
This rate has the correct two-dimensional structure, but it treats an oracle
of risk \(10^{-3}\) in the same way as an oracle of constant risk.  The
question of this paper is whether the two structures can be kept at once:
can the Natarajan term localize at \(\Lstar\), while the DS term remains a
realizable-order remainder?

\paragraph{Why the literal transfer fails.}
The first attempt is to replace the Boolean cube by a \(K\)-ary Hamming graph
and repeat the binary orientation argument.  This construction records
binary error patterns relative to a reference labeling, so its natural
capacity is Graph dimension rather than DS or Natarajan dimension
\cite{natarajan1989learning,daniely2015multiclass}.  For a
finite alphabet this route introduces an unwanted alphabet dependence; for a
general label space Graph dimension may even be infinite when DS dimension
is finite \cite{daniely2015multiclass,brukhim2022characterization}.  This
is the obstruction behind the failure of direct ERM-style
multiclass arguments \cite{cohen2025natarajan}.  Multiclass one-inclusion
orientations remain fundamental---see, for example,
\cite{asilis2024regularization}---but the Boolean-cube geometry is not the
part of the binary proof that can be carried over unchanged.

\paragraph{The result.}
We instead obtain, for finite label sets, a learner \(\RMC\) achieving the
two-dimensional optimistic rate
\[
 \Risk(\RMC(S))-\Lstar
 =
 \widetilde O\!\left(
 \sqrt{\frac{\Lstar\dN}{n}}+\frac{\dDS}{n}
 \right).
\tag{1.3}\label{eq:intro-rate}
\]
The learner uses neither \(\Lstar\) nor \(\delta\).  The fully quantified
statement in \cref{thm:main} contains a \(\dN\log^3(en)\) localized term and
a \(\dDS\log^2(en)\) coverage term.  The rate is optimal:
\cref{thm:lower} constructs, for every pair \((\dN,\dDS)\), a class on
which every learner, at every fixed \(\Lstar\), incurs excess risk
\(\Omega(\sqrt{\Lstar\dN/n}+\dDS/n)\), with both terms forced by one
distribution.  Upper and lower bounds together give the multiclass
analogue of \eqref{eq:binary-target}, up to logarithmic factors.  The
comparison is summarized in \cref{tab:rates}; confidence terms and
polylogarithms are suppressed there.

\begin{table}[t]
  \centering
  \caption{Risk bounds at a glance, with upper- and lower-bound sources
  separated.  The binary row has \(d=\mathrm{VC}(\cH)\); the multiclass
  rows use \(\dN\) and \(\dDS\).  In the highlighted row both directions
  are proved in this paper, and they match up to logarithmic factors.}
  \label{tab:rates}
  \begin{tabular}{@{}llll@{}}
    \toprule
    Setting & Excess-risk scale & Upper bound & Lower bound \\
    \midrule
    Binary
      & \(\sqrt{\Lstar d/n}+d/n\)
      & \cite{mathiasen2026optimal}
      & \cite{devroye1995lower,asilis2025smallerror} \\
    Multiclass
      & \(\sqrt{\dN/n}+\dDS/n\)
      & \cite{cohen2025natarajan,pabbaraju2026optimal}
      & \cite{cohen2025natarajan,hanneke2024improved} \\
    \midrule
    \rowcolor{TabHighlight}
    Multiclass
      & \(\sqrt{\Lstar\dN/n}+\dDS/n\)
      & \textbf{this work} (\cref{thm:main})
      & \textbf{this work} (\cref{thm:lower}) \\
    \bottomrule
  \end{tabular}
\end{table}

\paragraph{What is transferred.}
The coefficient-one comparison principle is inspired by
\cite{mathiasen2026optimal}; the learned-menu device and correct-label
outside-menu decomposition are inherited from
\cite[Section~1.1 and Theorem~3.5]{cohen2025natarajan}.
The useful invariant from the binary proof is not its cube but its
comparison principle: keep coefficient one on the comparator risk, and make
the fluctuation scale with that same risk.  To make this principle compatible
with multiclass capacity, we first reduce the label geometry to a short
learned menu.  There are then two distinct ways to be wrong.  A correct
comparator label may be missing from the menu; this is a coverage failure and
is paid for by \(\dDS\).  Conditional on a fixed menu, the learner may choose
the wrong label inside it; this is an ordinary binary loss whose description
length is controlled by \(\dN\).  The proof is designed so that these two
charges never merge.

This separation explains both the three-way sample split and the form of the
bound.  The three-stage skeleton is adapted from
\cite[Algorithm~4]{cohen2025natarajan}.  The first block compresses the correct region of any comparator into
a finite family.  The second block learns a menu that covers the labels of a
canonical member of that family.  Only after the menu is frozen does the
third block fit an empirical-risk-dominating menu-loss compressor.  At that
point the remaining question is purely statistical: can a compression rule
that beats a comparator empirically be compared with it in population risk
without paying an absolute \(\sqrt{k/n}\) term?

The relative compression theorem in \cref{thm:relative-compression} answers
this question.  If a deterministic compression rule of size \(k\) has no
larger empirical loss than a fixed comparator \(h\), then
\[
 \Risk(A(S))
 \le \Risk(h)
 +O\!\left(
 \sqrt{\frac{\Risk(h)(k\log n+\log(1/\delta))}{n}}
 +\frac{k\log n+\log(1/\delta)}{n}
 \right).
\tag{1.4}\label{eq:intro-compression}
\]
No stability is required.  The mechanism is the classical one for
compression bounds: ordered descriptions are counted, and after their
selected coordinates are removed, one-sided Bernstein inequalities control
the untouched observations.  Bernstein refinements of exactly this type
already localize lossy compression at the reconstruction's empirical error
\cite[Theorem~8]{gottlieb2017semimetrics}.  What the composition consumes
is a different interface: comparator-facing rather than
reconstruction-facing, and order-dependent with repeated indices.  The
concentration step itself is elementary; the contribution is to formulate it
in the form that can be composed with the random menu while preserving
coefficient one.

\paragraph{What the paper adds.}
There are three mathematical pieces and one structural conclusion.  First, we
prove the comparator-facing relative compression theorem with explicit
constants, ordered messages, repeated indices, and no stability assumption;
the proof route is the counting-plus-Bernstein one of the compression
literature, and the claim of novelty is confined to the interface, not the
mechanism.  Second, we place the theorem at the final stage of the
source-checked cover--menu--compression reduction of
\cite{cohen2025natarajan}, taking care that the menu is independent
of the final sample and that empty and singleton menus introduce neither a
singular \(\log p\) nor hidden side information.  Third, we prove a
matching fixed-\(\Lstar\) lower bound (\cref{thm:lower}): a single class
and a single distribution activate both dimensions at every fixed
\(\Lstar\), so the rate of \eqref{eq:intro-rate} is optimal up to
logarithmic factors.  The construction pairs a noisy Natarajan cube with a
Natarajan-free pseudo-cube obtained from the separation of
\cite{brukhim2022characterization}.  The resulting risk
decomposition identifies the roles of the dimensions: \(\dDS\) controls
correct-label coverage, while \(\dN\) controls localized fluctuations inside
the learned menu, and the lower bound certifies that this division of
labor is intrinsic rather than an artifact of the analysis.

\paragraph{Extension to list learning.}
The architecture never uses the fact that the final block outputs a
single label, and \cref{sec:list} exploits this: for learners that
output \(r\) labels per point, the same three blocks yield an
optimistic rate against the best \(r\)-tuple of hypotheses
(\cref{thm:list-upper}), and the two lower-bound engines upgrade to a
bound of the same two-term shape with exact calibration
(\cref{thm:list-lower}), whose fluctuation term is of the kind
\cite{pabbaraju2026optimal} expected to be necessary against list
comparators.  A byproduct of the second engine removes the factor
\(r\) from the realizable list lower bound of
\cite{hanneke2024improved} (\cref{cor:list-realizable}).  The
multiclass theorems remain the paper's center; the list section is a
short demonstration that the comparison principle composes beyond
single-valued prediction.

\paragraph{Further related work.}
The Natarajan dimension is the classical multiclass analogue of the VC
dimension \cite{natarajan1989learning}; the Graph dimension governs
ERM-style uniform-convergence arguments but not learnability at large
alphabets \cite{daniely2015multiclass}.  The DS dimension entered through
the multiclass one-inclusion framework of \cite{daniely2014optimal}, and
its finiteness characterizes multiclass PAC learnability
\cite{brukhim2022characterization}.  One-inclusion prediction originates
with \cite{haussler1994predicting} and was brought to multiclass loss by
\cite{rubinstein2006shifting}; optimal realizable bounds for binary VC
classes were later obtained from it without uniform convergence
\cite{adenali2023optimal}.  In the multiclass realizable problem, the rate
was improved by \cite{hanneke2024improved} and settled by
\cite{pabbaraju2026optimal}, who bounded the one-inclusion density through
the DS dimension; a sharp \(k\)-ary Sauer--Shelah--Perles inequality in the
same currency appears in \cite{hanneke2026sauer}.

The learned menu is a list-valued object.  List prediction entered the
multiclass story as a tool in the characterization of
\cite{brukhim2022characterization}; list learnability in its own right
was characterized by \cite{charikar2023list}, whose rates were
successively sharpened through multiclass boosting
\cite{brukhim2023boosting} and the optimal \(k\)-ary Sauer lemma of
\cite{hanneke2026sauer}, and brought to near-optimal form by
\cite{pabbaraju2026optimal}, against the realizable lower bound of
\cite{hanneke2024improved}.  Prediction that is graded on only part of
the observation space is the subject of partial concept classes
\cite{alon2022partial}.  \Cref{sec:list} engages the list thread
directly: it upgrades both main theorems to list learners, with the
comparator strengthened to the best tuple of hypotheses and the
fluctuation localized at the oracle list risk, refining the worst-case
agnostic list rates of \cite{charikar2023list,pabbaraju2026optimal} in
the same sense in which \eqref{eq:main-bound} refines the worst-case
multiclass rate.

Sample compression enters PAC learning with the unpublished 1986 manuscript
of Littlestone and Warmuth; \cite{floyd1995sample} tied compression size
to the VC dimension, and \cite{david2016supervised} brought
boosting-to-compression arguments to the multiclass setting, the tradition
in which the cover stage sits.  Compression bounds at nonzero empirical
risk begin with \cite{graepel2005pac}, which also isolated the role of
message order and repetition conventions.  Two facts calibrate what
compression alone can deliver: generic agnostic schemes of size \(k\)
cannot beat the worst-case rate \(\sqrt{k\log(n/k)/n}\)
\cite{hanneke2019sharp}, and stability removes logarithmic factors where
it is available \cite{hanneke2021stable}.
\Cref{thm:relative-compression} trades neither: it keeps the logarithm and
the stability-free interface, and instead moves the anchor to a fixed
comparator's population risk with coefficient one.  The localization
obtained this way is distribution-free, entering only through \(\Lstar\)
with no Massart- or Tsybakov-type condition; it is therefore not a
fast-rate statement in the local-complexity sense of
\cite{bartlett2005local}.  It belongs instead to the optimistic-rate
tradition: relative-deviation inequalities whose fluctuation scales with
the risk itself go back to Vapnik and Chervonenkis and were extended by
\cite{panchenko2002some}; the name is from \cite{srebro2010optimistic},
which proved such rates for smooth losses; and
\cite{zhivotovskiy2018localization} developed localization for VC
classification beyond local Rademacher complexities, a precursor of the
binary fixed-risk line above.

\paragraph{Scope.}
The result is information-theoretic: the compression and one-inclusion
routines need not be computationally efficient for an abstract class.  The
label alphabet is finite, although the bound is uniform in its cardinality.
Logarithmic factors remain in the upper bound; the lower bound of
\cref{thm:lower} matches it up to those factors and does not address the
sharp \(\log(1/\delta)\) dependence.  The lower-bound family uses label
alphabets growing with \(\dDS/\dN\), which \cref{rem:alphabet} shows is
necessary, and which the uniformity in \(K\) of \cref{thm:main}
anticipates.

\paragraph{Organization.}
\Cref{sec:architecture} develops the upper-bound proof as a sequence of
obstacles and design choices, ending with the error identity that separates the two
dimensions.  \Cref{sec:results} states the relative compression theorem, the
learner, and the main risk bound.  \Cref{sec:lower} proves the matching
lower bound.  \Cref{sec:list} extends both bounds to list learners.
\Cref{sec:discussion} interprets the results.  The appendix
places each source-convention check immediately after its restated external
lemma, gives the complete concentration argument and full probability
bookkeeping, and closes with the proofs for the lower bound and for the
list extension.

\section{Setting and proof architecture}
\label{sec:architecture}

The direct multiclass lift fails because it asks one combinatorial object to
do two jobs: represent the label geometry and control the stochastic
fluctuation.  The proof below separates these jobs.  A learned menu first
turns multiclass prediction into a selective binary loss; an exact algebraic
identity then shows that only correct comparator labels missing from that
menu must be restored.  This section develops that identity before any of the
compression machinery, and then explains why each of the three sample blocks
is needed.

The cover--menu--compression backbone is inherited from the three-phase
learner of \cite[Sections~3.1--3.4]{cohen2025natarajan}.  Our modification
is at its final statistical interface: we replace the absolute compression
comparison by a comparator-facing relative one and then track how that change
propagates through the composition.

\paragraph{Notation.}
For \(m\in\N\) we write \([m]=\{1,\ldots,m\}\).  Logarithms are natural,
and \(\ind\{\cdot\}\) denotes the indicator of an event.  A \emph{sample}
is a finite iid sequence of labeled \emph{observations}; for a
distribution \(\cD\), the law of a sample of \(n\) observations is
\(\cD^{n}\).  Expectations over auxiliary finite indices are written with
subscripts, as in \(\E_\sigma\).  Universal constants \(C,c>0\) may
change value between occurrences, and \(\widetilde O(\cdot)\) suppresses
factors polylogarithmic in \(n\), \(\dN\), and \(\dDS\).

\subsection{Setting and the two capacity parameters}

We use the finite-alphabet dimension and learning conventions of
\cite[Section~2]{cohen2025natarajan}.

Let \((\cX,\mathscr X)\) be a measurable space, let
\([K]=\{1,\ldots,K\}\), and let \(\cH\subseteq[K]^\cX\).  For a distribution
\(\cD\) on \(\cZ=\cX\times[K]\), write
\[
 \Risk_{\cD}(f)=\P_{(X,Y)\sim\cD}(f(X)\ne Y),
 \qquad
 \Lstar=\inf_{h\in\cH}\Risk_{\cD}(h).
\]
Unless another distribution is named, \(\P\) and \(\E\) are taken under the
ambient \(\cD\) and their subscript is omitted.
For a finite sequence \(s=(z_1,\ldots,z_N)\), \(\Rhat_s(f)\) denotes its
average zero--one loss.  Samples are iid unless stated otherwise.

\begin{assumption}[Measurability and tie-breaking]\label{asm:measurable}
The label alphabet is finite.  Every hypothesis evaluation, learning map,
selection map, reconstruction map, and finite-class sampling rule used below
is jointly measurable after deterministic tie-breaking, or the class is
restricted by the usual countability/separability convention that ensures
this property.
\end{assumption}

This convention keeps all probability statements literal rather than phrased
in outer probability.  It is the only regularity condition beyond iid
sampling.

\begin{definition}[Natarajan dimension \cite{natarajan1989learning}]
\label{def:natarajan}
A sequence \(x_1,\ldots,x_d\) is Natarajan-shattered by \(\cH\) if there are
label pairs \(y_i\ne y_i'\) such that every one of the \(2^d\)
coordinatewise choices is realized by the trace of \(\cH\).  The largest
such \(d\) is the Natarajan dimension \(\dN\).
\end{definition}

\begin{definition}[DS dimension \cite{daniely2014optimal}]
\label{def:ds}
A nonempty finite set \(V\subseteq[K]^d\) is a \(d\)-dimensional pseudo-cube
if, for each \(v\in V\) and coordinate \(i\), there is \(v'\in V\) that
differs from \(v\) only at coordinate \(i\).  A sequence is DS-shattered if
the trace of \(\cH\) contains such a pseudo-cube; the largest length is the
Daniely--Shalev-Shwartz dimension \(\dDS\).
\end{definition}

Finiteness of \(\dDS\) characterizes multiclass PAC learnability
\cite{brukhim2022characterization}.  Every Natarajan cube is a
pseudo-cube, and therefore
\begin{equation}
 \dN\le \dDS.
\label{eq:dimensions-order}
\end{equation}
The inequality does not make the two quantities interchangeable.  In the
argument below, \(\dDS\) controls whether a correct label is represented at
all, while \(\dN\) controls the complexity of choosing among labels that have
already been represented.

Both dimensions also vanish together: a single point at which two
hypotheses disagree is already Natarajan- and DS-shattered, so \(\dDS=0\)
forces all hypotheses of \(\cH\) to coincide pointwise, in which case any
fixed member attains \(\Lstar\) and there is nothing to learn.  The
standing assumption \(1\le\dN\le\dDS<\infty\) in \cref{thm:main} therefore
excludes only this degenerate case; in particular it forces
\(\cH\ne\emptyset\).

\subsection{The menu as a selective binaryization}

Following \cite[Section~3.3]{cohen2025natarajan}, a \emph{\(p\)-menu} is a
measurable map
\(\mu:\cX\to2^{[K]}\) with \(\sup_x|\mu(x)|\le p\).  Once a menu is fixed,
define the inside-menu loss
\[
 \ell_\mu(f,(x,y))
 =\ind\{f(x)\ne y,\ y\in\mu(x)\},
 \qquad
 \MenuLoss(f)=\E\,\ell_\mu(f,(X,Y)).
\tag{2.2}\label{eq:menu-loss}
\]
This is a binary-valued selective loss: the predictor remains multiclass,
but the loss deliberately ignores observations whose label lies outside the
menu.  The apparent danger
is that ignored observations might hide arbitrary error.  The following
elementary identity shows exactly what has to be put back; it isolates the
pointwise algebra behind \cite[Theorem~3.5]{cohen2025natarajan}, the hinge
on which the localized refinement turns.

\begin{lemma}[Selective-menu decomposition]
\label{lem:menu-calculus}
Let \(\mu\) be a menu.  For all classifiers \(f\) and \(h\),
\[
 \P(h(X)=Y,\ Y\notin\mu(X))
 \le
 \P(h(X)=Y\ne f(X))
 +\P(f(X)=Y,\ Y\notin\mu(X));
\tag{2.3}\label{eq:menu-miss-transfer}
\]
for all classifiers \(g\) and \(h\),
\[
 \Risk_{\cD}(g)-\Risk_{\cD}(h)
 \le
 \MenuLoss(g)-\MenuLoss(h)
 +\P(h(X)=Y,\ Y\notin\mu(X));
\tag{2.4}\label{eq:menu-risk-algebra}
\]
and, for every classifier \(h\),
\[
 \MenuLoss(h)\le\Risk_{\cD}(h).
\tag{2.5}\label{eq:menu-loss-below-risk}
\]
\end{lemma}

\begin{proof}
For \eqref{eq:menu-miss-transfer}, consider a point on which \(h\) is correct
but its label is absent from the menu.  Either \(f\) disagrees with this
correct label, which is the first event on the right, or \(f\) is also
correct, which is the second event.

For \eqref{eq:menu-risk-algebra}, fix \((x,y)\).  If
\(y\in\mu(x)\), the difference between the zero--one losses of \(g\) and
\(h\) is exactly their inside-menu loss difference.  If
\(y\notin\mu(x)\), both inside-menu losses vanish.  The ordinary loss
difference can then be positive only when \(h(x)=y\), and in that case it is
at most one.  This is precisely the last indicator in
\eqref{eq:menu-risk-algebra}.  Taking expectations proves the inequality.
Finally, \(\ell_\mu(h,(x,y))\le\ind\{h(x)\ne y\}\) pointwise, which gives
\eqref{eq:menu-loss-below-risk}.
\end{proof}

The decomposition is the conceptual center of the proof.  Put
\[
 \alpha_\mu(h)
 \defeq \P(h(X)=Y,\ Y\notin\mu(X)).
\]
Then the desired comparison has the form
\begin{equation}
 \Risk_{\cD}(g)-\Risk_{\cD}(h)
 \le
 \bigl(\MenuLoss(g)-\MenuLoss(h)\bigr)+\alpha_\mu(h).
\tag{2.6}\label{eq:proof-spine}
\end{equation}
The first term asks how well the learner predicts inside a fixed menu.  The
second asks only whether the menu contains labels on which the comparator is
\emph{correct}.  In particular, oracle mistakes outside the menu are not
charged a second time.  This is where coefficient one on \(\Risk_{\cD}(h)\) is
won or lost.

An idealized case makes the point transparent.  Suppose a menu is fixed in
advance and contains every correct label of \(h\).  Then
\(\alpha_\mu(h)=0\), so it is enough to compare \(g\) with \(h\) under the
binary loss \(\ell_\mu\).  If this comparison localizes at
\(\MenuLoss(h)\), then \eqref{eq:menu-loss-below-risk} automatically
localizes it at \(\Risk_{\cD}(h)\).  The full learner has only one additional
task: construct such a menu approximately, with a correct-label miss
probability of realizable order.

\subsection{Three blocks for three logically different tasks}

The three-block construction is shown in \cref{fig:architecture}.  The split
is not merely a convenience for applying concentration inequalities.  Each
fresh block freezes an object that becomes the statistical environment for
the next block.

\begin{figure}[t]
  \centering
  \begin{tikzpicture}[
      box/.style={draw=LinkColor, rounded corners=2pt, fill=LinkColor!5,
        align=center, inner sep=7pt, text width=.225\textwidth,
        minimum height=2.3cm},
      arr/.style={-{Stealth[length=2.2mm]}, thick, draw=LinkColor},
      edge label/.style={font=\scriptsize, fill=white,
        inner xsep=1pt, inner ysep=.5pt},
      node distance=18mm]
    \node[box] (cover) {
      \textbf{Block \(S_1\): cover}\\[2pt]
      replace \(\cH\) by a finite family\\[3pt]
      \(\log|\widehat\cF|\lesssim\dDS\log^2 n\)};
    \node[box, right=of cover] (menu) {
      \textbf{Block \(S_2\): menu}\\[2pt]
      cover correct labels of a canonical witness\\[3pt]
      \(\alpha_\mu(h)\lesssim\dDS\log^2 n/n\)};
    \node[box, right=of menu] (compress) {
      \textbf{Block \(S_3\): compress}\\[2pt]
      learn under the frozen binary loss \(\ell_\mu\)\\[3pt]
      \(k\lesssim\dN\log^2 n\)};
    \draw[arr] (cover) -- node[edge label, midway, above=2pt] {finite family}
      (menu);
    \draw[arr] (menu) -- node[edge label, midway, above=2pt] {fixed menu}
      (compress);
  \end{tikzpicture}
  \caption{The proof does not apply one complexity measure to the original
  multiclass loss.  It first pays \(\dDS\) to make correct labels visible,
  then pays \(\dN\) to learn within the resulting menu.}
  \label{fig:architecture}
\end{figure}
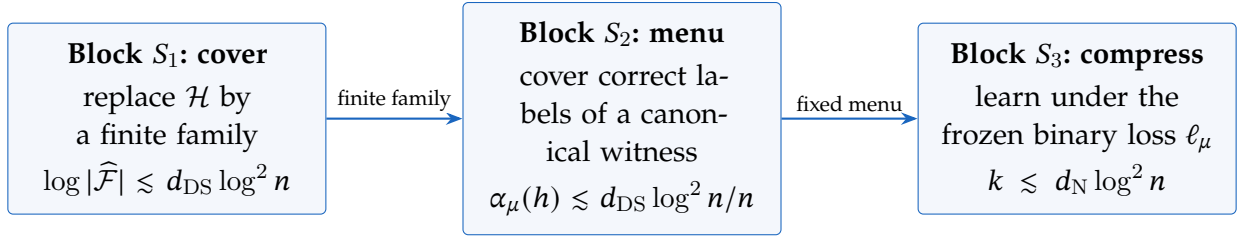

\paragraph{Block \(S_1\): make the comparator finite without learning it.}
The menu learner used in the second block operates on a finite family, while
the original class \(\cH\) may be infinite.  The one-inclusion,
boosting-to-compression, and cover steps are due to
\cite[Proposition~3.3 and Theorem~3.2]{cohen2025natarajan}; the sharp
conversion of their realizable parameter to \(\dDS\) uses
\cite[Corollary~1.1]{pabbaraju2026optimal}.  Together they produce a
deterministic ordinary compressor of size \(O(\dDS\log n)\).  Enumerating all
of its messages on \(S_1\) gives a finite family \(\widehat\cF\).  For a labeled sequence
\(s\) and a classifier \(h\), write \(s[h]\) for the ordered subsequence on
which \(h\) is correct.  More importantly, for every
fixed comparator \(h\), the family contains a canonical witness
\(f_{h,S_1}=A_1(S_1[h])\) satisfying, with high probability,
\[
 \log|\widehat\cF|\le C\dDS\log^2(en),
 \qquad
 \P(h(X)=Y\ne f_{h,S_1}(X))
 \le C\frac{\dDS\log^2(en)+\log(1/\eta)}{|S_1|}
\tag{2.7}\label{eq:architecture-cover}
\]
(the native scale of both logarithms is \(e|S_1|\); since
\(|S_1|\le n\), we state them at the common scale \(en\) once and for all).
The second guarantee is deliberately restricted to the region on which
\(h\) is correct.  That is all \eqref{eq:menu-miss-transfer} will ask us to
control; approximating \(h\) on its errors would solve a stronger and
unnecessary problem.

\paragraph{Block \(S_2\): cover the witness rather than identify the oracle.}
Conditional on \(S_1\), the family \(\widehat\cF\) and the canonical witness
are fixed.  The multiplicative-weights menu learner and its correct-label
coverage guarantee are due to \cite[Algorithm~3 and
Theorem~3.4]{cohen2025natarajan}.  When run for
\(T=|S_2|\) rounds, it returns classifiers \(h_1,\ldots,h_{T-1}\) and hence the
menu
\[
 \mu(x)=\{h_1(x),\ldots,h_{T-1}(x)\},
 \qquad \sup_x|\mu(x)|\le T-1.
\]
For every fixed \(f\in\widehat\cF\), its correct-label menu miss
\[
 \beta_f(\mu)
 \defeq \P(f(X)=Y,\ Y\notin\mu(X))
\]
is, with high probability,
\begin{equation}
 \beta_f(\mu)
 \le
 \frac{4\log|\widehat\cF|+14\log(3/\eta)+12}{T}.
\tag{2.8}\label{eq:architecture-menu}
\end{equation}
Apply this to \(f_{h,S_1}\).  The set inclusion
\eqref{eq:menu-miss-transfer}, together with
\eqref{eq:architecture-cover}--\eqref{eq:architecture-menu}, gives
\[
 \alpha_\mu(h)
 \lesssim
 \frac{\dDS\log^2(en)+\log(1/\eta)}{n}.
\tag{2.9}\label{eq:architecture-alpha}
\]
The menu need not reveal which hypothesis is the oracle.  It only has to
contain, most of the time, the labels produced by one canonical witness on
the oracle's correct region.

\paragraph{Block \(S_3\): localize only after the menu is frozen.}
For a binary loss \(\ell:\cF\times\cZ\to\{0,1\}\), call a deterministic
compression rule \(A=(\kappa,\rho)\) \emph{empirical-risk-dominating} if it
selects an ordered sequence of at most \(k\) sample elements, reconstructs
\(A(s)=\rho(\kappa(s))\), and satisfies
\begin{equation}
 \Rhat_s^\ell(A(s))
 \le \inf_{h\in\cH}\Rhat_s^\ell(h)
\tag{2.10}\label{eq:emp-domination}
\end{equation}
for every sample \(s\).  Repeated selected indices are allowed.  Conditional
on the first two blocks and the menu seed, \(\mu\) is fixed and \(S_3\)
remains iid.  The fixed-menu compression construction of
\cite{cohen2025natarajan} then supplies a rule satisfying
\eqref{eq:emp-domination} under \(\ell_\mu\), with description length
\[
 k\lesssim \dN\log(T-1)\log(e|S_3|)
 \lesssim \dN\log^2(en)
\]
when \(T\ge3\); empty and singleton menus admit a size-zero finalizer.

At this point a worst-case compression theorem would give a fluctuation of
order \(\sqrt{\dN/n}\), erasing the information that the oracle is nearly
correct.  The remaining module must instead compare the final reconstruction
with a fixed comparator at the comparator's own Bernoulli variance scale.
This is the role of \cref{thm:relative-compression}.

\paragraph{Why the order cannot be shortened.}
The cover must be frozen before the menu rounds, otherwise the finite-family
guarantee does not apply to a fixed family.  The menu must be frozen before
the final compression sample, otherwise \(\ell_\mu\) is not a fixed loss on
the observations used by the compression theorem.  The canonical witness
depends on \(S_1\), but it is not supplied to the menu algorithm: the menu
guarantee is pointwise for every member fixed after conditioning on \(S_1\).
These independence relations explain the architecture; the exact sigma-fields
and tower-property calculation are recorded in \cref{app:main}.

Every imported result used in the preceding discussion is restated as a
formal lemma in \cref{app:toolkit-statements}, with its source locator,
convention mapping, and applicability check.  Thus the main text can follow
the proof's causal order without turning the reader away to another paper.

\section{The upper bound and its two modules}
\label{sec:results}

Let \(\RMC\) denote the three-block learner described formally in
\cref{alg:rmc}.  We state its guarantee before entering the two technical
modules, so that the role of each intermediate bound remains visible.

\begin{restatable}[Small-error multiclass learning]{theorem}{ThmMain}
\label{thm:main}
Suppose \cref{asm:measurable} holds and
\(\cH\subseteq[K]^\cX\) has \(1\le\dN\le\dDS<\infty\).
There is a universal constant \(C>0\) such that, for every distribution
\(\cD\) on \(\cX\times[K]\), every \(n\ge3\), and every \(\delta\in(0,1)\),
the output of \(\RMC\) satisfies, with probability at least \(1-\delta\)
over the sample and the menu learner's randomness,
\[
 \Risk_{\cD}(\RMC(S))
 \le \Lstar+C\Biggl(
 \sqrt{\frac{\Lstar(\dN\log^3(en)+\log(1/\delta))}{n}}
 +\frac{\dN\log^3(en)+\dDS\log^2(en)+\log(1/\delta)}{n}
 \Biggr).
\tag{3.1}\label{eq:main-bound}
\]
The learner is independent of \(\Lstar\) and \(\delta\), and may be
computationally inefficient.
\end{restatable}

The theorem answers the interpolation question in the same risk scale as
\eqref{eq:binary-target}.  At \(\Lstar=0\), the stochastic square-root term
vanishes and the DS contribution has realizable order.  When \(\Lstar\) is
constant, the result recovers the established worst-case rate up to
logarithms.  Between these endpoints, only the Natarajan fluctuation is
multiplied by the actual oracle risk.  The interpolation is not merely
consistent with the two endpoints: \cref{thm:lower} exhibits a single
class in which a single distribution forces both terms at every fixed
\(\Lstar\), so \eqref{eq:main-bound} is optimal up to logarithmic
factors.

\begin{remark}[Constants]\label{rem:constants}
The coefficient of \(\Lstar\) in \eqref{eq:main-bound} is exactly one;
\(C\) multiplies only the fluctuation and the remainder.  Let
\(c_1,c_2,c_3\) be the absolute constants of the three imported bounds,
which the source states asymptotically: the compression size
\(k_1(m)\le c_1(\dRE(\cH)+1)\log(em)\) of
\cref{lem:ordinary-compression}, the correct-region guarantee of
\cref{lem:compression-cover} in its \((k_A(m)+1)\)-form, and the size
bound \eqref{eq:menu-compression-size} of
\cref{lem:fixed-menu-compression}.  Every other constant in the proof is
explicit---\(12\) and \(20\) in \cref{thm:relative-compression}, the
ceiling of \cref{lem:ds-re-bridge}, the menu constants of
\eqref{eq:menu-miss}, block sizes \(n_i\ge n/5\), three \(\delta/3\)
events---and tracking them through \cref{app:main} shows that one may
take
\[
 C=12\sqrt{5(c_3+2)}+100(c_3+2)+5(c_2+4)(2636\,c_1+4),
\]
the factor \(2636\) bounding \((\dRE(\cH)+1)/\dDS\) via
\cref{lem:ds-re-bridge}.  The lower bound needs no bookkeeping:
\eqref{eq:lower-bound} holds with the absolute constants \(1/12\) and
\(1/15\).
\end{remark}

\subsection{The learner}

The formal learner mirrors \cref{fig:architecture}.  The objects
\(A_1\), \(\CompCover\), \(\MW\), and \(B_{\mu,T}\) are the deterministic
ordinary compressor, its induced cover, the finite-family menu routine, and
the fixed-menu finalizer, respectively.  Their full interface statements are
collected in \cref{app:toolkit-statements}.

Algorithm~1 retains the three-stage skeleton of
\cite[Algorithm~4]{cohen2025natarajan}.  Its first-stage \(\dDS\) size uses
the realizable guarantee of \cite[Corollary~1.1]{pabbaraju2026optimal}; its
new step is to analyze the last-stage reconstruction through
\cref{thm:relative-compression} instead of an absolute square-root bound.

\begin{algorithm}[H]
  \caption{Relative menu compression \(\RMC\)}
  \label{alg:rmc}
  \begin{algorithmic}[1]
    \Require Sample \(S\in(\cX\times[K])^n\), class \(\cH\), \(n\ge3\)
    \State Split \(S\) into independent consecutive blocks \(S_1,S_2,S_3\)
           whose sizes differ by at most one; put \(T=|S_2|\).
    \State Form the compression cover
           \(\widehat\cF\gets\CompCover(S_1,\cH,A_1)\).
    \State Learn the menu
           \(\mu\gets\MW(T,S_2,1/2,\widehat\cF)\).
    \State \Return \(B_{\mu,T}(S_3)\), the fixed-menu
           empirical-risk-dominating reconstruction under \(\ell_\mu\).
  \end{algorithmic}
\end{algorithm}

The last line includes all menu sizes.  When \(T\le2\),
\(B_{\mu,T}\) is a size-zero rule that predicts the unique menu member when
there is one; when \(T\ge3\), it is the list sample-compression rule with the
deterministic menu bound \(T-1\).  The small-menu branch is not a coding
convention: for an empty or singleton menu, predicting the unique menu
member incurs zero inside-menu loss pointwise, so the size-zero rule is
automatically empirical-risk-dominating.  This detail is formalized in
\cref{lem:menu-finalizer}; the statistical point is that the last rule and
its loss are fixed before \(S_3\) is observed.

The guarantee of
\cite{cohen2025natarajan} is stated as a sample complexity, and its proof
chooses the block lengths as functions of the target accuracy and
confidence.  The fixed-risk form \eqref{eq:main-bound} instead holds for the
balanced split simultaneously for every \(n\) and \(\delta\).  This
uniformity is not cosmetic: it is exactly what the claim that \(\RMC\) uses
neither \(\Lstar\) nor \(\delta\) requires, since a split tuned to the
target accuracy would reintroduce the parameter the learner is supposed to
avoid.  The parameter-freeness target itself is the one set by the binary
fixed-risk line \cite{hanneke2024revisiting,mathiasen2026optimal}.

\subsection{The module that preserves coefficient one}

Once the menu has been frozen, the multiclass geometry has disappeared from
the concentration step.  We have a binary loss \(\ell\), a deterministic
compression rule that dominates a comparator empirically, and a description
length \(k\).  A standard compression bound would control population loss by
empirical loss plus \(O(\sqrt{k\log N/N})\).  That is too coarse here: even
if the comparator makes no errors, the square-root term remains.

Compression bounds at nonzero empirical risk go back to
\cite{graepel2005pac}, where the role of message order and repetition
conventions is already explicit.
The closest antecedent is the empirical-Bernstein compression bound of
\cite[Theorem~8]{gottlieb2017semimetrics}, which shares both devices of the
proof below---description counting and Bernstein control of the held-out
observations---and localizes the risk at the reconstruction's empirical
error.  We make no claim of a new concentration mechanism.  The statement
below differs from that antecedent in three ways, each forced by the
composition it must enter.  It is anchored at a \emph{fixed comparator's
population} risk with coefficient one, rather than at the reconstruction's
empirical error; this is what survives substitution into the menu
decomposition.  Its messages are \emph{ordered with repeated indices},
matching the compressor the menu construction actually produces.  And it
holds for \emph{every} \(N\) and \(k\) with explicit constants and no
restriction on the error level, which the conditional application inside a
random-menu composition consumes verbatim.
The target---coefficient one on the comparator risk with fluctuation at its
own Bernoulli scale---is inspired by the binary small-error comparison of
\cite{mathiasen2026optimal}.  The theorem below does not import their cube
geometry; it realizes that comparison principle through compression-message
counting and one-sided Bernstein inequalities.

The size-zero case suggests the correct replacement.  If a single fixed
predictor has no larger empirical loss than a fixed comparator, then two
one-sided Bernoulli deviations compare their population risks.  The
fluctuation is proportional to the square root of the comparator's mean, not
to one.  A nonzero compression message changes only how many such
comparisons must hold simultaneously.

For a binary loss \(\ell:\cF\times\cZ\to\{0,1\}\), write
\[
 L_{\cD}^\ell(f)=\E_{Z\sim\cD}\ell(f,Z),
 \qquad
 \Rhat_s^\ell(f)=\frac1N\sum_{i=1}^N\ell(f,z_i).
\]
For integers \(N\ge1\) and \(k\ge0\), and \(\delta\in(0,1)\), define
\[
 \Gamma_N(k,\delta)
 \defeq
 \frac{(k+1)\log(N+1)+\log(4/\delta)}{N}.
\tag{3.2}\label{eq:gamma}
\]

\begin{restatable}[Relative compression]{theorem}{ThmRelative}
\label{thm:relative-compression}
Let \((\cZ,\mathscr Z)\) and \((\cF,\mathscr F)\) be measurable spaces,
let \(\cH\subseteq\cF\), and let
\(\ell:\cF\times\cZ\to\{0,1\}\) be jointly measurable.  Fix \(N\ge1\)
and an integer \(k\ge0\).  Let \(A=(\kappa,\rho)\) be a measurable
deterministic selection scheme such that, on every sample
\(s\in\cZ^N\), \(\kappa(s)\) is an ordered sequence of at most \(k\)
elements of \(s\) (indices may repeat), \(A(s)=\rho(\kappa(s))\), and
\[
 \Rhat_s^\ell(A(s))
 \le \inf_{h\in\cH}\Rhat_s^\ell(h).
\tag{3.3}\label{eq:relative-premise}
\]
Then, for every probability distribution \(\cD\) on
\((\cZ,\mathscr Z)\), \(S\sim\cD^N\), every fixed \(h\in\cH\), and every
\(\delta\in(0,1)\), with probability at least \(1-\delta\),
\[
 L_{\cD}^\ell(A(S))
 \le L_{\cD}^\ell(h)
 +12\sqrt{L_{\cD}^\ell(h)\Gamma_N(k,\delta)}
 +20\Gamma_N(k,\delta).
\tag{3.4}\label{eq:relative-bound}
\]
No stability assumption on the selection map is required.
\end{restatable}

\paragraph{How the proof works.}
There are at most
\(M=\sum_{j=0}^kN^j\le(N+1)^{k+1}\) ordered descriptions.  Fix one of them,
let \(U\) be the set of distinct sample coordinates it names, and condition
on those coordinates.  The reconstruction is now fixed, while the
observations outside \(U\) are still iid.  Let \(\widetilde L_I(f_I)\) be its
average loss on these untouched observations.  A one-sided Bernstein
inequality therefore gives, simultaneously over all descriptions,
\[
 L_{\cD}^\ell(f_I)
 \le \widetilde L_I(f_I)
 +\sqrt{2L_{\cD}^\ell(f_I)s_I}+\frac{s_I}{3},
 \qquad
 s_I=\frac{\log(2M/\delta)}{N-|U|}.
\tag{3.5}\label{eq:relative-overview-one}
\]
The deletion of \(U\) is essential: it is what makes the reconstruction fixed
relative to the observations on which its loss is measured.  Repeated
indices create no additional dependence because only distinct coordinates
are deleted; they are nevertheless counted as different ordered messages.

For the description selected on the realized sample, empirical domination
does not directly compare held-out loss with the comparator.  It first gives
the deterministic inequality
\[
 \widetilde L_I(f_I)
 \le \frac{N}{N-|U|}\Rhat_S^\ell(h).
\tag{3.6}\label{eq:relative-overview-two}
\]
A second one-sided Bernstein inequality bounds the fixed comparator's
empirical loss by its population loss \(L=L_{\cD}^\ell(h)\).  Substituting this
bound into \eqref{eq:relative-overview-two} and then into
\eqref{eq:relative-overview-one} leaves an implicit inequality because the
unknown reconstruction risk appears under the square root.  Solving it gives
\(L+O(\sqrt{L\Gamma}+\Gamma)\).  The full proof in
\cref{app:relative} records the two probability events, the effect of
\(N-|U|\), and every constant leading to \(12\) and \(20\).

Unlike that empirical-error antecedent, the form above gives the
fixed-comparator, order-dependent comparison required by the random-menu
composition, without a stability assumption.

\subsection{Assembling coverage and localization}

We now follow the learner from left to right.  This is also the shortest way
to see why the two dimensions land in different terms of
\eqref{eq:main-bound}.

The cover and menu guarantees in the first two blocks are inherited from
\cite[Theorems~3.2 and~3.4]{cohen2025natarajan}, with their realizable
parameter converted to \(\dDS\) by
\cite[Corollary~1.1]{pabbaraju2026optimal}.  The new operation in the
assembly is the conditional application of \cref{thm:relative-compression}
to the third block.

Fix \(\eps>0\) and choose \(h^\star\in\cH\) with
\(\Risk_{\cD}(h^\star)\le\Lstar+\eps\).  On the first block, the compression
cover contains the canonical witness
\(f_1^\star=A_1(S_1[h^\star])\).  Conditional on \(S_1\), the menu theorem
applies to this fixed witness on the fresh second block.  The two correct
region guarantees and \eqref{eq:menu-miss-transfer} yield
\begin{equation}
 \alpha_\mu(h^\star)
 =\P(h^\star(X)=Y,\ Y\notin\mu(X))
 \le
 C\frac{\dDS\log^2(en)+\log(1/\delta)}{n}.
\tag{3.7}\label{eq:alpha-sketch}
\end{equation}
This is the entire DS contribution.

Next freeze \(S_1,S_2\) and the menu seed.  The menu \(\mu\), the binary
loss \(\ell_\mu\), and the final compression rule are now fixed, whereas
\(S_3\) remains iid.  The final rule is empirical-risk-dominating and has
message length \(O(\dN\log^2(en))\).  Applying
\cref{thm:relative-compression} under the conditional distribution gives
\[
 \MenuLoss(\RMC(S))-\MenuLoss(h^\star)
 \le C\Biggl(
 \sqrt{\frac{\MenuLoss(h^\star)(\dN\log^3(en)+\log(1/\delta))}{n}}
 +\frac{\dN\log^3(en)+\log(1/\delta)}{n}
 \Biggr).
\tag{3.8}\label{eq:menu-relative-sketch}
\]
This is the entire Natarajan contribution.

It remains to put the two estimates in the same currency.  The exact menu
decomposition \eqref{eq:proof-spine} gives
\begin{equation}
 \Risk_{\cD}(\RMC(S))-\Risk_{\cD}(h^\star)
 \le
 \MenuLoss(\RMC(S))-\MenuLoss(h^\star)
 +\alpha_\mu(h^\star).
\tag{3.9}\label{eq:menu-to-risk-sketch}
\end{equation}
Finally, \(\MenuLoss(h^\star)\le\Risk_{\cD}(h^\star)\) by
\eqref{eq:menu-loss-below-risk}.  Thus substituting
\eqref{eq:alpha-sketch} and \eqref{eq:menu-relative-sketch} into
\eqref{eq:menu-to-risk-sketch} preserves coefficient one on
\(\Risk_{\cD}(h^\star)\).  A union bound over the three blockwise events and
\(\eps\downarrow0\) prove \cref{thm:main}.  The complete proof in
\cref{app:main} names the conditioning sigma-fields, derives the
\(\Gamma_N\) bound, handles the two small-menu branches, and justifies the
limit when the infimum defining \(\Lstar\) is not attained.

\begin{corollary}[Small-error sample complexity]
\label{cor:sample-complexity}
For any target excess error \(\eps\in(0,1)\), the same learner achieves
\(\Risk_{\cD}(\RMC(S))\le\Lstar+\eps\) with probability at least \(1-\delta\)
from a sample of size
\[
 n=\widetilde O\!\left(
 \frac{\Lstar(\dN+\log(1/\delta))}{\eps^2}
 +\frac{\dDS+\log(1/\delta)}{\eps}
 \right),
\tag{3.10}\label{eq:sample-complexity}
\]
where the logarithms hidden by \(\widetilde O\) are in
\(n,\dN,\dDS\).
\end{corollary}

The two terms in \eqref{eq:sample-complexity} are not the result of inverting
one undifferentiated worst-case bound.  The first is the sample size needed
to resolve inside-menu fluctuations at oracle risk \(\Lstar\); the second is
the sample size needed to make correct labels visible.  The algorithm does
not need to know which term dominates.

\section{A matching lower bound}
\label{sec:lower}

\Cref{thm:main} interpolates between two endpoints whose optimality was,
until now, supported by separate constructions: at \(\Lstar=0\) the
\(\dDS/n\) term matches the realizable lower bound
\cite{pabbaraju2026optimal}, and for constant \(\Lstar\) the localized
term is supported through binary subclasses and the worst-case agnostic
lower bound \cite{cohen2025natarajan}.  This
section closes the gap between those endpoints.  We construct a single
class in which a single distribution, at every fixed \(\Lstar\), forces
the localized term and the coverage term of \eqref{eq:main-bound}
simultaneously.  Everything in this
section is finitely supported, so no measurability issues arise and
\cref{asm:measurable} is not needed.

\begin{restatable}[Fixed-\(\Lstar\) lower bound]{theorem}{ThmLower}
\label{thm:lower}
For all integers \(1\le\dN\le\dDS\) there are a finite label set \([K]\)
and a finite class \(\cH\subseteq[K]^{\cX}\) with
\[
 \dN\ \le\ \Nat(\cH)\ \le\ \dN+1,
 \qquad
 \DS(\cH)\ =\ \dN+\dDS,
\]
such that for every \(n\ge1\), every \(\Lstar\in[0,1/8]\), and every
learner \(A\), possibly randomized, there is a distribution \(\cD\) on
\(\cX\times[K]\) with \(\inf_{h\in\cH}\Risk_{\cD}(h)=\Lstar\)
\emph{exactly}, under which
\[
 \E\bigl[\Risk_{\cD}(A(S))\bigr]-\Lstar
 \ \ge\
 \frac1{12}\,\ind\{\Lstar\ge\dN/n\}
 \sqrt{\frac{\Lstar\dN}{n}}
 \;+\;
 \frac1{15}\min\Bigl\{\frac{\dDS}{n},\frac14\Bigr\},
\tag{4.1}\label{eq:lower-bound}
\]
the expectation being over \(S\sim\cD^{n}\) and the learner's
internal randomness.  Both terms are forced by the same distribution,
drawn from a finite family that depends only on \((\cH,\Lstar,n)\) and
not on the learner.
\end{restatable}

Here \(\Nat(\cH)\) and \(\DS(\cH)\) denote the two dimensions of
\cref{def:natarajan,def:ds}.  Three reading notes.  First, the indicator gate is
harmless: for \(\Lstar<\dN/n\) the localized term is below \(\dDS/n\) on
both sides of the comparison, so \eqref{eq:main-bound} and
\eqref{eq:lower-bound} still match.  Second, since
\(\Nat(\cH)\le\dN+1\) and \(\DS(\cH)\le2\dDS\), the right-hand side of
\eqref{eq:lower-bound} can equivalently be written with the class's own
dimensions at the cost of a factor two in the constants.  Third, the bound
is stated in expectation; because the excess risk is bounded by one, it
implies that excess at least half the right-hand side occurs with
probability at least half the right-hand side, so \cref{thm:main} is
matched at fixed confidence, up to logarithmic factors.  The sharp
\(\log(1/\delta)\) dependence is not addressed here.

\paragraph{The two shattered structures.}
The class is a product over a disjoint domain,
\[
 \cH:=\cH_{\mathrm{cube}}\oplus\cW,
 \qquad
 (h_1,h_2)(x)=
 \begin{cases}
 h_1(x),&x\in\cX_1,\\
 h_2(x),&x\in\cX_2 .
 \end{cases}
\]
The first factor carries the Natarajan structure and the noise:
\(\cX_1=\{x_0,x_1,\ldots,x_{\dN}\}\), and \(\cH_{\mathrm{cube}}\) consists
of the \(2^{\dN}\) hypotheses \(g_\sigma\), \(\sigma\in\{\pm1\}^{\dN}\),
with \(g_\sigma(x_0)=y_0\) on an anchor point and
\(g_\sigma(x_i)\in\{a_i,b_i\}\) according to \(\sigma_i\), for distinct
labels \(a_i\ne b_i\) and an anchor label \(y_0\).  The second factor
carries the DS structure without paying Natarajan capacity, and this is
the one place an external ingredient enters.

\begin{lemma}[Separation pseudo-cubes]
\label{lem:separation}
For every \(D\in\N\) there are a finite label set \([K_D]\) and a
pseudo-cube \(V\subseteq[K_D]^{D}\) whose induced class
\(\cW=\{v(\cdot):v\in V\}\) on the domain \([D]\) satisfies
\(\Nat(\cW)\le1\) and \(\DS(\cW)=D\).
\end{lemma}

\begin{proof}
\cite{brukhim2022characterization} construct a class \(\cH_0\) with
\(\Nat(\cH_0)=1\) that is not PAC learnable, and prove that PAC
learnability is characterized by finiteness of the DS dimension; hence
\(\DS(\cH_0)=\infty\).  Fix \(D\) and a DS-shattered sequence of length
\(D\): its trace contains a pseudo-cube \(V\), which is a finite set by
definition, so only finitely many labels occur in it.  Any Natarajan
shattering of the induced class \(\cW\) transfers, through the
realizing hypotheses, to a Natarajan shattering of \(\cH_0\) on the
corresponding points, so \(\Nat(\cW)\le\Nat(\cH_0)=1\).  Finally
\(\DS(\cW)\ge D\) because the full sequence is shattered by \(V\), and
\(\DS(\cW)\le D\) because the domain has \(D\) points and no sequence
with a repeated point is DS-shattered: at a repeated coordinate every
trace vector takes equal values, while a pseudo-cube requires a neighbor
differing at exactly one of the two copies.
\end{proof}

\begin{example}[The \(6\)-cycle]
\label{ex:sixcycle}
At \(D=2\) the object of \cref{lem:separation} can be written down:
\(V=\{(1,1),(1,2),(2,2),(2,3),(3,3),(3,1)\}\subseteq[3]^2\).  Every
element has a neighbor in each coordinate, so \(V\) is a pseudo-cube and
\(\DS=2\); no two rows share two columns, so no pair of witness labels is
shattered and \(\Nat=1\).  Products of such cycles do \emph{not} give the
general case: by \cref{lem:dim-calculus} below, Natarajan dimension adds
across products, which is exactly why \cref{lem:separation} needs the
global construction of \cite{brukhim2022characterization}.
\end{example}

\begin{restatable}[Dimension calculus]{lemma}{LemDimCalc}
\label{lem:dim-calculus}
The restriction of a pseudo-cube to a nonempty subset of its coordinates
is a pseudo-cube.  Consequently, for classes on disjoint domains,
\[
 \Nat(\cH_1\oplus\cH_2)=\Nat(\cH_1)+\Nat(\cH_2),
 \qquad
 \DS(\cH_1\oplus\cH_2)=\DS(\cH_1)+\DS(\cH_2),
\]
and anchor points on which all hypotheses agree contribute to neither
dimension.
\end{restatable}

With \(\cW\) from \cref{lem:separation} at \(D=\dDS\),
\cref{lem:dim-calculus} gives
\(\Nat(\cH)=\dN+\Nat(\cW)\in\{\dN,\dN+1\}\) and
\(\DS(\cH)=\dN+\dDS\), as claimed in \cref{thm:lower}.

\paragraph{The two engines.}
The localized term is produced by a noisy pair-Assouad construction on the
cube factor (\cref{lem:assouad-lower}), the multiclass transplant of the
classical binary construction \cite{devroye1995lower}: each \(x_i\) receives mass
\(\tau\asymp\Lstar/\dN\) and a label drawn from the witness pair with bias
\(\gamma\asymp\sqrt{\dN/(n\Lstar)}\), while the anchor absorbs the
remaining mass at a deterministic label, keeping the oracle risk equal to
\(\Lstar\) for every sign pattern simultaneously.  Distinguishing the
favored label at one point from \(n\) samples is then a two-point testing
problem whose error is bounded through the total variation between the two
sample laws \cite[Chapter~2]{tsybakov2009introduction}, and each undecided
point costs excess \(2\gamma\tau\).

The coverage term is produced on the pseudo-cube factor by a realizable
leave-one-out argument, in the tradition of the classical realizable lower
bounds \cite{ehrenfeucht1989lower}.  Its combinatorial core deserves to
be stated on its own; the pseudo-cube axiom enters exactly once.

\begin{lemma}[Fiber lemma]
\label{lem:fiber}
Let \(V\subseteq[K]^D\) be a pseudo-cube, let \(J\subseteq[D]\) be
nonempty, and suppose the \emph{cloud}
\(C=\{u\in V:u_j=w_j\ \forall j\notin J\}\) is nonempty for some fixed
values \(w\).  Then for every \(i\in J\) and every label \(y\),
\[
 \bigl|\{u\in C:u_i=y\}\bigr|\ \le\ \tfrac12\,|C| .
\]
\end{lemma}

\begin{proof}
Partition \(C\) into fibers: \(u\sim u'\) iff \(u_j=u_j'\) for all
\(j\ne i\).  For \(u\in C\), its pseudo-cube neighbor \(u^{(i)}\) agrees
with \(u\) off \(i\), hence lies in \(C\) (because \(i\in J\)) and in the
fiber of \(u\); so every fiber has at least two elements.  Within a
fiber, distinct elements differ exactly at coordinate \(i\), so their
\(i\)th entries are pairwise distinct and each label occurs at most once
per fiber.  Hence
\(|\{u\in C:u_i=y\}|\le\#\{\text{fibers}\}\le|C|/2\).
\end{proof}

In the realizable problem ``labels drawn by an unknown \(v\in V\),''
the cloud is precisely the posterior support after observing the sample,
and the posterior under the uniform prior is uniform on it.
\Cref{lem:fiber} therefore says that at every unseen coordinate,
\emph{every} prediction is wrong with conditional probability at least
\(\tfrac12\)---the pseudo-cube keeps all of its coordinates permanently
ambiguous.  Uniformly spreading mass \(\nu/D\) over the \(D\) points
leaves each point unseen with probability \((1-\nu/D)^n\), and
\(\nu:=\min\{D/n,\tfrac14\}\) balances mass against ambiguity
(\cref{lem:ds-lower}).

\paragraph{Putting it all together.}
The final distribution family places mass \(1-\nu\) on the noisy cube
component, calibrated so that the total oracle risk is exactly
\(\Lstar\), and mass \(\nu\) on the pseudo-cube component with labels
given by \(v\in V\):
\[
 \cD_{\sigma,v}
 :=(1-\nu)\,\widetilde\cD_\sigma+\nu\,\cU_v,
 \qquad(\sigma,v)\in\{\pm1\}^{\dN}\times V,
\]
where \(\cU_v\) denotes the uniform distribution on the \(D\) labeled
pairs \(\{(z_i,v_i):i\in[D]\}\).
Because \(\cH\) is a product class over disjoint supports, its infimum
risk splits into the two components; the pseudo-cube component is
realizable, so \(\inf_{h}\Risk_{\cD_{\sigma,v}}(h)=\Lstar\) for
\emph{every} \((\sigma,v)\), and excess risks add across the two
components.  Averaging the side-1 bound over \(\sigma\) and the side-2
bound over \(v\), the expected excess of any learner against the uniform
prior is at least the sum of the two terms in \eqref{eq:lower-bound}, so
some single \((\sigma,v)\) achieves it.  The full proof, with all
constants, is in \cref{app:lower}.

\begin{remark}[Growing alphabets are necessary]
\label{rem:alphabet}
For \(\dDS\gg\dN\) the label alphabet in \cref{thm:lower} must grow.  At
constant excess risk, \cref{lem:ds-lower} forces \(n\ge c\,\DS(\cH)\)
samples in the realizable case; ERM arguments give \(n=O(d_G)\) at
constant excess, where \(d_G\) is the Graph dimension
\cite{daniely2015multiclass}; and \(d_G=O(\dN\log K)\) by the classical
comparison of the two dimensions \cite{ben1995characterizations}.
Chaining the three, \(\DS(\cH)=O(\Nat(\cH)\log K)\): a separation
\(\dDS\gg\dN\) is only possible when \(K\ge\exp(\Omega(\dDS/\dN))\).
This is why \cref{thm:main}'s uniformity in \(K\) is the right setting
for the two-dimensional rate: at bounded alphabets the two dimensions
collapse to within a logarithmic factor, and the interpolation question
degenerates.
\end{remark}

\begin{remark}[Beyond the constructed family]
\label{rem:universal-lower}
If the exact-\(\Lstar\) constraint is relaxed to
\(\inf_h\Risk_\cD(h)\le\Lstar\), the two terms of \eqref{eq:lower-bound}
hold for \emph{every} class, with \(\Nat(\cH)-1\) and \(\DS(\cH)-1\) in
place of the constructed dimensions: fix one coordinate of a shattered
structure and use its value as the anchor; the subfamily fixing that
coordinate is again a shattering structure of one lower dimension, both
for sign cubes and, by the neighbor axiom, for pseudo-cubes.  The
exact-\(\Lstar\), single-distribution form of \cref{thm:lower} genuinely
uses the constructed product class: for an entangled class the two
structures need not admit a common distribution whose infimum risk is
computable exactly.
\end{remark}

\section{Extension: optimistic rates for list learning}
\label{sec:list}

\begin{table}[t]
  \centering
  \caption{List-learning risk bounds.  Each cell shows the excess-risk
  scale and its source; constants, logarithms, and the truncations of
  \cref{thm:lower,thm:list-lower} are suppressed.  Binary labels admit
  no nontrivial lists: for \(K=2\) and \(r\ge2\) the full-label list
  has zero risk.  In the third row the first rates are due to
  \cite{charikar2023list}, successively improved by
  \cite{brukhim2023boosting,hanneke2026sauer,pabbaraju2026optimal};
  the lower bound of \cite{hanneke2024improved} reads
  \(\dDSk{r}/(rn)\), and \cref{cor:list-realizable} removes its factor
  \(r\).  Its two columns differ in shape because a list can hedge
  against a single comparator, so no fluctuation-type lower bound is
  possible there (\cref{rem:list-gap}).  In the highlighted row the
  comparator is the best \(r\)-tuple, both columns share one shape,
  and the dimension gap is \cref{rem:list-gap}.}
  \label{tab:list-rates}
  \begin{tabular}{@{}lll@{}}
    \toprule
    Setting & Upper bound & Lower bound \\
    \midrule
    Binary, \(r=1\)
      & \begin{tabular}[t]{@{}l@{}}
          \(\sqrt{\Lstar d/n}+d/n\)\\[1pt]
          {\footnotesize\cite{mathiasen2026optimal}}
        \end{tabular}
      & \begin{tabular}[t]{@{}l@{}}
          \(\sqrt{\Lstar d/n}+d/n\)\\[1pt]
          {\footnotesize\cite{devroye1995lower,asilis2025smallerror}}
        \end{tabular} \\
    \addlinespace
    Multiclass, \(r=1\)
      & \begin{tabular}[t]{@{}l@{}}
          \(\sqrt{\Lstar\dN/n}+\dDS/n\)\\[1pt]
          {\footnotesize\cref{thm:main}}
        \end{tabular}
      & \begin{tabular}[t]{@{}l@{}}
          \(\sqrt{\Lstar\dN/n}+\dDS/n\)\\[1pt]
          {\footnotesize\cref{thm:lower}}
        \end{tabular} \\
    \addlinespace
    \(r\)-list vs.\ single \(h\)
      & \begin{tabular}[t]{@{}l@{}}
          \(\sqrt{r^{4}\dNk{r}/n}+r\dDSk{r}/n\)\\[1pt]
          {\footnotesize\cite{charikar2023list,pabbaraju2026optimal}}
        \end{tabular}
      & \begin{tabular}[t]{@{}l@{}}
          \(\dDSk{r}/n\)\\[1pt]
          {\footnotesize\cref{cor:list-realizable}, cf.\
            \cite{hanneke2024improved}}
        \end{tabular} \\
    \midrule
    \rowcolor{TabHighlight}
    \(r\)-list vs.\ \(r\)-tuple
      & \begin{tabular}[t]{@{}l@{}}
          \(\sqrt{\Lstarr\,r\dN/n}+r\dDS/n\)\\[1pt]
          {\footnotesize\textbf{this work}, \cref{thm:list-upper}}
        \end{tabular}
      & \begin{tabular}[t]{@{}l@{}}
          \(\sqrt{\Lstarr\,\dNk{r}/n}+\dDSk{r}/n\)\\[1pt]
          {\footnotesize\textbf{this work}, \cref{thm:list-lower}}
        \end{tabular} \\
    \bottomrule
  \end{tabular}
\end{table}

The intermediate object of the upper-bound architecture is already a
list-valued predictor: the menu of \cref{sec:architecture} assigns each
point a set of candidate labels, and the final block exists only to
collapse that set to a single label.  It is therefore natural to ask what
the architecture says about \emph{list learning}, where the learner is
allowed to output a bounded list of labels and is charged only when the
true label misses the entire list.  List PAC learning was characterized
by \cite{charikar2023list}: for a fixed list size \(r\), learnability is
equivalent to finiteness of their \(r\)-DS dimension.  The quantitative
picture has been sharpened in a rapid sequence of works, through
multiclass boosting \cite{brukhim2023boosting}, an optimal \(k\)-ary
Sauer lemma \cite{hanneke2026sauer}, and the resolution of the
Daniely--Shalev-Shwartz density conjecture \cite{pabbaraju2026optimal},
whose rates are the current state of the art.  All of these guarantees,
however, compare the list learner to the best \emph{single} hypothesis,
and their fluctuation terms are worst-case in the oracle risk.  This
section states the optimistic-rate analogue: the comparator is upgraded
to the best \(r\)-\emph{tuple} of hypotheses, the fluctuation is
localized at the oracle list risk, and a matching construction forces
both terms of the rate at every fixed oracle level.  The lower bound
proves, quantitatively and with exact calibration, an expectation
recorded in \cite{pabbaraju2026optimal}, that a fluctuation term of
\(r\)-ary Natarajan type is necessary against list comparators; and its
second engine yields, as a byproduct, a factor-\(r\) improvement of the
known realizable list lower bound (\cref{cor:list-realizable}).

\paragraph{Setting.}
An \(r\)-list predictor is a measurable map
\(\lambda:\cX\to\{Y\subseteq[K]:|Y|\le r\}\), and its list risk is
\[
 \ListRisk_{\cD}(\lambda)
 =\P\bigl(Y\notin\lambda(X)\bigr).
\]
For an \(r\)-tuple \(\vec h=(h^1,\ldots,h^r)\in\cH^r\), write
\(\Lambda_{\vec h}(x)=\{h^1(x),\ldots,h^r(x)\}\) and
\[
 \Lstarr
 =\inf_{\vec h\in\cH^r}\ListRisk_{\cD}(\Lambda_{\vec h}),
\]
the oracle list risk.  Since a singleton is a \(1\)-tuple,
\(\Lstarr\le\Lstar\), so a bound localized at \(\Lstarr\) implies the
corresponding bound against single comparators.  We use the list
dimensions of \cite{charikar2023list}: a sequence \(x_1,\ldots,x_d\) is
\emph{\(r\)-Natarajan shattered} if there are witness sets
\(Y_i\subseteq[K]\) with \(|Y_i|=r+1\) such that the trace of \(\cH\)
contains the full product \(\prod_i Y_i\), and \(\dNk{r}\) is the largest
such \(d\); a sequence is \emph{\(r\)-DS shattered} if some finite trace
set gives every member, in every coordinate, at least \(r\) neighbors
differing from it exactly there, and \(\dDSk{r}\) is the largest such
length.  At \(r=1\) these are \cref{def:natarajan,def:ds}, and
\(\dNk{r}\le\dDSk{r}\le\dDS\) with both sequences nonincreasing
in \(r\).  \Cref{tab:list-rates} places the two theorems of this
section against the known list-learning rates.

\begin{restatable}[Optimistic rates for list learners]{theorem}{ThmListUpper}
\label{thm:list-upper}
Suppose \cref{asm:measurable} holds, \(\cH\) has
\(1\le\dN\le\dDS<\infty\), and \(r\ge1\).  There is a universal constant
\(C>0\) and a learner \(\RMC_r\), outputting lists of size at most
\(r\), such that for every distribution \(\cD\), every \(n\ge3\), and
every \(\delta\in(0,1)\), with probability at least \(1-\delta\),
\[
 \ListRisk_{\cD}\bigl(\RMC_r(S)\bigr)
 \le \Lstarr+C\Biggl(
 \sqrt{\frac{\Lstarr\bigl(r\,\dN\log^3(en)+\log(1/\delta)\bigr)}{n}}
 +\frac{r\bigl(\dN\log^3(en)+\dDS\log^2(en)
   +\log(er/\delta)\bigr)}{n}
 \Biggr).
\tag{5.1}\label{eq:list-upper}
\]
The learner is independent of \(\Lstarr\) and \(\delta\).  At \(r=1\)
the statement is \cref{thm:main}.
\end{restatable}

The proof is in \cref{app:list}.  Blocks \(S_1\) and \(S_2\) are reused
verbatim: the same menu covers each component of a comparator tuple at
realizable order.  The only new module is the third block, a fixed-menu
compression scheme whose reconstructions are \(r\)-lists and whose
empirical domination is against every \(r\)-tuple.  It needs no new
capacity analysis: on the subsample where an optimal tuple is
list-correct, the points split into \(r\) groups according to which
component covers them, each group is a realizable instance for the
single-hypothesis scheme of \cref{lem:fixed-menu-compression}, and the
union of the \(r\) reconstructions dominates.  The message is an
\(r\)-tuple of ordinary messages, and
\cref{thm:relative-compression} applies with the description count
adjusted accordingly: its proof never inspects the loss or the message
format beyond measurability, boundedness, and the number of possible
descriptions.

\begin{restatable}[Lower bound for list learners]{theorem}{ThmListLower}
\label{thm:list-lower}
Let \(r\ge1\).  For all integers \(d\ge1\) and even \(D\) with
\(2\le D\le 2d\) there are a finite label set \([K]\) and a finite class
\(\cH\subseteq[K]^{\cX}\) with
\[
 d\ \le\ \dNk{r}(\cH)\ \le\ d+\tfrac D2,
 \qquad
 \dDSk{r}(\cH)\ =\ d+D,
\]
such that for every \(n\ge1\), every
\(\Lstarr\in[0,\tfrac1{8(r+1)}]\), and every learner \(A\) outputting
lists of size at most \(r\), possibly randomized, there is a
distribution \(\cD\) on \(\cX\times[K]\) with
\(\inf_{\vec h\in\cH^r}\ListRisk_{\cD}(\Lambda_{\vec h})=\Lstarr\)
\emph{exactly}, under which
\[
 \E\bigl[\ListRisk_{\cD}(A(S))\bigr]-\Lstarr
 \ \ge\
 \frac1{48}\,\ind\{\Lstarr\ge d/n\}
 \sqrt{\frac{\Lstarr d}{n}}
 \;+\;
 \frac1{15}\min\Bigl\{\frac{D}{n},\frac14\Bigr\}.
\tag{5.2}\label{eq:list-lower}
\]
Both terms are forced by the same distribution, and the calibration is
Bayes-exact: every list of size \(r\), from any class, has risk at least
\(\Lstarr\) under every member of the family.
\end{restatable}

The proof, also in \cref{app:list}, upgrades both engines of
\cref{sec:lower}.  The pair-Assouad construction becomes
\((r+1)\)-ary: each hard point now carries \(r+1\) candidate labels, of
which the oracle list drops the lightest, and a learner's list must
still exclude at least one candidate.  Distinguishing which candidate is
light remains a two-point testing problem, because the two laws obtained
by exchanging the roles of two candidates differ on exactly two atoms;
the testing chain of \cref{lem:assouad-lower} applies with the pair
masses \(\tfrac1{r+1}\pm\gamma\)-type in place of
\(\tfrac12\pm\gamma\).  The padded leave-one-out argument runs on
pseudo-cube structures of width \(2r\): every fiber then offers at least
\(2r\) distinct values, of which an \(r\)-list covers at most half, so
\cref{lem:ds-lower} survives with the same constant.

The width-\(2r\) device has a consequence for the classical,
single-hypothesis-realizable list setting that seems worth recording
separately.  The best known realizable list lower bound,
\(\Omega\bigl((\dDSk{r}+\log(1/\delta))/(r\eps)\bigr)\) of
\cite{hanneke2024improved} (see the summary table of
\cite{pabbaraju2026optimal}), loses a factor \(r\) because its hard
structures offer \(r+1\) values per coordinate, of which an
\(r\)-list misses only one; doubling the width removes the loss.

\begin{corollary}[Realizable list learning without the factor
\(r\)]\label{cor:list-realizable}
For every \(r\ge1\) and \(D\ge1\) there are a finite label set and a
finite class \(\cH\) with \(\dDSk{r}(\cH)=D\) such that for every
\(n\ge1\) and every learner \(A\) outputting lists of size at most
\(r\), possibly randomized, there is a distribution realizable by
\(\cH\) under which
\[
 \E\bigl[\ListRisk_{\cD}(A(S))\bigr]
 \ \ge\ \frac1{15}\min\Bigl\{\frac Dn,\frac14\Bigr\}.
\]
\end{corollary}

Consequently the realizable sample complexity of \(r\)-list learning
is \(\Omega(\dDSk{r}/\eps)\): the dimension term of
\cite{hanneke2024improved} improves by a factor of \(r\), and of the
\(1/r\)-versus-\(r\) gap to the upper bound
\(O\bigl(r(\dDSk{r}+\log(1/\delta))/\eps\bigr)\) of
\cite[Corollary~1.3]{pabbaraju2026optimal} only the single factor
\(r\) remains.  The proof (\cref{app:list}) pads the full width-\(2r\)
product class on \(D\) points and applies the list leave-one-out
bound; no Assouad component is involved.

Three remarks calibrate the pair of results.

\begin{remark}[What lists buy, and what they do not]
\label{rem:list-cancellation}
The fluctuation term of \eqref{eq:list-lower} carries no gain in \(r\):
its constant is absolute.  This is a cancellation, not an accident of
the proof.  Spreading one unit of side-1 mass over \(r+1\) candidates
divides the per-candidate testing signal \(\gamma\) by \(\Theta(r)\),
but calibrating the oracle list risk to \(\Lstarr\) multiplies the
per-point mass \(\tau\) by \(\Theta(r)\), and the excess per testing
mistake is \(\gamma\tau\)-scaled, so the two effects cancel in the
product \(\gamma\tau d\).  Lists therefore do not shrink the fluctuation
attached to a given oracle level; what they buy is the smaller oracle
level itself, \(\Lstarr\le\Lstar\), and the smaller dimensions
\(\dNk{r}\le\dN\), \(\dDSk{r}\le\dDS\) at which the terms are charged.
\end{remark}

\begin{remark}[The dimension gap]
\label{rem:list-gap}
The fluctuation of \eqref{eq:list-upper} scales with \(r\dN\), the
fluctuation of \eqref{eq:list-lower} with \(\dNk{r}\), and
\(\dNk{r}\le\dN\) always, with arbitrarily large gaps possible.  We
conjecture that the truth is \(\dNk{r}\) up to factors polynomial in
\(r\), and two facts support this.  First, in the single-comparator
worst-case setting the fluctuation is already controlled by the
\(r\)-ary Natarajan dimension:
\cite[Corollary~1.4]{pabbaraju2026optimal} pays
\(\widetilde O(\sqrt{r^4\dNk{r}/n})\) via a fixed-menu list
compression of size \(\widetilde O(r^4\dNk{r})\) built on the list
one-inclusion machinery of \cite{charikar2023list,hanneke2026sauer}.
Second, \cite{pabbaraju2026optimal} records, after that corollary, the
expectation that a term of this type is necessary once the algorithm
must compete with the best list formed from hypotheses of \(\cH\),
while noting that the standard \(r=1\) lower-bound strategy fails
against single comparators because a list can hedge;
\cref{thm:list-lower} proves exactly this expectation, quantitatively
and with exact calibration, at the \(\dNk{r}\) of the constructed
class.  What is missing on the upper side is a fixed-menu compression
at \(\dNk{r}\) whose empirical domination holds against every
\(r\)-tuple rather than every single hypothesis: the scheme of
\cite{pabbaraju2026optimal} dominates single comparators only, and our
partition device pays \(r\dN\) instead.  Symmetrically, the
deterministic terms of the two bounds are stated at \(\dDS\) and at
\(\dDSk{r}\) respectively; replacing the cover of block \(S_1\) by one
built from a realizable list learner
\cite{charikar2023list,pabbaraju2026optimal} is the natural route
toward \(\dDSk{r}\) in \eqref{eq:list-upper}, at the cost of reworking
the imported modules, and we leave it open.
\end{remark}

\begin{remark}[Scope]
\label{rem:list-scope}
At \(r=1\) the two theorems reproduce
\cref{thm:main,thm:lower}, with worse constants in the lower bound; we
did not optimize the \(r\)-uniform chain.  The restriction \(D\le2d\) in
\cref{thm:list-lower} mirrors the unconditional diagonal of
\cref{thm:lower}: the general range reduces, exactly as in
\cref{lem:separation}, to importing a width-\(2r\) separation family
with \(\dNk{r}\)-dimension one at every size \(D\), a list analogue of
\cite{brukhim2022characterization} that we state as an interface in
\cref{app:list} and verify by hand at \(D=2\).
\end{remark}

\section{Discussion}
\label{sec:discussion}

\paragraph{What transfers from the binary argument.}
The construction of \cite{mathiasen2026optimal} contains two conceptually
different ingredients: a binary edge-isoperimetric mechanism, and a
coefficient-one comparison whose fluctuation is localized by the comparator
error.  Only the second is used here.  The multiclass label geometry is first
compressed by the cover and menu of
\cite[Sections~3.1--3.3]{cohen2025natarajan}, with the realizable parameter
sharpened to \(\dDS\) using \cite{pabbaraju2026optimal}; once the effective
loss is binary, relative compression supplies the first-order comparison.
This separation is useful:
it suggests that an optimistic bound need not reproduce the combinatorial
geometry of the problem that originally revealed it.

\paragraph{The two dimensions have different statistical roles.}
In \cref{eq:main-bound}, \(\dDS\) controls the probability that a label on
which the oracle is correct never enters the learned menu.  This is a
realizable-style coverage event and therefore contributes linearly in
\(1/n\).  Conditional on the menu, \(\dN\) controls the description length of
the final empirical-risk-dominating rule.  Its contribution is multiplied by
\(\Lstar\) inside the square root.  The decomposition
\eqref{eq:menu-to-risk-sketch} is what keeps these roles separate; charging
all labels outside the menu would mix oracle noise into the coverage term and
destroy coefficient one.  \Cref{thm:lower} certifies that the separation is
a property of the problem and not of the proof: the same two structures can
be made to charge any learner separately, in one class and one
distribution.

\paragraph{Limitations.}
First, the logarithmic factors have not been optimized.  One logarithm comes
from boosting a realizable learner into a compressor, one from enumerating the
compression cover or the final list compressor, and one from the
stability-free description union bound in
\cref{thm:relative-compression}.  The last factor is not removable for
generic agnostic compression schemes \cite{hanneke2019sharp}.  Stable
compression schemes can remove it \cite{hanneke2021stable}, but the
fixed-menu compressor used here is not known to be stable, so that route
does not follow from the present argument.

Second, the learner is information-theoretic.  For an abstract hypothesis
class, the one-inclusion and reconstruction operations can be intractable.
Representation-preserving reductions are an active route toward more
algorithmic multiclass guarantees \cite{hanneke2025representation}, but
\cref{thm:main} makes no oracle-efficiency claim.

Third, the theorem assumes a finite alphabet.  The rate is uniform in \(K\),
and the menu prevents an explicit \(\log K\) term, but the realizable theorem
used in \cref{lem:ds-re-bridge} is invoked in its finite-label form.  The
finite-DS/infinite-Graph examples motivating the rejection of a direct Graph
argument belong to the more general label-space setting; we do not claim an
infinite-label extension here.

\paragraph{Open questions.}
Three questions seem especially concrete.
\begin{enumerate}
  \item Can one obtain
  \(\sqrt{\Lstar\dN/n}+\dDS/n\) without any logarithmic loss by jointly
  sharpening the boosted cover and the menu compressor?  Stability alone can
  address only the final description-counting factor.
  \item \Cref{thm:lower} is stated in expectation and leaves the sharp
  \(\log(1/\delta)\) dependence at fixed \(\Lstar\) open; the binary
  small-error theory \cite{asilis2025smallerror} suggests the target
  form.
  \item Can the three independent blocks be replaced by cross-fitting while
  retaining coefficient one and a learner independent of \(\Lstar\) and
  \(\delta\)?
\end{enumerate}
The relative-compression theorem is intentionally modular: any improvement to
the menu construction or its final description length immediately sharpens
the corresponding term in \cref{eq:main-bound}.

\paragraph{AI Disclosure.}
The authors acknowledge the use of GPT-5.6 Solar and Claude Fable~5 as
assistive tools in preparing this manuscript.  The project began with our
reading of the binary optimal-rate theory of \cite{mathiasen2026optimal};
the question of the multiclass analogue, the choice to route it through the
cover--menu--compression architecture, and the design of the lower-bound
construction were developed in extended discussions with both models, which
also supported the drafting and polishing of the exposition and the
detailed proof steps.  The Lean~4 formalization was written by
Claude Fable~5.  All AI-assisted material,
including every theorem statement, proof, and formalization choice, was
reviewed, edited, and validated by the authors, who take full
responsibility for the final manuscript and its results.

\bibliographystyle{alphainit} 
\bibliography{refs}   

\newcommand{\etalchar}[1]{$^{#1}$}
\begin{thebibliography}{BDMM23}

\bibitem[ACSZ23]{adenali2023optimal}
Ishaq Aden-Ali, Yeshwanth Cherapanamjeri, Abhishek Shetty, and Nikita
  Zhivotovskiy.
\newblock Optimal {PAC} bounds without uniform convergence.
\newblock In {\em Proceedings of the 64th Annual IEEE Symposium on Foundations
  of Computer Science}, pages 1203--1223, 2023.

\bibitem[ADD{\etalchar{+}}24]{asilis2024regularization}
Julian Asilis, Siddartha Devic, Shaddin Dughmi, Vatsal Sharan, and Shang-Hua
  Teng.
\newblock Regularization and optimal multiclass learning.
\newblock In {\em Proceedings of the 37th Conference on Learning Theory},
  volume 247 of {\em Proceedings of Machine Learning Research}, pages 260--310,
  2024.

\bibitem[AHHM21]{alon2022partial}
Noga Alon, Steve Hanneke, Ron Holzman, and Shay Moran.
\newblock A theory of {PAC} learnability of partial concept classes.
\newblock In {\em Proceedings of the 62nd Annual IEEE Symposium on Foundations
  of Computer Science}, pages 658--671, 2021.

\bibitem[AHV25]{asilis2025smallerror}
Julian Asilis, Mikael~M{\o}ller H{\o}gsgaard, and Grigoris Velegkas.
\newblock On agnostic {PAC} learning in the small error regime.
\newblock In {\em Advances in Neural Information Processing Systems 38}, 2025.
\newblock arXiv:2502.09496.

\bibitem[BBM05]{bartlett2005local}
Peter~L. Bartlett, Olivier Bousquet, and Shahar Mendelson.
\newblock Local {Rademacher} complexities.
\newblock {\em The Annals of Statistics}, 33(4):1497--1537, 2005.

\bibitem[BCD{\etalchar{+}}22]{brukhim2022characterization}
Nataly Brukhim, Daniel Carmon, Irit Dinur, Shay Moran, and Amir Yehudayoff.
\newblock A characterization of multiclass learnability.
\newblock In {\em Proceedings of the 63rd Annual IEEE Symposium on Foundations
  of Computer Science}, pages 943--955, 2022.

\bibitem[BCHL95]{ben1995characterizations}
Shai Ben-David, Nicol{\`o} Cesa-Bianchi, David Haussler, and Philip~M. Long.
\newblock Characterizations of learnability for classes of
  $\{0,\ldots,n\}$-valued functions.
\newblock {\em Journal of Computer and System Sciences}, 50(1):74--86, 1995.

\bibitem[BDMM23]{brukhim2023boosting}
Nataly Brukhim, Amit Daniely, Yishay Mansour, and Shay Moran.
\newblock Multiclass boosting: Simple and intuitive weak learning criteria.
\newblock In {\em Advances in Neural Information Processing Systems 36}, pages
  1403--1425, 2023.

\bibitem[BLM13]{boucheron2013concentration}
St{\'e}phane Boucheron, G{\'a}bor Lugosi, and Pascal Massart.
\newblock {\em Concentration Inequalities: A Nonasymptotic Theory of
  Independence}.
\newblock Oxford University Press, Oxford, 2013.

\bibitem[CEH{\etalchar{+}}26]{cohen2025natarajan}
Alon Cohen, Liad Erez, Steve Hanneke, Tomer Koren, Yishay Mansour, Shay Moran,
  and Qian Zhang.
\newblock Sample complexity of agnostic multiclass classification: {Natarajan}
  dimension strikes back.
\newblock In {\em Proceedings of the 58th Annual ACM Symposium on Theory of
  Computing}, 2026.
\newblock arXiv:2511.12659.

\bibitem[CP23]{charikar2023list}
Moses Charikar and Chirag Pabbaraju.
\newblock A characterization of list learnability.
\newblock In {\em Proceedings of the 55th Annual ACM Symposium on Theory of
  Computing}, 2023.

\bibitem[DL95]{devroye1995lower}
Luc Devroye and G{\'a}bor Lugosi.
\newblock Lower bounds in pattern recognition and learning.
\newblock {\em Pattern Recognition}, 28(7):1011--1018, 1995.

\bibitem[DMY16]{david2016supervised}
Ofir David, Shay Moran, and Amir Yehudayoff.
\newblock Supervised learning through the lens of compression.
\newblock In {\em Advances in Neural Information Processing Systems 29}, pages
  2784--2792, 2016.

\bibitem[DS14]{daniely2014optimal}
Amit Daniely and Shai Shalev-Shwartz.
\newblock Optimal learners for multiclass problems.
\newblock In {\em Proceedings of the 27th Conference on Learning Theory},
  volume~35 of {\em Proceedings of Machine Learning Research}, pages 287--316,
  2014.

\bibitem[DSBS15]{daniely2015multiclass}
Amit Daniely, Sivan Sabato, Shai Ben-David, and Shai Shalev-Shwartz.
\newblock Multiclass learnability and the {ERM} principle.
\newblock {\em Journal of Machine Learning Research}, 16:2377--2404, 2015.

\bibitem[EHKV89]{ehrenfeucht1989lower}
Andrzej Ehrenfeucht, David Haussler, Michael Kearns, and Leslie Valiant.
\newblock A general lower bound on the number of examples needed for learning.
\newblock {\em Information and Computation}, 82(3):247--261, 1989.

\bibitem[FW95]{floyd1995sample}
Sally Floyd and Manfred Warmuth.
\newblock Sample compression, learnability, and the {Vapnik}--{Chervonenkis}
  dimension.
\newblock {\em Machine Learning}, 21(3):269--304, 1995.

\bibitem[GHS05]{graepel2005pac}
Thore Graepel, Ralf Herbrich, and John Shawe-Taylor.
\newblock {PAC}-{Bayesian} compression bounds on the prediction error of
  learning algorithms for classification.
\newblock {\em Machine Learning}, 59(1--2):55--76, 2005.

\bibitem[GKN17]{gottlieb2017semimetrics}
Lee-Ad Gottlieb, Aryeh Kontorovich, and Pinhas Nisnevitch.
\newblock Nearly optimal classification for semimetrics.
\newblock {\em Journal of Machine Learning Research}, 18(37):1--22, 2017.

\bibitem[HK19]{hanneke2019sharp}
Steve Hanneke and Aryeh Kontorovich.
\newblock A sharp lower bound for agnostic learning with sample compression
  schemes.
\newblock In {\em Proceedings of the 30th International Conference on
  Algorithmic Learning Theory}, volume~98 of {\em Proceedings of Machine
  Learning Research}, pages 489--505, 2019.

\bibitem[HK21]{hanneke2021stable}
Steve Hanneke and Aryeh Kontorovich.
\newblock Stable sample compression schemes: New applications and an optimal
  {SVM} margin bound.
\newblock In {\em Proceedings of the 32nd International Conference on
  Algorithmic Learning Theory}, volume 132 of {\em Proceedings of Machine
  Learning Research}, pages 697--721, 2021.

\bibitem[HLW94]{haussler1994predicting}
David Haussler, Nick Littlestone, and Manfred~K. Warmuth.
\newblock Predicting $\{0,1\}$-functions on randomly drawn points.
\newblock {\em Information and Computation}, 115(2):248--292, 1994.

\bibitem[HLZ24]{hanneke2024revisiting}
Steve Hanneke, Kasper~Green Larsen, and Nikita Zhivotovskiy.
\newblock Revisiting agnostic {PAC} learning.
\newblock In {\em Proceedings of the 65th Annual IEEE Symposium on Foundations
  of Computer Science}, 2024.

\bibitem[HMMS26]{hanneke2026sauer}
Steve Hanneke, Qinglin Meng, Shay Moran, and Amirreza Shaeiri.
\newblock An optimal {Sauer} lemma over $k$-ary alphabets, 2026.
\newblock arXiv:2604.12952.

\bibitem[HMS25]{hanneke2025representation}
Steve Hanneke, Qinglin Meng, and Amirreza Shaeiri.
\newblock Representation preserving multiclass agnostic to realizable
  reduction.
\newblock In {\em Proceedings of the 42nd International Conference on Machine
  Learning}, volume 267 of {\em Proceedings of Machine Learning Research},
  pages 21995--22008, 2025.

\bibitem[HMZ24]{hanneke2024improved}
Steve Hanneke, Shay Moran, and Qian Zhang.
\newblock Improved sample complexity for multiclass {PAC} learning.
\newblock In {\em Advances in Neural Information Processing Systems 37}, 2024.

\bibitem[MQZ26]{mathiasen2026optimal}
Markus~Engelund Mathiasen, Jian Qian, and Nikita Zhivotovskiy.
\newblock An optimal agnostic {PAC} algorithm, 2026.
\newblock arXiv:2608.06363.

\bibitem[Nat89]{natarajan1989learning}
Balas~K. Natarajan.
\newblock On learning sets and functions.
\newblock {\em Machine Learning}, 4(1):67--97, 1989.

\bibitem[Pab26]{pabbaraju2026optimal}
Chirag Pabbaraju.
\newblock The optimal sample complexity of multiclass and list learning, 2026.
\newblock arXiv:2604.24749.

\bibitem[Pan02]{panchenko2002some}
Dmitriy Panchenko.
\newblock Some extensions of an inequality of {Vapnik} and {Chervonenkis}.
\newblock {\em Electronic Communications in Probability}, 7:55--65, 2002.

\bibitem[RBR06]{rubinstein2006shifting}
Benjamin I.~P. Rubinstein, Peter~L. Bartlett, and J.~Hyam Rubinstein.
\newblock Shifting, one-inclusion mistake bounds and tight multiclass expected
  risk bounds.
\newblock In {\em Advances in Neural Information Processing Systems 19}, 2006.

\bibitem[SST10]{srebro2010optimistic}
Nathan Srebro, Karthik Sridharan, and Ambuj Tewari.
\newblock Optimistic rates for learning with a smooth loss, 2010.
\newblock arXiv:1009.3896. Conference version: \emph{Smoothness, low noise and
  fast rates}, NeurIPS 2010.

\bibitem[Tsy09]{tsybakov2009introduction}
Alexandre~B. Tsybakov.
\newblock {\em Introduction to Nonparametric Estimation}.
\newblock Springer Series in Statistics. Springer, New York, 2009.

\bibitem[ZH18]{zhivotovskiy2018localization}
Nikita Zhivotovskiy and Steve Hanneke.
\newblock Localization of {VC} classes: Beyond local {Rademacher} complexities.
\newblock {\em Theoretical Computer Science}, 742:27--49, 2018.

\end{thebibliography}

\clearpage
\appendix
\section{Proofs for the upper bound}
\label{sec:appendix}

\subsection{Imported modules}
\label{app:toolkit-statements}

This subsection states every external result used in the proof in the exact
form in which it is applied.  The statements are collected here so that the
main text can follow the causal order of the argument; no proof step depends
on an unstated result.  Each source locator, convention change, and
applicability check is recorded in the proof immediately following the
corresponding statement.  \Cref{fig:dependencies} displays the dependency
structure of the upper-bound proof and separates the imported modules from
the statements proved in this paper; the lower bound of \cref{thm:lower}
is handled separately in \cref{app:lower}.

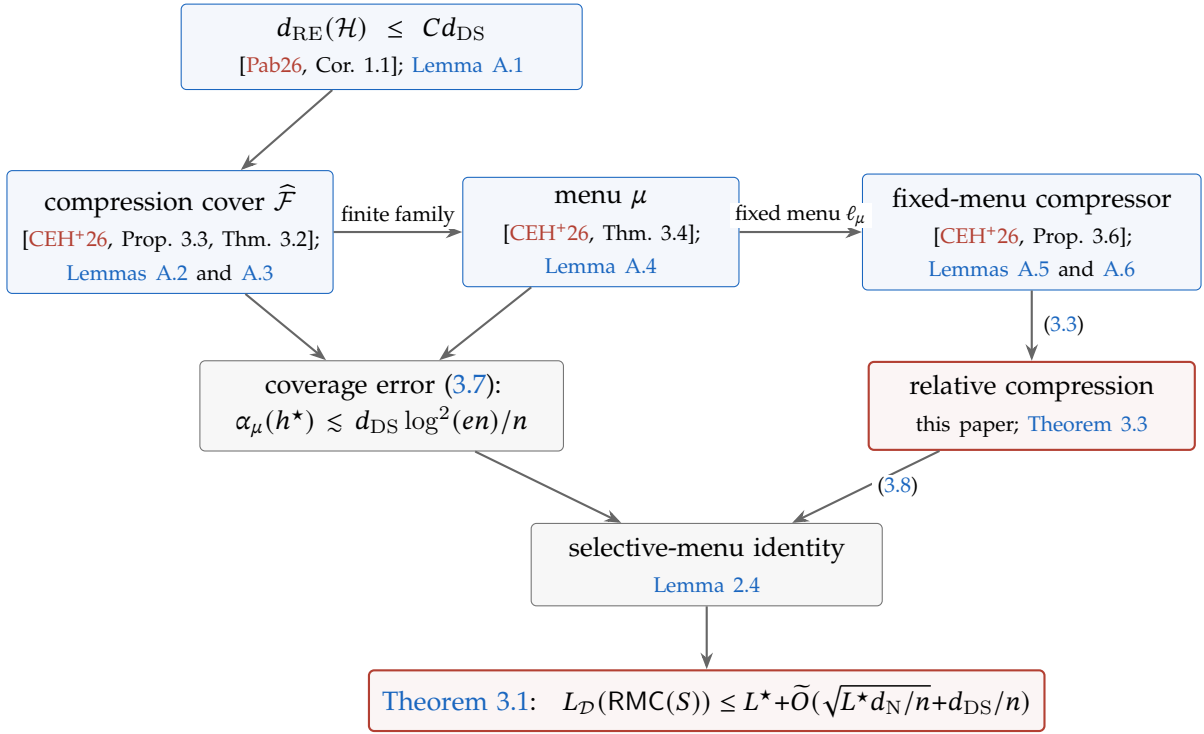
\begin{figure}[H]
  \centering
  \begin{tikzpicture}[
      imp/.style={draw=LinkColor, rounded corners=2pt, fill=LinkColor!5,
        align=center, inner sep=5pt, font=\small},
      new/.style={draw=CiteColor, thick, rounded corners=2pt,
        fill=CiteColor!5, align=center, inner sep=5pt, font=\small},
      pass/.style={draw=black!55, rounded corners=2pt, fill=black!3,
        align=center, inner sep=5pt, font=\small},
      arr/.style={-{Stealth[length=2.2mm]}, thick, draw=black!60},
      lbl/.style={font=\scriptsize, fill=white, inner xsep=2pt,
        inner ysep=.5pt}]
    \node[imp, text width=.30\textwidth] (pab) at (2.8,2.45) {
      \(\dRE(\cH)\le C\dDS\)\\[1pt]
      {\scriptsize \cite[Cor.~1.1]{pabbaraju2026optimal};
        \cref{lem:ds-re-bridge}}};
    \node[imp, text width=.24\textwidth] (cover) at (0,0) {
      compression cover \(\widehat\cF\)\\[1pt]
      {\scriptsize \cite[Prop.~3.3, Thm.~3.2]{cohen2025natarajan};
        \cref{lem:ordinary-compression,lem:compression-cover}}};
    \node[imp, text width=.20\textwidth] (menu) at (5.7,0) {
      menu \(\mu\)\\[1pt]
      {\scriptsize \cite[Thm.~3.4]{cohen2025natarajan};
        \cref{lem:mw-menu}}};
    \node[imp, text width=.25\textwidth] (comp) at (11.4,0) {
      fixed-menu compressor\\[1pt]
      {\scriptsize \cite[Prop.~3.6]{cohen2025natarajan};
        \cref{lem:fixed-menu-compression,lem:menu-finalizer}}};
    \node[pass, text width=.27\textwidth] (alpha) at (2.8,-2.3) {
      coverage error \eqref{eq:alpha-sketch}:\\[1pt]
      \(\alpha_\mu(h^\star)\lesssim\dDS\log^2(en)/n\)};
    \node[new, text width=.24\textwidth] (rel) at (11.4,-2.3) {
      relative compression\\[1pt]
      {\scriptsize this paper; \cref{thm:relative-compression}}};
    \node[pass, text width=.26\textwidth] (id) at (7.1,-4.4) {
      selective-menu identity\\[1pt]
      {\scriptsize \cref{lem:menu-calculus}}};
    \node[new, text width=.52\textwidth] (main) at (7.1,-6.2) {
      \cref{thm:main}:\quad
      \(\Risk_{\cD}(\RMC(S))\le\Lstar
        +\widetilde O(\sqrt{\Lstar\dN/n}+\dDS/n)\)};
    \draw[arr] (pab) -- (cover);
    \draw[arr] (cover) -- node[lbl, above=1pt] {finite family} (menu);
    \draw[arr] (menu) -- node[lbl, above=1pt] {fixed menu \(\ell_\mu\)}
      (comp);
    \draw[arr] (cover) -- (alpha);
    \draw[arr] (menu) -- (alpha);
    \draw[arr] (comp) -- node[lbl, right=2pt]
      {\eqref{eq:relative-premise}} (rel);
    \draw[arr] (alpha) -- (id);
    \draw[arr] (rel) -- node[lbl, right=2pt]
      {\eqref{eq:menu-relative-sketch}} (id);
    \draw[arr] (id) -- (main);
  \end{tikzpicture}
  \caption{Dependency structure of the proof.  Blue boxes are imported
  modules: each carries its source locator and the lemma below in which it
  is restated with its convention checks.  Red boxes are proved in this
  paper.  Gray boxes are the intermediate estimates and the pointwise
  identity that composes them.  The concentration input to
  \cref{thm:relative-compression} is the one-sided Bernstein inequality of
  \cref{lem:bernoulli-bernstein}
  \cite[Sections~2.4 and~2.8]{boucheron2013concentration}.}
  \label{fig:dependencies}
\end{figure}

\paragraph{Realizable dimension.}
For a learner \(B\), let
\[
 \mathcal E^{\mathrm{RE}}_{B,\cH}(m)
 =
 \sup_{\substack{\cD:\,\inf_{h\in\cH}\Risk_{\cD}(h)=0}}
 \E_{S\sim\cD^m,B}\Risk_{\cD}(B(S))
\]
be its worst-case expected error over \(\cH\)-realizable distributions; the
expectation includes the learner's internal randomness.  Following
\cite[Definition~2.8]{cohen2025natarajan}, define
\[
 \dRE(\cH,r)
 =
 \inf_B\inf\{m\in\N:\mathcal E^{\mathrm{RE}}_{B,\cH}(m)\le r\},
 \qquad
 \dRE(\cH)=\dRE(\cH,1/(9e)).
\]

\begin{lemma}[DS-to-realizable-dimension bridge]
\label{lem:ds-re-bridge}
Let \(\cH\subseteq[K]^\cX\) have \(1\le\dDS<\infty\).  Then
\[
 \dRE(\cH)
 \le
 \left\lceil 9.64(18e)
 \bigl(\dDS+\log(36e)\bigr)\right\rceil
 \le C\dDS
\]
for a universal constant \(C\).
\end{lemma}

\begin{proof}
\cite[Corollary~1.1]{pabbaraju2026optimal} gives, for every finite-label
class of DS dimension \(\dDS\), a realizable learner whose error is at most
\(\epsilon_0\), with probability at least \(1-\delta_0\), whenever the
sample size is at least
\[
 9.64\frac{\dDS+\log(2/\delta_0)}{\epsilon_0}
\]
observations.  Take \(\epsilon_0=\delta_0=1/(18e)\).  Since the loss is at
most one, the learner's expected error is at most
\[
 (1-\delta_0)\epsilon_0+\delta_0
 \le \epsilon_0+\delta_0=\frac1{9e}.
\]
This is exactly the expected-error threshold in
\cite[Definition~2.8]{cohen2025natarajan}, including the learner's internal
randomness.  Substituting
\(\log(2/\delta_0)=\log(36e)\) and
\(1/\epsilon_0=18e\), then taking the ceiling, proves the first inequality.
The second follows because \(\log(36e)\) is a numerical constant and
\(\dDS\ge1\).
\end{proof}

For a deterministic selection scheme \(A=(\kappa,\rho)\) defined on samples
of every finite length, write
\[
 k_A(N)
 \defeq
 \sup_{1\le m\le N}\ \sup_{s\in\cZ^m}|\kappa(s)|
\]
for its size function.

\begin{lemma}[Ordinary compression from realizable dimension]
\label{lem:ordinary-compression}
Let \(\cH\subseteq[K]^\cX\) have finite DS dimension and
\(\dRE(\cH)<\infty\).  There is a measurable deterministic zero--one
empirical-risk-dominating compression scheme \(A_1\), defined on all finite
samples, whose size function satisfies
\[
 k_1(m)\le C(\dRE(\cH)+1)\log(em),
 \qquad m\ge1.
\]
Consequently, under the hypotheses of \cref{lem:ds-re-bridge},
\(k_1(m)\le C\dDS\log(em)\).
\end{lemma}

\begin{proof}
\cite[Equation~(17) and Lemma~C.1 in Appendices~B--C]
{cohen2025natarajan} compare the optimal expected-error rate encoded by
\(\dRE\) with the deterministic one-inclusion density.  Proposition~3.3
of the same paper then boosts the resulting
one-inclusion rule into a deterministic ordinary compression scheme, in the
boosting-to-compression tradition of \cite{david2016supervised}.  In
the source's notation, the construction is defined for every finite sample,
its reconstruction has no larger empirical zero--one loss than any
comparator in \(\cH\), and its size obeys
\[
 k_1(m)\le C(\dRE(\cH)+1)\log(em),
 \qquad m\ge1.
\]
Thus it is precisely an empirical-risk-dominating selection scheme in the
sense used here.  Two degenerate branches deserve explicit conventions.  On
the empty sample the message is empty, and we fix
\(\rho(\varnothing)=h^\dagger\) for an arbitrary fixed classifier
\(h^\dagger\).  If a nonempty sample \(s\) contains no observation on which
any hypothesis of \(\cH\) is correct, then
\(\inf_{h\in\cH}\Rhat_s(h)=1\), so every reconstruction, in particular the
empty-message one, satisfies the domination requirement automatically.
Neither branch affects the size bound.  The measurability clause is covered
by \cref{asm:measurable}.  Finally, \cref{lem:ds-re-bridge} and \(\dDS\ge1\)
give \(\dRE(\cH)+1\le C\dDS\), and hence
\(k_1(m)\le C\dDS\log(em)\).
\end{proof}

For a labeled sequence \(s=((x_i,y_i))_{i=1}^m\), \(m\ge1\), and a
classifier \(h\), let \(s[h]\) be the ordered subsequence of examples for
which \(h(x_i)=y_i\).  For \(j\ge0\), let \(s^j\) be the ordered
\(j\)-tuples drawn from the entries of \(s\), with repetition allowed;
\(s^0\) contains only the empty tuple.

\begin{lemma}[Compression cover]
\label{lem:compression-cover}
Let \(A=(\kappa,\rho)\) be a measurable deterministic ordinary compression
scheme for \(\cH\), with size function \(k_A\).  Define
\[
 \CompCover(s,\cH,A)
 =
 \left\{\rho(t):
 t\in\bigcup_{j=0}^{k_A(m)}s^j\right\},
 \qquad
 f_{h,s}=A(s[h]).
\]
Then \(f_{h,s}\in\CompCover(s,\cH,A)\), and
\[
 |\CompCover(s,\cH,A)|
 \le
 \sum_{j=0}^{k_A(m)}m^j
 \le
 (m+1)^{k_A(m)+1}.
\]
For every distribution \(\cD\), every fixed \(h\in\cH\), and every
\(\eta\in(0,1)\), if \(S\sim\cD^m\), then with probability at least
\(1-\eta\),
\[
 \P(h(X)=Y\ne f_{h,S}(X))
 \le
 C\frac{(k_A(m)+1)\log(em)+\log(1/\eta)}{m}.
\]
Under the hypotheses of \cref{lem:ds-re-bridge}, choose \(A=A_1\) from
\cref{lem:ordinary-compression} and, for \(S\sim\cD^m\), set
\(\widehat\cF=\CompCover(S,\cH,A_1)\).  For every realization of \(S\),
\[
 \log|\widehat\cF|
 \le C\dDS\log^2(em).
\tag{S.1}\label{eq:cover-size}
\]
Moreover, for each fixed \(h\in\cH\), with probability at least
\(1-\eta\),
\[
 \P(h(X)=Y\ne f_{h,S}(X))
 \le
 C\frac{\dDS\log^2(em)+\log(1/\eta)}{m}.
\tag{S.2}\label{eq:cover-miss}
\]
The family and the canonical witness \(f_{h,S}=A_1(S[h])\) are measurable
functions of \(S\) under \cref{asm:measurable}.
\end{lemma}

\begin{proof}
Fix \(s\) and \(h\).  The message \(\kappa(s[h])\) is an ordered sequence
of entries of \(s\).  If \(s[h]\) is empty, the selection property forces
this message to be empty as well.  If \(s[h]\) is nonempty, the definition
of the size function gives
\[
 |\kappa(s[h])|
 \le k_A(|s[h]|)
 \le k_A(m).
\]
Thus, in both cases, \(|\kappa(s[h])|\le k_A(m)\).
Consequently
\(f_{h,s}=\rho(\kappa(s[h]))\) belongs to the displayed cover.  There are
at most \(m^j\) ordered messages of length \(j\), so
\[
 |\CompCover(s,\cH,A)|
 \le \sum_{j=0}^{k_A(m)}m^j
 \le (m+1)^{k_A(m)+1}.
\]
The last envelope follows, for example, by comparing with the binomial
expansion of \((m+1)^{k_A(m)+1}\); unlike the shorthand
\(m^{k_A(m)+1}\), it remains valid at \(m=1\).

\cite[Algorithm~2 and Theorem~3.2]{cohen2025natarajan} apply this very
witness \(A(S[h])\) and give, pointwise for every fixed comparator \(h\),
\[
 \P(h(X)=Y\ne f_{h,S}(X))
 \le
 C\frac{(k_A(m)+1)\log(em)+\log(1/\eta)}{m}
\]
with probability at least \(1-\eta\).  This is a correct-region guarantee:
no claim is made on examples on which \(h(X)\ne Y\).

Now take \(A=A_1\).  The preceding cardinality bound and
\cref{lem:ordinary-compression} yield
\[
 \log|\widehat\cF|
 \le (k_1(m)+1)\log(m+1)
 \le C\dDS\log^2(em),
\]
which is \eqref{eq:cover-size}; substituting the same size estimate into the
source generalization bound gives \eqref{eq:cover-miss}.  Finally,
\cref{asm:measurable} makes the compression map and reconstruction
measurable.  Taking ordered subsequences and forming the finite family
therefore preserves measurability, both for \(\widehat\cF\) and for the
canonical witness \(f_{h,S}\).
\end{proof}

\begin{lemma}[Finite-family menu]
\label{lem:mw-menu}
Let \(\mathcal G\) be a sigma-field.  Conditional on \(\mathcal G\), let
\(\cF_0\) be a fixed nonempty finite family and let \(f\in\cF_0\) be fixed.
Let \(T\in\N\), \(T\ge1\), let \(Z_1,\ldots,Z_T\sim\cD\) be iid and
independent of \(\mathcal G\), and let the random seed be independent of
\(\sigma(\mathcal G,Z_1,\ldots,Z_T)\).  Run
\cite[Algorithm~3]{cohen2025natarajan} with learning rate \(1/2\) on these
observations.  Its classifiers \(h_1,\ldots,h_{T-1}\) define
\[
 \mu(x)=\{h_1(x),\ldots,h_{T-1}(x)\},
 \qquad
 \sup_x|\mu(x)|\le T-1.
\]
For every \(\eta\in(0,1)\), with conditional probability at least
\(1-\eta\),
\[
 \P(f(X)=Y,\ Y\notin\mu(X))
 \le
 \frac{4\log|\cF_0|+14\log(3/\eta)+12}{T}.
\tag{S.3}\label{eq:menu-miss}
\]
The probability is over the fresh observations and the algorithm's random
seed; the assertion is pointwise in the comparator \(f\).
\end{lemma}

\begin{proof}
After conditioning on \(\mathcal G\), both \(\cF_0\) and the comparator
\(f\) are fixed, while the observations and the algorithmic seed retain the
independence required by the source theorem.
\cite[Algorithm~3 and Theorem~3.4]{cohen2025natarajan}, specialized to
learning rate \(1/2\), run multiplicative weights for \(T\) iid rounds on
this finite family.  The output classifiers are
\(h_1,\ldots,h_{T-1}\), so their pointwise union satisfies
\(\sup_x|\mu(x)|\le T-1\).  For every comparator fixed before those rounds,
the theorem gives
\[
 \P(f(X)=Y,\ Y\notin\mu(X))
 \le
 \frac{4\log|\cF_0|+14\log(3/\eta)+12}{T}
\]
with probability at least \(1-\eta\).  This is exactly
\eqref{eq:menu-miss}.  Since the theorem is applied on each conditional
probability space, its probability is only over the fresh observations and
the independent seed, as claimed.
\end{proof}

\begin{lemma}[Fixed-menu loss compression]
\label{lem:fixed-menu-compression}
Let \(\mu\) be a fixed measurable \(p\)-menu with
\(p\in\N\), \(p\ge2\), and let \(\cH\) have Natarajan dimension
\(\dN\ge1\).  For every \(N\in\N\), \(N\ge1\), there is a measurable
deterministic empirical-risk-dominating compression scheme, denoted
\(\LSCS(\mu,\cH,p)\), under \(\ell_\mu\), whose size satisfies
\[
 k_\mu(N)\le C\dN\log(p)\log(eN).
\tag{S.4}\label{eq:menu-compression-size}
\]
Here \(p\) may be any deterministic upper bound on
\(\sup_x|\mu(x)|\), rather than its exact value.
\end{lemma}

\begin{proof}
\cite[Proposition~3.6 and Algorithm~8]{cohen2025natarajan} construct, for
every fixed \(p\)-menu with \(p\ge2\), an
\(\ell_\mu\)-sample-compression scheme of size
\(C\dN\log(p)\log(eN)\).  Their defining condition says that the
reconstructed classifier's empirical \(\ell_\mu\)-loss is no larger than
that of any comparator in \(\cH\); this is exactly empirical domination in
\eqref{eq:emp-domination}.  The construction is deterministic once \(\mu\)
is fixed, and measurability is supplied by \cref{asm:measurable}.

Inspecting the source construction shows that the parameter \(p\) enters in
exactly two places, and both are monotone.  Algorithmically, \(p\) sets the
subsample length \(m=\lfloor 60\dN\log p\rfloor\) of the weak learners that
are majority-voted; a larger \(p\) only lengthens these subsamples, which
preserves each weak guarantee.  Quantitatively, with the source's
constants, a weak learner run on subsamples of length
\(m=\lfloor 60\dN\log p\rfloor\) over a menu of true width
\(q\in[2,p]\) has expected error at most
\[
 \frac{20\dN\log q}{m+1}
 \le
 \frac{20\dN\log p}{60\dN\log p}
 =\frac13,
\]
which is exactly the threshold the boosting step consumes; true widths
below two only shrink the restricted trace further.  Analytically, the
proof uses \(p\) only
through the pointwise inequality \(|\mu(x)|\le p\) on the restricted trace,
which any upper bound satisfies a fortiori.  Thus running the scheme with a
deterministic upper bound in place of the exact supremum leaves every step
of the source proof valid and only loosens the stated size
\eqref{eq:menu-compression-size}, which is the direction in which it is
consumed.  The displayed Algorithm~8 contains a term singular at \(p=1\)
(the subsample length vanishes), which explains the restriction \(p\ge2\);
the size-zero branch in \cref{lem:menu-finalizer} handles the empty- and
singleton-menu cases used in our algorithm.

As in \cref{lem:ordinary-compression}, we fix the degenerate branch by
convention: if no hypothesis of \(\cH\) is correct on any inside-menu
observation of \(s\), the selector returns the empty message and the
reconstruction is a fixed classifier.  Every comparator then incurs
\(\ell_\mu\)-loss on every observation with \(y_i\in\mu(x_i)\), so
\(\inf_{h\in\cH}\Rhat^{\ell_\mu}_s(h)\) equals the fraction of such
observations, while the \(\ell_\mu\)-loss of any reconstruction is at most
that fraction; empirical domination on this branch is automatic.
\end{proof}

The source compression interface begins at \(p=2\), whereas the learned menu
may be empty or a singleton when the second block is small.  The following
internal lemma closes this edge case without changing the algorithmic or
statistical argument.

\begin{lemma}[Menu finalizer]
\label{lem:menu-finalizer}
Let \(\mu\) be a fixed measurable menu and let \(\cH\) have Natarajan
dimension \(\dN\ge1\).  Suppose \(T\in\N\), \(T\ge1\), and
\(\sup_x|\mu(x)|\le T-1\).  If \(T\le2\), define \(g_\mu(x)\) to return the
unique member of \(\mu(x)\) when present and a fixed label otherwise, and let
\(B_{\mu,T}\) be the size-zero selection scheme with empty message and
reconstruction \(\rho(\varnothing)=g_\mu\).  If \(T\ge3\), let
\(B_{\mu,T}=\LSCS(\mu,\cH,T-1)\).  Then, for every \(N\in\N\), \(N\ge1\),
the scheme \(B_{\mu,T}\) is empirical-risk-dominating under \(\ell_\mu\).
Its compression size on \(N\) observations is zero when \(T\le2\), and is
at most
\[
 C\dN\log(T-1)\log(eN)
\]
when \(T\ge3\).
\end{lemma}

\begin{proof}
If \(T\le2\), every menu is empty or a singleton.  Whenever
\(y\in\mu(x)\), its unique member is \(y\), so
\(\ell_\mu(g_\mu,(x,y))=0\) pointwise.  Thus \(g_\mu\) is a size-zero
empirical-risk-dominating selection scheme with the stated empty message.
If \(T\ge3\), use \cref{lem:fixed-menu-compression} with the deterministic
upper bound \(p=T-1\ge2\).
\end{proof}

\begin{lemma}[One-sided Bernoulli Bernstein inequality]
\label{lem:bernoulli-bernstein}
Let \(W_1,\ldots,W_m\) be iid Bernoulli random variables with mean \(p\),
let \(\overline W=m^{-1}\sum_{i=1}^m W_i\), and let \(u>0\).  Each of the
inequalities
\[
 \overline W
 \le p+\sqrt{\frac{2pu}{m}}+\frac{u}{3m},
 \qquad
 p
 \le \overline W+\sqrt{\frac{2pu}{m}}+\frac{u}{3m}
\]
holds with probability at least \(1-e^{-u}\).
\end{lemma}

\begin{proof}
We use Bernstein's inequality in its sub-gamma tail form: if
\(X_1,\ldots,X_m\) are independent, centered, bounded above by \(1\), with
\(\sum_i\E X_i^2\le v\), then
\(\sum_iX_i\le\sqrt{2vu}+u/3\) with probability at least \(1-e^{-u}\)
\cite[Sections~2.4 and~2.8]{boucheron2013concentration}.  Apply this
separately to the sums \(\sum_{i=1}^m(W_i-p)\) and \(\sum_{i=1}^m(p-W_i)\).
In both cases the summands are centered, bounded above by one, and have
total variance \(mp(1-p)\le mp\), so \(v=mp\) is admissible.  Dividing the
resulting bounds by \(m\) proves the two claims.
\end{proof}

\subsection{Proof of the relative compression theorem}
\label{app:relative}

\ThmRelative*
\begin{proof}
Write \(L=L_{\cD}^\ell(h)\) and \(\Gamma=\Gamma_N(k,\delta)\).  If
\(\Gamma\ge1/20\), the claim is immediate because
\(20\Gamma\ge1\) and the loss is at most one.  We henceforth assume
\(\Gamma<1/20\).  In particular, \(N\ge2\) and
\[
 \frac{k}{N}\le\frac{\Gamma}{\log2}<0.073.
\tag{A.1}\label{eq:k-over-n}
\]

For every ordered index sequence \(I=(i_1,\ldots,i_j)\), \(0\le j\le k\),
let \(U_I\subseteq[N]\) be its set of distinct indices, let
\(r_I=|U_I|\), and define
\[
 f_I=\rho(Z_{i_1},\ldots,Z_{i_j}),
 \qquad
 \widetilde L_I(f_I)
 =\frac{1}{N-r_I}\sum_{i\notin U_I}\ell(f_I,Z_i).
\]
There are at most
\[
 M=\sum_{j=0}^kN^j\le(N+1)^{k+1}
\tag{A.2}\label{eq:description-count}
\]
such descriptions.  For a fixed \(I\), condition on
\((Z_i)_{i\in U_I}\).  Then \(f_I\) is fixed, while the remaining
\(N-r_I\) observations are iid.  The second inequality in
\cref{lem:bernoulli-bernstein},
integrated over the selected observations and followed by a union bound over
\eqref{eq:description-count}, shows that, with probability at least
\(1-\delta/2\), simultaneously for every \(I\),
\[
 p_I
 \le \widetilde L_I(f_I)+\sqrt{2p_Is_I}+\frac{s_I}{3},
 \quad
 p_I=L_{\cD}^\ell(f_I),
 \quad
 s_I=\frac{\log(2M/\delta)}{N-r_I}.
\tag{A.3}\label{eq:bernstein-reconstruction}
\]
Repeated indices cause no difficulty: only the distinct coordinates in
\(U_I\) are removed, while ordered sequences are still counted in \(M\).

For the realized sample, choose any index sequence \(I\) whose value sequence
equals \(\kappa(S)\).  The simultaneous event
\eqref{eq:bernstein-reconstruction} was established before this pointwise
choice, so no measurable selector among duplicate representations is needed.
Empirical domination and nonnegativity give
\[
 \widetilde L_I(f_I)
 \le\frac{N}{N-r_I}\Rhat_S^\ell(f_I)
 \le\frac{N}{N-r_I}\Rhat_S^\ell(h).
\tag{A.4}\label{eq:heldout-domination}
\]
Separately, the first inequality in \cref{lem:bernoulli-bernstein}, applied
to the fixed comparator, implies with probability at least \(1-\delta/2\),
\[
 \Rhat_S^\ell(h)
 \le
 L+\sqrt{\frac{2L\log(2/\delta)}{N}}
 +\frac{\log(2/\delta)}{3N}.
\tag{A.5}\label{eq:comparator-bernstein}
\]

On the intersection of these two events, put \(\theta=k/N\); the letter
\(u\) stays reserved for the deviation parameter of
\cref{lem:bernoulli-bernstein}.  By
\eqref{eq:k-over-n}, \(\theta<0.073\), so \((1-\theta)^{-1}\le1.08\).
Moreover, \(\log(2/\delta)\le N\Gamma\) and
\(\log(2M/\delta)\le N\Gamma\).  Set
\(q_I=\widetilde L_I(f_I)\).  Therefore
\[
 \begin{split}
 q_I
 &\le \frac{1}{1-\theta}
 \left(L+\sqrt{2L\Gamma}+\frac{\Gamma}{3}\right)\\
 &=L+\frac{\theta}{1-\theta}L
   +\frac{\sqrt{2L\Gamma}+\Gamma/3}{1-\theta}.
 \end{split}
\]
Because \(L\le1\), \(\theta\le\Gamma/\log2\), and
\((1-\theta)^{-1}\le1.08\),
the three excess terms on the last line are at most
\(1.56\Gamma\), \(1.53\sqrt{L\Gamma}\), and \(0.36\Gamma\),
respectively.  Also
\[
 s_I
 \le\frac{N\Gamma}{N-k}
 \le1.08\Gamma.
\]
Thus, after rounding constants upward,
\[
 q_I
 \le L+2\sqrt{L\Gamma}+2\Gamma,
 \qquad
 s_I\le1.08\Gamma.
\tag{A.6}\label{eq:q-s-bounds}
\]
It remains to solve the implicit inequality in
\eqref{eq:bernstein-reconstruction}.  If \(p_I\le q_I\), the desired
conclusion follows directly from \eqref{eq:q-s-bounds}.  Otherwise, put
\(d=p_I-q_I>0\).  Subadditivity of the square root and Young's inequality give
\[
 \begin{split}
 d
 &\le\sqrt{2q_Is_I}+\sqrt{2s_Id}+\frac{s_I}{3}\\
 &\le\sqrt{2q_Is_I}+\frac d2+\frac{4s_I}{3},
 \end{split}
\]
and therefore
\[
 p_I\le q_I+2\sqrt{2q_Is_I}+\frac{8s_I}{3}.
\tag{A.7}\label{eq:implicit-solved}
\]
To make the last substitution explicit, subadditivity of the square root,
\eqref{eq:q-s-bounds}, and
\[
 \sqrt{\Gamma\sqrt{L\Gamma}}
 \le\frac{\sqrt{L\Gamma}+\Gamma}{2},
\]
give
\[
 \sqrt{2q_Is_I}
 \le
 \sqrt{2.16}\left(
   \sqrt{L\Gamma}
   +\sqrt{2\Gamma\sqrt{L\Gamma}}
   +\sqrt2\,\Gamma
 \right)
 \le 2.51\sqrt{L\Gamma}+3.12\Gamma.
\]
Substitution into \eqref{eq:implicit-solved} now gives
\[
 p_I\le L+8\sqrt{L\Gamma}+12\Gamma.
\]
The looser constants \(12\) and \(20\) in
\eqref{eq:relative-bound} follow.  The two failure probabilities sum to at
most \(\delta\).
\end{proof}

\subsection{Proof of the multiclass theorem}
\label{app:main}

\ThmMain*
\begin{proof}
Let \(n_i=|S_i|\) and \(T=n_2\).  For the balanced split in
\cref{alg:rmc}, \(n_i\ge\lfloor n/3\rfloor\ge n/5\) for \(n\ge3\); all three
block sizes are therefore comparable to \(n\) up to universal constants.

Fix \(\eps>0\) and choose \(h^\star\in\cH\) such that
\[
 \Risk_{\cD}(h^\star)\le\Lstar+\eps.
\tag{A.8}\label{eq:epsilon-minimizer}
\]
Let \(A_1\) be the deterministic scheme in
\cref{lem:ordinary-compression}, let
\(\widehat\cF=\CompCover(S_1,\cH,A_1)\), and define the canonical witness
\[
 f^\star_1=f_{h^\star,S_1}=A_1(S_1[h^\star]).
\]
Both \(\widehat\cF\) and \(f^\star_1\) are
\(\sigma(S_1)\)-measurable, and \(f^\star_1\in\widehat\cF\) for every
realization of \(S_1\).  By \cref{lem:compression-cover}, with probability
at least \(1-\delta/3\), the event \(\mathcal E_1\) holds on which
\[
 \P(h^\star(X)=Y\ne f^\star_1(X))
 \le
 C\frac{\dDS\log^2(en)+\log(3/\delta)}{n}.
\tag{A.9}\label{eq:main-cover}
\]

Conditional on \(\sigma(S_1)\), the family and \(f^\star_1\) are fixed
before \(S_2\) and the independent menu seed are drawn.  Apply
\cref{lem:mw-menu} with \(\eta=\delta/3\).  The tower property turns its
conditional failure bound into the unconditional bound
\(\P(\mathcal E_2^c)\le\delta/3\), where
\[
 \beta_{f^\star_1}(\mu)
 \le
 \frac{4\log|\widehat\cF|+14\log(9/\delta)+12}{n_2}
 \le
 C\frac{\dDS\log^2(en)+\log(1/\delta)}{n}.
\]
On \(\mathcal E_1\cap\mathcal E_2\), apply
\eqref{eq:menu-miss-transfer} with \(h=h^\star\) and \(f=f^\star_1\).
Together with \eqref{eq:main-cover}, this gives
\[
 \alpha_\mu
 \defeq \P(h^\star(X)=Y,\ Y\notin\mu(X))
 \le
 C\frac{\dDS\log^2(en)+\log(1/\delta)}{n}.
\tag{A.10}\label{eq:main-alpha}
\]

Let \(R_\mu\) be an independent random seed encoding the draws in
Algorithm~3, and condition on
\(\mathcal G_2=\sigma(S_1,S_2,R_\mu)\).
Then \(\mu\) and the final rule \(B_{\mu,T}\) from
\cref{lem:menu-finalizer} are fixed, while
\(S_3\sim\cD^{n_3}\) remains iid and independent of \(\mathcal G_2\).
Write \(\widehat f=B_{\mu,T}(S_3)=\RMC(S)\).
The rule is empirical-risk-dominating under \(\ell_\mu\).  Its size
\(k_2(n_3)\) is zero if \(T\le2\); if \(T\ge3\), then
\[
 k_2(n_3)
 \le C\dN\log(T-1)\log(en_3)
 \le C\dN\log^2(en).
\tag{A.11}\label{eq:main-k2}
\]
Thus, in either branch,
\[
 \begin{split}
 \Gamma_{n_3}(k_2(n_3),\delta/3)
 &=
 \frac{(k_2(n_3)+1)\log(n_3+1)+\log(12/\delta)}{n_3}\\
 &\le
 C\frac{\dN\log^3(en)+\log(1/\delta)}{n}.
 \end{split}
\tag{A.12}\label{eq:main-gamma}
\]
Here the additive constant in \(\log(12/\delta)\) is absorbed by
\(\dN\log^3(en)\), using \(\dN\ge1\) and \(n\ge3\).

Apply \cref{thm:relative-compression} conditionally on
\(\mathcal G_2\), with comparator \(h^\star\) and failure probability
\(\delta/3\).  It produces an event \(\mathcal E_3\) satisfying
\(\P(\mathcal E_3^c\mid\mathcal G_2)\le\delta/3\) almost surely.
Taking expectations gives \(\P(\mathcal E_3^c)\le\delta/3\).  On
\(\mathcal E_3\), \eqref{eq:main-gamma} and
\eqref{eq:relative-bound} imply
\[
 \MenuLoss(\widehat f)-\MenuLoss(h^\star)
 \le C\Biggl(
 \sqrt{\frac{\MenuLoss(h^\star)(\dN\log^3(en)+\log(1/\delta))}{n}}
 +\frac{\dN\log^3(en)+\log(1/\delta)}{n}
 \Biggr).
\tag{A.13}\label{eq:main-menu-relative}
\]
Here \(\MenuLoss\) is the population loss for the realized fixed menu.
The decomposition \eqref{eq:menu-risk-algebra}, applied with
\(g=\widehat f\) and \(h=h^\star\), holds deterministically for every
realization of the menu:
\[
 \Risk_{\cD}(\widehat f)-\Risk_{\cD}(h^\star)
 \le
 \MenuLoss(\widehat f)-\MenuLoss(h^\star)
 +\alpha_\mu.
\tag{A.14}\label{eq:main-decomposition}
\]
Probability enters only through the estimates substituted into its
right-hand side.  By \eqref{eq:menu-loss-below-risk},
\(\MenuLoss(h^\star)\le\Risk_{\cD}(h^\star)\).  On
\(\mathcal E_1\cap\mathcal E_2\cap\mathcal E_3\), substituting
\eqref{eq:main-alpha} and \eqref{eq:main-menu-relative} into
\eqref{eq:main-decomposition}, then using
\eqref{eq:epsilon-minimizer}, proves \eqref{eq:main-bound} with
\(\Lstar\) replaced by \(\Lstar+\eps\).  Indeed, the three displayed
failure bounds sum to at most \(\delta\).

It remains to let \(\eps\downarrow0\): take a sequence
\(\eps_j\downarrow0\).  The learner is the same
for every \(j\)---in particular, it uses neither \(\eps_j\), \(\Lstar\), nor
\(\delta\)---and the deterministic risk thresholds obtained above decrease
to the right-hand side of \eqref{eq:main-bound}.  Continuity from above of
probability therefore yields the stated event with probability at least
\(1-\delta\).  Finally, \eqref{eq:dimensions-order} gives the simpler
\(\widetilde O\) form in \eqref{eq:intro-rate}.
\end{proof}

\section{Proofs for the lower bound}
\label{app:lower}

This appendix proves \cref{thm:lower}.  Except for the import recorded in
\cref{lem:separation}, whose four-line derivation from
\cite{brukhim2022characterization} already appears in the main text, the
argument is self-contained.  All distributions below are finitely
supported, so every infimum over the finite class \(\cH\) is attained and
no measurability conditions are required.

\Cref{fig:lower-dependencies} displays the dependency structure.

\begin{figure}[H]
  \centering
  \begin{tikzpicture}[
      imp/.style={draw=LinkColor, rounded corners=2pt, fill=LinkColor!5,
        align=center, inner sep=5pt, font=\small},
      new/.style={draw=CiteColor, thick, rounded corners=2pt,
        fill=CiteColor!5, align=center, inner sep=5pt, font=\small},
      arr/.style={-{Stealth[length=2.2mm]}, thick, draw=black!60},
      lbl/.style={font=\scriptsize, fill=white, inner xsep=2pt,
        inner ysep=.5pt}]
    \node[imp, text width=.30\textwidth] (sep) at (0,0) {
      separation pseudo-cube \(V\), \(\Nat\le1\)\\[1pt]
      {\scriptsize \cite{brukhim2022characterization};
        \cref{lem:separation}}};
    \node[new, text width=.22\textwidth] (fib) at (5.7,0) {
      fiber lemma\\[1pt]
      {\scriptsize \cref{lem:fiber}}};
    \node[new, text width=.30\textwidth] (dim) at (0,-2.2) {
      dimension calculus:\\
      \(\Nat\) and \(\DS\) are additive\\[1pt]
      {\scriptsize \cref{lem:dim-calculus}}};
    \node[new, text width=.24\textwidth] (loo) at (5.7,-2.2) {
      padded leave-one-out\\[1pt]
      {\scriptsize \cref{lem:ds-lower}}};
    \node[new, text width=.24\textwidth] (ass) at (11.4,-2.2) {
      noisy pair-Assouad\\[1pt]
      {\scriptsize \cref{lem:assouad-lower}}};
    \node[new, text width=.62\textwidth] (thm) at (5.7,-4.6) {
      \cref{thm:lower}:\quad
      \(\E\bigl[\Risk_{\cD}(A(S))\bigr]-\Lstar
        =\Omega\bigl(\sqrt{\Lstar\dN/n}+\dDS/n\bigr)\)
      \ at exact \(\Lstar\)};
    \draw[arr] (sep) -- node[lbl, right=2pt]
      {class \(\cH_{\mathrm{cube}}\oplus\cW\)} (dim);
    \draw[arr] (fib) -- (loo);
    \draw[arr] (dim) -- node[lbl, below=2pt, sloped=false]
      {dimensions} (thm);
    \draw[arr] (loo) -- node[lbl, left=3pt]
      {coverage term \eqref{eq:lb-side2}} (thm);
    \draw[arr] (ass) -- node[lbl, above right=1pt and -14pt, pos=0.42]
      {localized term \eqref{eq:lb-side1}} (thm);
  \end{tikzpicture}
  \caption{Dependency structure of the lower bound, in the format of
  \cref{fig:dependencies}.  The single imported module is the separation
  family of \cite{brukhim2022characterization}, restated with its
  derivation in \cref{lem:separation}; every red box is proved in this
  paper.  The testing inputs to \cref{lem:assouad-lower} are Pinsker's
  inequality and the Neyman--Pearson two-point bound
  \cite[Chapter~2]{tsybakov2009introduction}.}
  \label{fig:lower-dependencies}
\end{figure}
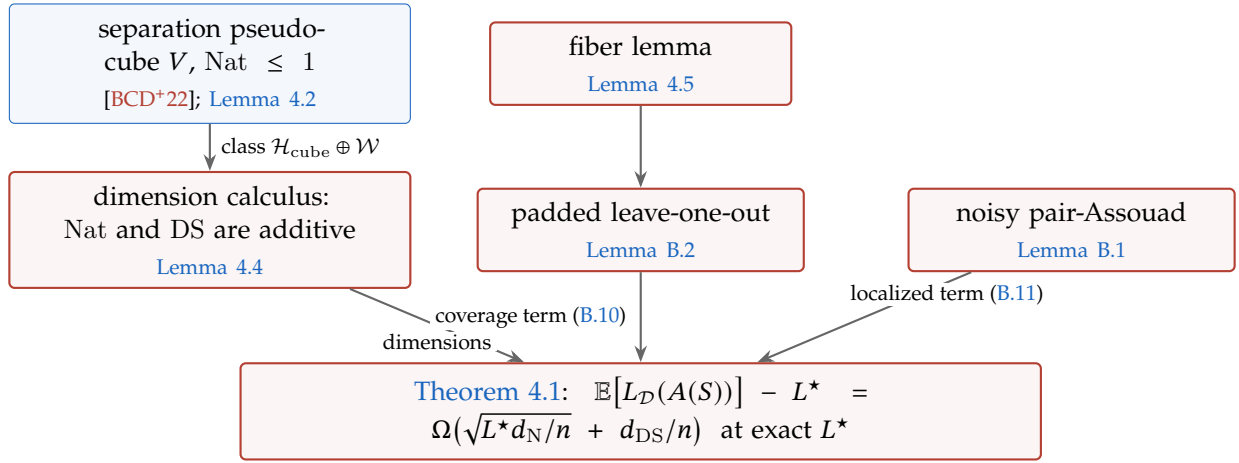

Throughout, a \emph{learner} is a map
\(A:(\cX\times[K])^n\times\cR\to[K]^{\cX}\) together with an arbitrary
distribution over the seed set \(\cR\); all bounds are on expectations
over the sample and the seed jointly, so it suffices to prove them for
each fixed seed, that is, for deterministic \(A\).

\subsection{The noisy pair-Assouad lemma}
\label{app:lower-assouad}

\begin{lemma}[Localized Natarajan lower bound]
\label{lem:assouad-lower}
Let \(d\ge1\) and let the points \(x_0,x_1,\ldots,x_d\in\cX\), the labels
\(a_i\ne b_i\), and the anchor label \(y_0\) be as in \cref{sec:lower}, with
a class \(\cH\) whose restrictions to \(\{x_0,\ldots,x_d\}\) are exactly
the cube family \(\{g_\sigma:\sigma\in\{\pm1\}^d\}\).  Fix \(n\ge1\) and
\(\bar L\in[d/n,\,1/6]\), and set
\[
 \gamma:=\frac18\sqrt{\frac{d}{n\bar L}}\ \in\Bigl(0,\frac18\Bigr],
 \qquad
 \tau:=\frac{\bar L}{d(\tfrac12-\gamma)},
 \qquad
 d\tau=\frac{\bar L}{\tfrac12-\gamma}\le\frac49 .
\]
For \(\sigma\in\{\pm1\}^d\) let \(\widetilde\cD_\sigma\) place mass
\(1-d\tau\) on \((x_0,y_0)\) and mass \(\tau\) on each \(x_i\), with label
\(a_i\) with probability \(\tfrac12+\gamma\sigma_i\) and \(b_i\) with
probability \(\tfrac12-\gamma\sigma_i\).  Then:
\begin{enumerate}
\item[(i)] \(\inf_{h\in\cH}\Risk_{\widetilde\cD_\sigma}(h)
=d\tau\bigl(\tfrac12-\gamma\bigr)=\bar L\) for every \(\sigma\), attained by
any hypothesis restricting to \(g_\sigma\).
\item[(ii)] For every learner \(A\), with \(\sigma\) uniform on
\(\{\pm1\}^d\),
\[
 \E_{\sigma}\,\E_{S\sim\widetilde\cD_\sigma^{n}}
 \bigl[\Risk_{\widetilde\cD_\sigma}(A(S))\bigr]-\bar L
 \ \ge\ \frac1{10}\sqrt{\frac{\bar Ld}{n}} .
\]
\end{enumerate}
\end{lemma}

\begin{proof}
Write \(y_i^{(+1)}:=a_i\) and \(y_i^{(-1)}:=b_i\).

\emph{Risk identity and (i).}  For any classifier \(g\),
\[
 \Risk_{\widetilde\cD_\sigma}(g)
 =(1-d\tau)\,\ind\{g(x_0)\ne y_0\}
 +\sum_{i=1}^d \tau\,r_i(g),
 \qquad
 r_i(g)=
 \begin{cases}
 \tfrac12-\gamma,& g(x_i)=y_i^{(\sigma_i)},\\[1pt]
 \tfrac12+\gamma,& g(x_i)=y_i^{(-\sigma_i)},\\[1pt]
 1,&\text{otherwise}.
 \end{cases}
\tag{B.1}\label{eq:lb-risk-identity}
\]
Since \(r_i\ge\tfrac12-\gamma\) always, every \(g\) has risk at least
\(d\tau(\tfrac12-\gamma)\), with equality iff \(g(x_0)=y_0\) and
\(g(x_i)=y_i^{(\sigma_i)}\) for all \(i\); the hypothesis restricting to
\(g_\sigma\) achieves this, and by the restriction hypothesis no member
of \(\cH\) does better.  This proves (i).  Moreover, comparing
\eqref{eq:lb-risk-identity} against the minimum and using
\(\bigl(\tfrac12+\gamma\bigr)-\bigl(\tfrac12-\gamma\bigr)=2\gamma\) and
\(1-\bigl(\tfrac12-\gamma\bigr)\ge2\gamma\),
\[
 \Risk_{\widetilde\cD_\sigma}(g)-\bar L
 \ \ge\ 2\gamma\tau\sum_{i=1}^d
 \ind\bigl\{g(x_i)\ne y_i^{(\sigma_i)}\bigr\}.
\tag{B.2}\label{eq:lb-excess-decomp}
\]

\emph{Per-coordinate testing.}  Fix a coordinate \(i\) and a value of
\(\sigma_{-i}\), and compare the two sample laws
\(\widetilde\cD_{\sigma}^{n}\) with \(\sigma_i=\pm1\).  They
differ only in the label law at \(x_i\), so, writing
\(\kl(\alpha\,\|\,\beta)\) for the divergence between Bernoulli laws, the
per-observation divergence is
\[
 \tau\cdot\kl\Bigl(\tfrac12+\gamma\,\Big\|\,\tfrac12-\gamma\Bigr)
 =\tau\cdot2\gamma\log\frac{\tfrac12+\gamma}{\tfrac12-\gamma}
 \le \tau\cdot\frac{4\gamma^2}{\tfrac12-\gamma}
 \le16\tau\gamma^2,
\tag{B.3}\label{eq:lb-kl}
\]
using \(\log(1+t)\le t\) and \(\gamma\le\tfrac14\).  By the chain rule
and Pinsker's inequality (see, e.g., \cite{boucheron2013concentration}),
\[
 \TV\Bigl(\widetilde\cD_{\sigma^{i,+}}^{n},
 \widetilde\cD_{\sigma^{i,-}}^{n}\Bigr)
 \le\sqrt{8n\tau\gamma^2}
 =\sqrt{\frac{1}{8(\tfrac12-\gamma)}}
 \le\frac{1}{\sqrt3},
\tag{B.4}\label{eq:lb-tv}
\]
where we substituted
\(n\tau\gamma^2=\frac{1}{64(1/2-\gamma)}\) from the definitions of
\(\gamma\) and \(\tau\), and used \(\gamma\le\tfrac18\).  Define the test
\(\Psi:=+1\) if \(A(S)(x_i)=a_i\) and \(\Psi:=-1\) otherwise; then
\(\ind\{A(S)(x_i)\ne y_i^{(\sigma_i)}\}\ge\ind\{\Psi\ne\sigma_i\}\) in both
cases of \(\sigma_i\).  By the Neyman--Pearson two-point bound
(see, e.g., \cite[Chapter~2]{tsybakov2009introduction}), for
\(\sigma_i\) uniform and conditionally on \(\sigma_{-i}\),
\[
 \E_{\sigma_i}\,\E\,\ind\{\Psi\ne\sigma_i\}
 \ \ge\ \frac{1-\TV}{2}
 \ \ge\ \frac{1-1/\sqrt3}{2}
 \ \ge\ \frac15 .
\tag{B.5}\label{eq:lb-test}
\]

\emph{Combining.}  Averaging \eqref{eq:lb-test} over \(\sigma_{-i}\),
summing over \(i\), and inserting into \eqref{eq:lb-excess-decomp}:
\[
 \E_\sigma\,\E_{S\sim\widetilde\cD_\sigma^{n}}
 \bigl[\Risk_{\widetilde\cD_\sigma}(A(S))\bigr]-\bar L
 \ \ge\ 2\gamma\tau\cdot\frac d5
 =\frac{2\gamma}{5}\cdot\frac{\bar L}{\tfrac12-\gamma}
 \ \ge\ \frac{4\gamma\bar L}{5}
 =\frac{4\bar L}{5}\cdot\frac18\sqrt{\frac{d}{n\bar L}}
 =\frac1{10}\sqrt{\frac{\bar Ld}{n}} . \qedhere
\]
\end{proof}

\subsection{The realizable pseudo-cube bound}
\label{app:lower-ds}

\begin{lemma}[Padded leave-one-out bound]
\label{lem:ds-lower}
Let \(V\subseteq[K]^D\) be a pseudo-cube realized by hypotheses of
\(\cH\) on distinct points \(z_1,\ldots,z_D\).  Fix
\(\nu\in(0,\tfrac12]\) and let \(\{\cD_v\}_{v\in V}\) be distributions of
the form
\[
 \cD_v=(1-\nu)\,\cQ+\nu\,\cU_v,
\]
where the outside component \(\cQ\) does not depend on \(v\) and is
supported off \(\{z_1,\ldots,z_D\}\), and \(\cU_v\) is the uniform
distribution on the \(D\) labeled pairs \(\{(z_i,v_i):i\in[D]\}\).  Then for every learner \(A\), with
\(v\) uniform on \(V\),
\[
 \E_v\,\E_{S\sim\cD_v^{n}}
 \Biggl[\frac{\nu}{D}\sum_{i=1}^{D}
 \ind\{A(S)(z_i)\ne v_i\}\Biggr]
 \ \ge\ \frac{\nu}{2}\Bigl(1-\frac{\nu}{D}\Bigr)^{n}
 \ \ge\ \frac{\nu}{2}\,e^{-2\nu n/D}.
\tag{B.6}\label{eq:lb-ds}
\]
\end{lemma}

\begin{proof}
Let \(J\subseteq[D]\) be the set of indices \(i\) such that \(z_i\) does
not occur in the sample.  The \(x\)-part of the sample has the same law
for every \(v\), and the labels at the observed \(z_j\) equal \(v_j\)
deterministically, while the labels of the outside component carry no
information about \(v\).  Fix the entire sample except for the identity
of \(v\): under the uniform prior, the posterior of \(v\) is uniform on
the cloud \(C=\{u\in V:u_j=v_j\ \forall j\notin J\}\), because every
\(u\in C\) generates the observed sample with the same probability and
every \(u\notin C\) with probability zero.  For \(i\in J\), the
prediction \(A(S)(z_i)\) is a fixed label given the sample, so by
\cref{lem:fiber} it disagrees with \(v_i\) with conditional probability
at least \(\tfrac12\).  Hence
\[
 \E_v\,\E_{S\sim\cD_v^{n}}\Biggl[\frac{\nu}{D}\sum_{i=1}^{D}
 \ind\{A(S)(z_i)\ne v_i\}\Biggr]
 \ \ge\ \frac{\nu}{D}\cdot\frac12\,\E\,|J|
 =\frac{\nu}{2}\Bigl(1-\frac{\nu}{D}\Bigr)^{n},
\]
since each index is unseen with probability \((1-\nu/D)^n\).  The final
bound in \eqref{eq:lb-ds} is the inequality \(1-t\ge e^{-2t}\) for
\(t\in[0,\tfrac12]\), applied with \(t=\nu/D\le\tfrac12\).
\end{proof}

\subsection{Dimension calculus}
\label{app:lower-dims}

\LemDimCalc*

\begin{proof}
\emph{Projection.}  Let \(V\subseteq[K]^D\) be a pseudo-cube and
\(S\subseteq[D]\) nonempty.  For \(u\in V|_S\) pick a preimage \(v\in V\)
and a coordinate \(i\in S\); the neighbor \(v^{(i)}\), which differs from
\(v\) exactly at \(i\), restricts to an element of \(V|_S\) differing
from \(u\) exactly at \(i\).  Hence \(V|_S\) is a pseudo-cube.

\emph{Sums, lower direction.}  Concatenating a Natarajan-shattered
(resp.\ DS-shattered) sequence of \(\cH_1\) with one of \(\cH_2\)
shatters the concatenation, because a product of sign cubes is a sign
cube and a product of pseudo-cubes is a pseudo-cube: a single-coordinate
neighbor changes one factor only.

\emph{Sums, upper direction.}  The trace of \(\cH_1\oplus\cH_2\) on a
mixed sequence is the product of the factor traces.  If it contains a
pseudo-cube \(W\), the projection of \(W\) to the side-1 coordinates is a
pseudo-cube contained in the trace of \(\cH_1\), so the number of side-1
coordinates is at most \(\DS(\cH_1)\); likewise for side 2, giving
\(\DS(\cH_1\oplus\cH_2)\le\DS(\cH_1)+\DS(\cH_2)\).  The Natarajan case is
identical, since restricting the \(2^d\) sign patterns to a coordinate
subset yields all sign patterns on that subset.

\emph{Anchors.}  In a shattered sequence every coordinate takes at least
two values in the trace: a witness pair for Natarajan shattering, a
differing neighbor for DS shattering.  Both fail at a point on which all
hypotheses agree.
\end{proof}

\subsection{Proof of the lower bound theorem}
\label{app:lower-main}

\ThmLower*

\begin{proof}
\emph{The class.}  Let \(\cW\) be the class of \cref{lem:separation}
with \(D:=\dDS\), realized by a pseudo-cube \(V\subseteq[K_D]^{D}\) on
points \(z_1,\ldots,z_D\), with \(\Nat(\cW)\le1\) and \(\DS(\cW)=D\).
Let \(\cH_{\mathrm{cube}}\) be the cube family of \cref{sec:lower} on
\(x_0,\ldots,x_{\dN}\).  Since binary pseudo-cubes are connected under
single-coordinate flips, the only pseudo-cube in a sign cube is the full
cube, so \(\Nat(\cH_{\mathrm{cube}})=\DS(\cH_{\mathrm{cube}})=\dN\), the
anchor contributing nothing by \cref{lem:dim-calculus}.  Set
\(\cH:=\cH_{\mathrm{cube}}\oplus\cW\) on the disjoint union of the two
domains, with the label alphabet the union of the two label sets.  By
\cref{lem:dim-calculus},
\(\Nat(\cH)=\dN+\Nat(\cW)\in\{\dN,\dN+1\}\) and
\(\DS(\cH)=\dN+\dDS\), as stated.

\emph{The distribution family.}  Fix \(n\ge1\) and
\(\Lstar\in[0,1/8]\), and set
\[
 \nu:=\min\Bigl\{\frac Dn,\frac14\Bigr\},
 \qquad
 \bar L:=\frac{\Lstar}{1-\nu}\ \in\ \Bigl[\Lstar,\tfrac43\Lstar\Bigr]
 \ \subseteq\ \Bigl[0,\tfrac16\Bigr].
\]
The inner side-1 distribution \(\widetilde\cD_\sigma\) is chosen by
cases.  If \(\Lstar\ge\dN/n\), take \(\widetilde\cD_\sigma\) as in
\cref{lem:assouad-lower} with parameters \((d,\bar L)=(\dN,\bar L)\);
this is admissible because \(\bar L\ge\Lstar\ge\dN/n\) and
\(\bar L\le1/6\).  If \(0<\Lstar<\dN/n\), let \(\widetilde\cD_\sigma\)
(independent of \(\sigma\)) place mass \(2\bar L\) on \(x_1\) with the
label uniform on \(\{a_1,b_1\}\) and mass \(1-2\bar L\) on \((x_0,y_0)\);
every hypothesis then has side-1 risk exactly \(\bar L\), which
calibrates the oracle risk without any claim of hardness.  If
\(\Lstar=0\), place all side-1 mass on \((x_0,y_0)\).  In all cases define,
for \((\sigma,v)\in\{\pm1\}^{\dN}\times V\),
\[
 \cD_{\sigma,v}
 :=(1-\nu)\,\widetilde\cD_\sigma+\nu\,\cU_v .
\tag{B.7}\label{eq:lb-family}
\]
Because \(\cH\) is a product class and the two supports are disjoint,
\[
 \inf_{h\in\cH}\Risk_{\cD_{\sigma,v}}(h)
 =(1-\nu)\inf_{h_1}\Risk_{\widetilde\cD_\sigma}(h_1)
 +\nu\cdot0
 =(1-\nu)\bar L=\Lstar
\tag{B.8}\label{eq:lb-exact}
\]
for every \((\sigma,v)\): the side-1 infimum is \(\bar L\) by
\cref{lem:assouad-lower}(i) in the first case and by direct computation
in the other two, and the side-2 component is realizable (the hypothesis
realizing \(v\) makes no error there).  Excess risk therefore decomposes
as
\[
 \Risk_{\cD_{\sigma,v}}(g)-\Lstar
 =(1-\nu)\bigl(\Risk_{\widetilde\cD_\sigma}(g)-\bar L\bigr)
 +\frac{\nu}{D}\sum_{i=1}^{D}\ind\{g(z_i)\ne v_i\},
\tag{B.9}\label{eq:lb-split}
\]
where the first term is nonnegative for every \(g\) by the side-1
infimum property.

\emph{Side-2 term.}  \Cref{lem:ds-lower} applies with
\(\cQ=\widetilde\cD_\sigma\) for each fixed \(\sigma\): the outside
component does not depend on \(v\) and avoids \(\{z_i\}\).  Averaging
over \(\sigma\) as well,
\[
 \E_{\sigma,v}\,\E_{S\sim\cD_{\sigma,v}^{n}}
 \Biggl[\frac{\nu}{D}\sum_{i=1}^{D}
 \ind\{A(S)(z_i)\ne v_i\}\Biggr]
 \ \ge\ \frac{\nu}{2}e^{-2\nu n/D}
 \ \ge\ \frac1{15}\min\Bigl\{\frac Dn,\frac14\Bigr\},
\tag{B.10}\label{eq:lb-side2}
\]
by the two regimes of \(\nu\): if \(n\ge4D\) then \(\nu=D/n\le\tfrac14\)
and \(\tfrac\nu2e^{-2\nu n/D}=\tfrac{D}{2e^2n}\ge\tfrac1{15}\cdot\tfrac
Dn\); if \(n<4D\) then \(\nu=\tfrac14\) and \(2\nu n/D<2\), so
\(\tfrac\nu2e^{-2\nu n/D}\ge\tfrac18e^{-2}\ge\tfrac1{15}\cdot\tfrac14\).

\emph{Side-1 term.}  Suppose \(\Lstar\ge\dN/n\).  The argument of
\cref{lem:assouad-lower}(ii) applies to the padded family
\(\{\cD_{\sigma,v}\}\) with \(v\) fixed: for each coordinate \(i\), the
two laws \(\cD_{\sigma^{i,\pm},v}\) differ only in the label law at
\(x_i\), whose mass is now \((1-\nu)\tau\le\tau\), so the per-observation
divergence is at most the bound \eqref{eq:lb-kl} and
\eqref{eq:lb-tv}--\eqref{eq:lb-test} are unchanged.  The excess
decomposition \eqref{eq:lb-split} carries the factor \((1-\nu)\), and
with \((1-\nu)\sqrt{\bar L}=\sqrt{1-\nu}\sqrt{\Lstar}
\ge\sqrt{3/4}\,\sqrt{\Lstar}\),
\[
 \E_{\sigma,v}\,\E_{S\sim\cD_{\sigma,v}^{n}}
 \Bigl[(1-\nu)
 \bigl(\Risk_{\widetilde\cD_\sigma}(A(S))-\bar L\bigr)\Bigr]
 \ \ge\ (1-\nu)\cdot\frac1{10}\sqrt{\frac{\bar L\dN}{n}}
 \ \ge\ \frac{\sqrt3}{20}\sqrt{\frac{\Lstar\dN}{n}}
 \ \ge\ \frac1{12}\sqrt{\frac{\Lstar\dN}{n}} .
\tag{B.11}\label{eq:lb-side1}
\]
(The side-1 risk of \(A(S)\) under \(\widetilde\cD_\sigma\) is read off
the same decomposition \eqref{eq:lb-risk-identity}, which only concerns
the side-1 points; the side-2 part of the sample is shared by the two
tested laws and enters the test as ancillary information.)  For
\(\Lstar<\dN/n\) no side-1 claim is made and the first term of
\eqref{eq:lower-bound} is zero.

\emph{Conclusion.}  Summing \eqref{eq:lb-side1} and \eqref{eq:lb-side2}
inside \eqref{eq:lb-split},
\[
 \E_{\sigma,v}\,\E_{S\sim\cD_{\sigma,v}^{n}}
 \bigl[\Risk_{\cD_{\sigma,v}}(A(S))\bigr]-\Lstar
 \ \ge\
 \frac1{12}\,\ind\{\Lstar\ge\dN/n\}\sqrt{\frac{\Lstar\dN}{n}}
 +\frac1{15}\min\Bigl\{\frac{\dDS}{n},\frac14\Bigr\},
\]
and some \((\sigma,v)\) attains at least the average, giving a single
distribution \(\cD:=\cD_{\sigma,v}\) with the stated properties.  The
family \eqref{eq:lb-family} depends only on \((\cH,\Lstar,n)\).
\end{proof}

\subsection{Which engine for which structure}
\label{app:lower-engines-scope}

The two engines of this appendix, and the classical devices they
replace, are separated by the combinatorial structure they consume;
\cref{tab:engines} records the comparison, since the choice is reusable
beyond this paper.

\begin{table}[t]
  \centering
  \caption{Lower-bound devices and the structure they require.  The
  fiber engine is the one that survives when the shattered structure is
  a pseudo-cube but not a product; at width \(w\) its per-coordinate
  yield becomes \((w-r)/w\) against \(r\)-lists (\cref{app:list}),
  which is where \cref{cor:list-realizable} gains its factor \(r\).}
  \label{tab:engines}
  \footnotesize
  \begin{tabular}{@{}>{\raggedright\arraybackslash}p{0.30\textwidth}>{\raggedright\arraybackslash}p{0.33\textwidth}>{\raggedright\arraybackslash}p{0.27\textwidth}@{}}
    \toprule
    Device & Structure required & Per-coordinate yield \\
    \midrule
    Pair-Assouad
    (\cref{lem:assouad-lower}, \cref{lem:list-assouad})
      & full product cube of witness labels in the trace; commuting
        coordinate flips index the family
      & testing: excess \(\gamma\tau\) per tested coordinate,
        fluctuation type \\
    \addlinespace
    Leave-one-out \cite{ehrenfeucht1989lower}
      & realizable completions at unseen points, classically again a
        product trace
      & realizable type, \(1/n\) per coordinate \\
    \addlinespace
    Fiber engine
    (\cref{lem:fiber}, \cref{lem:ds-lower}, \cref{lem:list-ds-lower})
      & pseudo-cube only: a neighbor in each direction; no product, no
        involution structure
      & conditional error \(\ge\tfrac12\) per unseen coordinate;
        \(\ge(w-r)/w\) at width \(w\) \\
    \bottomrule
  \end{tabular}
\end{table}

The six-cycle of \cref{ex:sixcycle} shows that the extra generality of
the third row is not hypothetical, and it makes the phrase ``does not
lift'' from the abstract concrete.  The six-cycle is a pseudo-cube on
two coordinates with six vertices, and \(6\ne|S_1|\cdot|S_2|\) for any
witness sets realizing its trace as a product, so no Assouad indexation
by sign patterns exists.  Its per-coordinate flips are well defined,
since every fiber has exactly two elements, but they do not commute:
writing \(\phi_1,\phi_2\) for the two flips, on the vertex \((0,0)\)
one has \(\phi_1\phi_2(0,0)=(1,1)\) while \(\phi_2\phi_1(0,0)=(2,2)\);
the two flips generate a rotation of the hexagon rather than a
\(\{\pm1\}^2\)-action.  Every binary-flavored argument that pairs
distributions along commuting coordinate involutions therefore breaks
on this class, while the fiber engine only reads the sizes of fibers
and proceeds.  This is why side~2 of the construction in
\cref{app:lower-main} runs on the leave-one-out engine and not on an
Assouad engine.

A second reusable pattern is the calibration recipe behind
\eqref{eq:lb-exact}.  To pin the oracle risk of a family exactly at a
prescribed \(\Lstar\), pad with weight \(\nu\) of realizable mass, set
\(\bar L=\Lstar/(1-\nu)\), and spread \(\bar L\) over the noisy
component so that every classifier pays at least \(\bar L\) there
pointwise; the infimum is then \((1-\nu)\bar L=\Lstar\) exactly, it is
attained inside the class, and, since the pointwise bound holds for
every classifier, the calibration is Bayes-exact
(\cref{app:lean}, finding~(1)).  The family depends only on
\((\cH,\Lstar,n)\), never on the learner.  Any fixed-oracle-level lower
bound that claims equality rather than an inequality for its oracle
risk can reuse this padding unchanged.

\section{Proofs for the list extension}
\label{app:list}

Throughout this appendix \(r\ge1\) is the list size.  We identify a
classifier \(f\) with the singleton list \(x\mapsto\{f(x)\}\); under this
identification every definition below restricts to its counterpart in
the main text at \(r=1\).  For a menu \(\mu\) and a list-valued
\(\lambda\), the inside-menu list loss is
\[
 \ell_\mu\bigl(\lambda,(x,y)\bigr)
 =\ind\{y\notin\lambda(x),\ y\in\mu(x)\},
 \qquad
 \MenuLoss(\lambda)=\E\,\ell_\mu(\lambda,(X,Y)),
\tag{C.1}\label{eq:list-menu-loss}
\]
which extends \eqref{eq:menu-loss} verbatim.  Empirical versions
\(\Rhat^{\ell_\mu}_s\) are defined as in \cref{sec:architecture}.

\subsection{The list menu calculus}

\begin{lemma}[List selective-menu decomposition]
\label{lem:list-menu-calculus}
Let \(\mu\) be a menu.  For all list-valued \(\lambda,\lambda'\),
\[
 \ListRisk_{\cD}(\lambda)-\ListRisk_{\cD}(\lambda')
 \le
 \MenuLoss(\lambda)-\MenuLoss(\lambda')
 +\P\bigl(Y\in\lambda'(X),\ Y\notin\mu(X)\bigr),
\tag{C.2}\label{eq:list-risk-algebra}
\]
and \(\MenuLoss(\lambda)\le\ListRisk_{\cD}(\lambda)\).  Moreover, for
every tuple \(\vec h=(h^1,\ldots,h^r)\in\cH^r\) and every classifier
tuple \((f_1,\ldots,f_r)\),
\[
 \P\bigl(Y\in\Lambda_{\vec h}(X),\ Y\notin\mu(X)\bigr)
 \le
 \sum_{j=1}^{r}
 \Bigl[\P\bigl(h^j(X)=Y\ne f_j(X)\bigr)
 +\P\bigl(f_j(X)=Y,\ Y\notin\mu(X)\bigr)\Bigr].
\tag{C.3}\label{eq:list-miss-transfer}
\]
\end{lemma}

\begin{proof}
For \eqref{eq:list-risk-algebra}, fix \((x,y)\).  If \(y\in\mu(x)\), the
difference of the two list losses equals the difference of the two
inside-menu list losses.  If \(y\notin\mu(x)\), both inside-menu losses
vanish, and the list-loss difference can be positive only when
\(y\in\lambda'(x)\), in which case it is at most one.  Taking
expectations proves \eqref{eq:list-risk-algebra}; the comparison
\(\ell_\mu(\lambda,z)\le\ind\{y\notin\lambda(x)\}\) is pointwise.  For
\eqref{eq:list-miss-transfer}, on the event
\(\{Y\in\Lambda_{\vec h}(X),\,Y\notin\mu(X)\}\) there is some \(j\) with
\(h^j(X)=Y\notin\mu(X)\); bound the probability of each such event by
\eqref{eq:menu-miss-transfer} with \(h=h^j\) and \(f=f_j\), and sum.
\end{proof}

\subsection{Relative compression with product messages}

The proof of \cref{thm:relative-compression} in \cref{app:relative}
interacts with the selection scheme only through the count of possible
descriptions: a message of at most \(k\) ordered sample elements has at
most \((N+1)^{k+1}\) realizations, and the argument is a union bound
over them.  The following variant records the count for messages that
are \(r\)-tuples of ordinary messages.

\begin{lemma}[Relative compression, product messages]
\label{lem:relative-product}
Let \(\ell:\cF\times\cZ\to\{0,1\}\) and \(\cH\subseteq\cF\) be as in
\cref{thm:relative-compression}, and let \(A\) be a measurable
deterministic scheme whose message on a sample \(s\in\cZ^N\) is an
\(r\)-tuple \((t^1,\ldots,t^r)\) of ordered sequences of at most \(L\)
elements of \(s\) each, with \(A(s)=\rho(t^1,\ldots,t^r)\), and which
satisfies the domination premise \eqref{eq:relative-premise}.  Then the
conclusion \eqref{eq:relative-bound} of
\cref{thm:relative-compression} holds with
\(\Gamma_N(k,\delta)\) taken at \(k+1=r(L+1)\), that is, with
\[
 \Gamma
 =\frac{r(L+1)\log(N+1)+\log(4/\delta)}{N}.
\]
\end{lemma}

\begin{proof}
The proof of \cref{thm:relative-compression} fixes a description,
conditions on the named coordinates, applies
\cref{lem:bernoulli-bernstein} on the untouched coordinates, and takes a
union bound over all descriptions; the number of descriptions enters
only through the bound \(M\le(N+1)^{k+1}\) on their count, which sets
the \((k+1)\log(N+1)\) term of \(\Gamma\).  An \(r\)-tuple of ordered
sequences of at most \(L\) sample elements has at most
\(\bigl((N+1)^{L+1}\bigr)^{r}=(N+1)^{r(L+1)}\) realizations, and each
tuple names at most \(rL\) sample coordinates, so every step of the
argument is unchanged after replacing \(M\) by this count.
\end{proof}

\subsection{Fixed-menu compression for tuple comparators}

\begin{lemma}[Fixed-menu list compression]
\label{lem:list-compression}
Let \(\mu\) be a fixed measurable menu with
\(\sup_x|\mu(x)|\le p\) for a deterministic \(p\ge2\), and let \(\cH\)
have Natarajan dimension \(\dN\ge1\).  For every \(N\ge1\) there is a
measurable deterministic scheme whose message is an \(r\)-tuple of
ordered sequences of at most \(L=k_\mu(N)\le C\dN\log(p)\log(eN)\)
sample elements, whose reconstruction is a list rule of size at most
\(r\), and which is empirically dominating against tuples: for every
sample \(s\in\cZ^N\), its output \(\widehat\lambda\) satisfies
\[
 \Rhat^{\ell_\mu}_s(\widehat\lambda)
 \ \le\
 \min_{\vec h\in\cH^{r}}\Rhat^{\ell_\mu}_s(\Lambda_{\vec h}).
\tag{C.4}\label{eq:list-domination}
\]
\end{lemma}

\begin{proof}
Write \(B=\LSCS(\mu,\cH,p)\) for the single-hypothesis scheme of
\cref{lem:fixed-menu-compression}, with size function
\(k_\mu(\cdot)\), and let \(\rho_B\) be its reconstruction.  Define the
decoder first: on an \(r\)-tuple \((t^1,\ldots,t^r)\) of messages, set
\[
 \rho(t^1,\ldots,t^r)(x)
 =\{\rho_B(t^1)(x),\ldots,\rho_B(t^r)(x)\},
\]
a list of size at most \(r\).  Define the selector \(\kappa(s)\) to be
the first, in a fixed enumeration of the finitely many \(r\)-tuples of
ordered sequences of at most \(L\) elements of \(s\), tuple \(t\) whose
reconstruction satisfies \eqref{eq:list-domination}; measurability is
\cref{asm:measurable} together with finiteness of the search space.  It
remains to prove that a valid tuple exists.

The empirical loss \(\Rhat^{\ell_\mu}_s(\Lambda_{\vec h})\) takes values
in \(\{0,\tfrac1N,\ldots,1\}\), so its infimum over \(\vec h\in\cH^r\)
is attained; fix a minimizer \(\vec h^\ast=(h^{\ast1},\ldots,h^{\ast
r})\).  Let
\[
 s^\ast
 =\bigl\{i\le N:\ y_i\in\mu(x_i),\ y_i\in\Lambda_{\vec h^\ast}(x_i)\bigr\},
\]
the observations on which the minimizing tuple is inside-menu correct,
and partition \(s^\ast\) into \(s^\ast_1,\ldots,s^\ast_r\) by assigning
\(i\) to the smallest \(j\) with \(h^{\ast j}(x_i)=y_i\).  Each
\(s^\ast_j\), viewed as a sample in the order inherited from \(s\), is
handed to \(B\): let \(t^j=\kappa_B(s^\ast_j)\) and
\(g_j=\rho_B(t^j)\), with the empty-part convention of
\cref{lem:ordinary-compression} when \(s^\ast_j=\varnothing\).  Since
\(h^{\ast j}\) is inside-menu correct on every observation of
\(s^\ast_j\), empirical domination of \(B\) on the sample
\(s^\ast_j\) gives
\[
 \Rhat^{\ell_\mu}_{s^\ast_j}(g_j)
 \le
 \Rhat^{\ell_\mu}_{s^\ast_j}(h^{\ast j})
 =0,
\]
so \(g_j\) is inside-menu correct on all of \(s^\ast_j\).
Consequently the list \(\lambda_0(x)=\{g_1(x),\ldots,g_r(x)\}\) has
\(\ell_\mu(\lambda_0,(x_i,y_i))=0\) for every \(i\in s^\ast\), and for
\(i\notin s^\ast\) either \(y_i\notin\mu(x_i)\), where every list has
zero loss, or \(y_i\in\mu(x_i)\) and \(\vec h^\ast\) also errs.  Hence
\[
 \Rhat^{\ell_\mu}_s(\lambda_0)
 \le
 \frac{|\{i:\ y_i\in\mu(x_i)\}\setminus s^\ast|}{N}
 =
 \Rhat^{\ell_\mu}_s(\Lambda_{\vec h^\ast}),
\]
and \((t^1,\ldots,t^r)\) is a valid tuple: each \(t^j\) is an ordered
sequence of at most \(k_\mu(|s^\ast_j|)\le k_\mu(N)=L\) elements of
\(s^\ast_j\subseteq s\).  The size bound on \(L\) is
\eqref{eq:menu-compression-size}.
\end{proof}

As in \cref{lem:menu-finalizer}, menus with \(T\le2\) admit a size-zero
branch: the singleton rule \(g_\mu\) of that lemma, viewed as a
\(1\)-list, already has zero inside-menu loss pointwise, and we use it
for every \(r\).  Write \(B^{(r)}_{\mu,T}\) for the resulting scheme
(\cref{lem:list-compression} with \(p=T-1\) when \(T\ge3\)).

\subsection{Proof of \texorpdfstring{\cref{thm:list-upper}}{Theorem 5.1}}

\ThmListUpper*

\begin{proof}
Run \cref{alg:rmc} unchanged through its first two blocks, and replace
the third-block finalizer by \(B^{(r)}_{\mu,T}\); the output is
\(\widehat\lambda=B^{(r)}_{\mu,T}(S_3)\), a list of size at most \(r\).
The block sizes again satisfy \(n_i\ge n/5\).

Fix \(\eps>0\) and a tuple \(\vec h^\star\in\cH^r\) with
\(\ListRisk_{\cD}(\Lambda_{\vec h^\star})\le\Lstarr+\eps\).  For each
component define the canonical witness
\(f^\star_j=A_1(S_1[h^{\star j}])\in\widehat\cF\).  Apply
\cref{lem:compression-cover} to each component with
\(\eta=\delta/(3r)\), and \cref{lem:mw-menu} to each witness with
\(\eta=\delta/(3r)\); the \(\sigma\)-field bookkeeping is verbatim that
of \cref{app:main}, applied \(r\) times, and the tower property is used
component by component.  On the intersection
\(\mathcal E_1\cap\mathcal E_2\) of the resulting events, which fails
with probability at most \(2\delta/3\),
\eqref{eq:list-miss-transfer} gives
\[
 \alpha^{(r)}_\mu
 \defeq
 \P\bigl(Y\in\Lambda_{\vec h^\star}(X),\ Y\notin\mu(X)\bigr)
 \le
 C\,\frac{r\bigl(\dDS\log^2(en)+\log(er/\delta)\bigr)}{n}.
\tag{C.5}\label{eq:list-alpha}
\]

Conditional on \(\mathcal G_2=\sigma(S_1,S_2,R_\mu)\), the menu is
fixed, \(S_3\) is iid and independent, and
\(B^{(r)}_{\mu,T}\) is an empirically dominating scheme against tuples
by \cref{lem:list-compression}, with product messages of block length
\(L=k_\mu(n_3)\le C\dN\log(T-1)\log(en_3)\) (zero if \(T\le2\)).  Apply
\cref{lem:relative-product} conditionally, with comparator class
\(\{\Lambda_{\vec h}:\vec h\in\cH^r\}\), the fixed comparator
\(\Lambda_{\vec h^\star}\), and failure probability \(\delta/3\); the
premise \eqref{eq:relative-premise} is \eqref{eq:list-domination}.  Its
\(\Gamma\) satisfies, in either branch,
\[
 \Gamma
 =\frac{r(L+1)\log(n_3+1)+\log(12/\delta)}{n_3}
 \le
 C\,\frac{r\dN\log^3(en)+\log(1/\delta)}{n}.
\tag{C.6}\label{eq:list-gamma}
\]
On the resulting event \(\mathcal E_3\),
\[
 \MenuLoss(\widehat\lambda)-\MenuLoss(\Lambda_{\vec h^\star})
 \le
 12\sqrt{\MenuLoss(\Lambda_{\vec h^\star})\,\Gamma}+20\,\Gamma.
\tag{C.7}\label{eq:list-relative}
\]
By \cref{lem:list-menu-calculus},
\(\MenuLoss(\Lambda_{\vec h^\star})
\le\ListRisk_{\cD}(\Lambda_{\vec h^\star})\le\Lstarr+\eps\), and
\[
 \ListRisk_{\cD}(\widehat\lambda)
 -\ListRisk_{\cD}(\Lambda_{\vec h^\star})
 \le
 \MenuLoss(\widehat\lambda)-\MenuLoss(\Lambda_{\vec h^\star})
 +\alpha^{(r)}_\mu
\]
holds deterministically.  Substituting \eqref{eq:list-alpha},
\eqref{eq:list-gamma}, and \eqref{eq:list-relative}, and summing the
three failure probabilities, proves the bound with \(\Lstarr+\eps\) in
place of \(\Lstarr\) with probability at least \(1-\delta\).  The
learner does not depend on \(\eps\), so letting \(\eps\downarrow0\)
along a sequence and using continuity from above, exactly as at the end
of \cref{app:main}, completes the proof.  At \(r=1\) the construction
and the bound coincide with those of \cref{thm:main}.
\end{proof}

\subsection{The \texorpdfstring{\((r+1)\)}{(r+1)}-ary Assouad lemma}

\begin{lemma}[Localized \(r\)-Natarajan lower bound]
\label{lem:list-assouad}
Let \(d\ge1\), let \(x_0,x_1,\ldots,x_d\in\cX\) be distinct points, let
\(a_1,\ldots,a_{r+1}\) be distinct labels and \(y_0\) an anchor label,
and let \(\cH\) be a class whose restrictions to
\(\{x_0,\ldots,x_d\}\) are exactly the cube family
\(\{g_c:c\in[r+1]^d\}\), \(g_c(x_0)=y_0\), \(g_c(x_i)=a_{c_i}\).  Fix
\(n\ge1\) and \(\bar L\in[d/n,\ \tfrac1{6(r+1)}]\), and set
\[
 \gamma=\frac1{8(r+1)}\sqrt{\frac{d}{n\bar L}}\ \in\Bigl(0,\tfrac1{8(r+1)}\Bigr],
 \qquad
 \tau=\frac{\bar L}{d\bigl(\tfrac1{r+1}-\gamma\bigr)},
 \qquad
 d\tau\le\tfrac{8}{7}(r+1)\bar L\le\tfrac{4}{21}.
\]
For \(\sigma\in[r+1]^d\) let \(\widetilde\cD_\sigma\) place mass
\(1-d\tau\) on \((x_0,y_0)\) and mass \(\tau\) on each \(x_i\), with
label \(a_{\sigma_i}\) having conditional probability
\(\tfrac1{r+1}-\gamma\) and each other candidate
\(a_c\), \(c\ne\sigma_i\), having conditional probability
\(\tfrac1{r+1}+\tfrac\gamma r\).  Then:
\begin{enumerate}
\item[(i)] every list \(\lambda\) of size at most \(r\) has
\(\ListRisk_{\widetilde\cD_\sigma}(\lambda)\ge
d\tau(\tfrac1{r+1}-\gamma)=\bar L\), with equality attained by the
tuple \((g_{\sigma+1},\ldots,g_{\sigma+r})\), where
\((\sigma+k)_i\equiv\sigma_i+k\ (\mathrm{mod}\ r+1)\);
\item[(ii)] for every learner \(A\) outputting lists of size at most
\(r\), with \(\sigma\) uniform on \([r+1]^d\),
\[
 \E_{\sigma}\,\E_{S\sim\widetilde\cD_\sigma^{n}}
 \bigl[\ListRisk_{\widetilde\cD_\sigma}(A(S))\bigr]-\bar L
 \ \ge\ \frac1{40}\sqrt{\frac{\bar Ld}{n}}.
\]
\end{enumerate}
\end{lemma}

\begin{proof}
\emph{Risk identity and (i).}  For any list \(\lambda\), the risk under
\(\widetilde\cD_\sigma\) decomposes over the atoms:
\[
 \ListRisk_{\widetilde\cD_\sigma}(\lambda)
 =(1-d\tau)\,\ind\{y_0\notin\lambda(x_0)\}
 +\sum_{i=1}^d \tau
 \sum_{c:\,a_c\notin\lambda(x_i)}\pi_{\sigma_i}(c),
\tag{C.8}\label{eq:list-risk-identity}
\]
where \(\pi_{\sigma_i}(c)\) is the conditional label law at \(x_i\).
Since \(|\lambda(x_i)|\le r<r+1\), at least one candidate is excluded at
each \(x_i\), and every candidate has mass at least
\(\tfrac1{r+1}-\gamma\); this proves the infimum claim.  The displayed
tuple excludes exactly \(a_{\sigma_i}\) at \(x_i\) and contains
\(y_0\) at \(x_0\) (every \(g_c(x_0)=y_0\)), attaining equality.
Moreover, whenever the excluded set at \(x_i\) is not exactly
\(\{a_{\sigma_i}\}\), some candidate of mass
\(\tfrac1{r+1}+\tfrac\gamma r\) is excluded, so
\[
 \ListRisk_{\widetilde\cD_\sigma}(\lambda)-\bar L
 \ \ge\
 \gamma\Bigl(1+\tfrac1r\Bigr)\tau\sum_{i=1}^d
 \ind\bigl\{\Psi_i(\lambda)\ne\sigma_i\bigr\},
\tag{C.9}\label{eq:list-excess}
\]
where \(\Psi_i(\lambda)\in[r+1]\) denotes the smallest index of a
candidate excluded by \(\lambda\) at \(x_i\).

\emph{Per-coordinate testing.}  Fix \(i\), fix \(\sigma_{-i}\), and fix
two values \(c\ne c'\) of \(\sigma_i\).  The two laws
\(\widetilde\cD_{\sigma}\) differ only by exchanging the masses of the
two atoms \((x_i,a_c)\) and \((x_i,a_{c'})\), which are
\(\tau(\tfrac1{r+1}-\gamma)\) and \(\tau(\tfrac1{r+1}+\tfrac\gamma r)\).
Writing \(q=\tfrac1{r+1}-\gamma\ge\tfrac{7}{8(r+1)}\) and
\(q'=q+\gamma(1+\tfrac1r)\), the per-observation divergence is
\[
 \kl\bigl(\widetilde\cD_{\sigma^{i,c}}\,\big\|\,
 \widetilde\cD_{\sigma^{i,c'}}\bigr)
 =\tau\,(q'-q)\log\frac{q'}{q}
 \le
 \tau\,\frac{(q'-q)^2}{q}
 \le
 5(r+1)\tau\gamma^2,
\tag{C.10}\label{eq:list-kl}
\]
using \(\log(1+t)\le t\) and \((1+\tfrac1r)^2\le4\).  By the chain rule
and Pinsker's inequality, and then the definitions of \(\gamma\) and
\(\tau\),
\[
 \TV\bigl(\widetilde\cD_{\sigma^{i,c}}^{\,n},
 \widetilde\cD_{\sigma^{i,c'}}^{\,n}\bigr)
 \le\sqrt{\tfrac52\,n(r+1)\tau\gamma^2}
 \le\sqrt{\tfrac52\cdot\tfrac1{56}}
 \le\frac1{\sqrt3}.
\tag{C.11}\label{eq:list-tv}
\]
Let \(\Psi=\Psi_i(A(S))\).  For each unordered pair \(\{c,c'\}\), the
Neyman--Pearson two-point bound gives
\(\P(\Psi\ne c\mid\sigma_i=c)+\P(\Psi\ne c'\mid\sigma_i=c')\ge
1-\TV\ge1-1/\sqrt3\).  Summing over the \(\binom{r+1}2\) pairs, each
conditional appears \(r\) times, so for \(\sigma_i\) uniform,
\[
 \E_{\sigma_i}\,\P(\Psi\ne\sigma_i)
 \ \ge\ \frac{1-1/\sqrt3}{2}\ \ge\ \frac15 .
\tag{C.12}\label{eq:list-test}
\]

\emph{Combining.}  Averaging \eqref{eq:list-test} over \(\sigma_{-i}\),
summing over \(i\), and inserting into \eqref{eq:list-excess}:
\[
 \E_\sigma\,\E\bigl[\ListRisk_{\widetilde\cD_\sigma}(A(S))\bigr]-\bar L
 \ \ge\
 \gamma\Bigl(1+\tfrac1r\Bigr)\tau\cdot\frac d5
 \ \ge\ \frac{\gamma\,(r+1)\bar L}{5}
 \ =\ \frac1{40}\sqrt{\frac{\bar Ld}{n}},
\]
using \(\tau d\ge(r+1)\bar L\) in the middle step.
\end{proof}

\subsection{The list leave-one-out bound}

\begin{lemma}[Padded leave-one-out bound, list version]
\label{lem:list-ds-lower}
Let \(V\subseteq[K]^D\) be realized by hypotheses of \(\cH\) on distinct
points \(z_1,\ldots,z_D\), and suppose every \(v\in V\) has, in every
coordinate \(i\), at least \(2r-1\) neighbors in \(V\) differing from it
exactly at \(i\).  Fix \(\nu\in(0,\tfrac12]\) and let
\(\cD_v=(1-\nu)\cQ+\nu\,\cU_v\) be as in \cref{lem:ds-lower}.  Then for
every learner \(A\) outputting lists of size at most \(r\), with \(v\)
uniform on \(V\),
\[
 \E_v\,\E_{S\sim\cD_v^{n}}
 \Biggl[\frac{\nu}{D}\sum_{i=1}^{D}
 \ind\{v_i\notin A(S)(z_i)\}\Biggr]
 \ \ge\ \frac{\nu}{2}\Bigl(1-\frac{\nu}{D}\Bigr)^{n}.
\]
\end{lemma}

\begin{proof}
The proof of \cref{lem:ds-lower} applies verbatim up to the fiber step:
conditionally on the sample, the posterior of \(v\) is uniform on the
cloud \(C=\{u\in V:u_j=v_j\ \forall j\notin J\}\), and for an unseen
index \(i\in J\) the list \(A(S)(z_i)\) is fixed given the sample.
Partition \(C\) into fibers by the values off \(i\).  If \(u\in C\), its
\(2r-1\) neighbors at coordinate \(i\) agree with \(u\) off \(i\), hence
lie in the same fiber and in \(C\) (the cloud constraint involves only
coordinates outside \(J\)); distinct elements of a fiber differ exactly
at \(i\), so each fiber contains at least \(2r\) elements with pairwise
distinct \(i\)-values.  A list of size \(r\) therefore contains the
\(i\)-value of at most half of each fiber, and the conditional
probability that \(v_i\notin A(S)(z_i)\) is at least \(\tfrac12\).  The
remainder of the proof, including
\(\E|J|=D(1-\nu/D)^n\), is unchanged.
\end{proof}

\subsection{The band family and \texorpdfstring{\(r\)}{r}-ary dimension
calculus}

\begin{lemma}[Width-\(2r\) band]
\label{lem:band}
Let \(m=4r+2\), \(K_2=\Z_m\), and
\[
 V_r=\bigl\{(i,j)\in\Z_m\times\Z_m:\ j-i\in\{0,1,\ldots,2r-1\}\bigr\},
\]
viewed as functions on two points.  Then (a) every \(v\in V_r\) has
exactly \(2r-1\) neighbors in each coordinate; (b) the class \(V_r\) has
\(r\)-Natarajan dimension exactly \(1\); (c) for every \(t\ge1\) the
product \(V_r^{\otimes t}\), on \(D=2t\) points, retains property (a),
satisfies \(\dNk{r}(V_r^{\otimes t})\le t\), and is \(r\)-DS shattered
on all \(D\) coordinates.
\end{lemma}

\begin{proof}
(a) Fixing \(j\), the first coordinate ranges over the \(2r\) values
\(j-2r+1,\ldots,j\); fixing \(i\), the second ranges over
\(i,\ldots,i+2r-1\).  In each case the \(2r-1\) other admissible values
give neighbors differing in exactly that coordinate.

(b) One point is \(r\)-Natarajan shattered: each coordinate takes
\(2r\ge r+1\) values.  Suppose the pair of points were shattered with
witness sets \(S_1,S_2\subseteq\Z_m\), \(|S_1|=|S_2|=r+1\), so that
\(S_1\times S_2\subseteq V_r\).  Fix \(i_0\in S_1\); every \(j\in S_2\)
satisfies \(j-i_0\in\{0,\ldots,2r-1\}\), so \(S_2\) lies in a cyclic
window of length \(2r<m/2\), and likewise \(S_1\); lifting both windows
to integer representatives, all differences \(j-i\), \(i\in S_1\),
\(j\in S_2\), must lie in \([0,2r-1]\).  In particular
\(\max S_2-\min S_1\le 2r-1\) and \(\min S_2-\max S_1\ge0\); subtracting,
\[
 (\max S_2-\min S_2)+(\max S_1-\min S_1)\ \le\ 2r-1,
\]
while each spread is at least \(|S_i|-1=r\), a contradiction.

(c) The dimension calculus of \cref{lem:dim-calculus} applies verbatim
at arity \(r\): a product of structures with property (a) has property
(a), since a single-coordinate neighbor changes one factor; restricting
an \((r+1)\)-ary product pattern to a subset of coordinates yields the
full pattern on that subset, so the sum bound
\(\dNk{r}(\cH_1\oplus\cH_2)\le\dNk{r}(\cH_1)+\dNk{r}(\cH_2)\) and its
DS analogue hold with the same proofs; and the \(r\)-DS shattering of
all \(2t\) coordinates is witnessed by \(V_r^{\otimes t}\) itself.
Finally, the counting bound behind the caps generalizes: if
\(W\subseteq[K]^k\) gives every element at least \(r\) neighbors in
every coordinate, then projecting off the last coordinate and using the
fact that each fiber has at least \(r+1\) elements yields
\(|W|\ge(r+1)^{k}\) by induction on \(k\).
\end{proof}

\subsection{Proof of \texorpdfstring{\cref{thm:list-lower}}{Theorem 5.2}}

\ThmListLower*

\begin{proof}
\emph{The class.}  Let \(t=D/2\) and let \(\cW=V_r^{\otimes t}\) be the
band power of \cref{lem:band}, realized on points
\(z_1,\ldots,z_D\) with label set \(\Z_{4r+2}\), disjoint from the
side-1 labels.  Let \(\cH_{\mathrm{cube}}=\{g_c:c\in[r+1]^d\}\) be the
cube family of \cref{lem:list-assouad} on \(x_0,\ldots,x_d\), and let
\(\cH=\cH_{\mathrm{cube}}\oplus\cW\) on the disjoint union of the
domains.  The trace of \(\cH_{\mathrm{cube}}\) on
\(x_1,\ldots,x_d\) is the full product \(\{a_1,\ldots,a_{r+1}\}^d\), so
\(\dNk{r}(\cH)\ge d\), and it gives every member \(r\) neighbors per
coordinate, so together with \cref{lem:band}(c),
\(\dDSk{r}(\cH)\ge d+D\).  For the caps: an \(r\)-Natarajan shattering
of \(k\) cube coordinates forces \((r+1)^k\le|\cH_{\mathrm{cube}}|
=(r+1)^d\), the anchor cannot appear in any shattered sequence, the
band power contributes at most \(t\) by \cref{lem:band}(c), and the sum
bound gives \(\dNk{r}(\cH)\le d+t=d+\tfrac D2\).  Likewise an \(r\)-DS
shattering of \(k\) cube coordinates forces
\((r+1)^k\le(r+1)^d\) by the counting bound of \cref{lem:band}(c), a
shattering of side-2 coordinates is injective on \(\{z_1,\ldots,z_D\}\),
and the sum bound yields \(\dDSk{r}(\cH)\le d+D\), so
\(\dDSk{r}(\cH)=d+D\).

\emph{The family.}  Fix \(n\ge1\) and
\(\Lstarr\in[0,\tfrac1{8(r+1)}]\), and set
\[
 \nu=\min\Bigl\{\frac Dn,\frac14\Bigr\},
 \qquad
 \bar L=\frac{\Lstarr}{1-\nu}
 \ \in\ \Bigl[\Lstarr,\tfrac43\Lstarr\Bigr]
 \ \subseteq\ \Bigl[0,\tfrac1{6(r+1)}\Bigr].
\]
If \(\Lstarr\ge d/n\), let \(\widetilde\cD_\sigma\) be as in
\cref{lem:list-assouad} with parameters \((d,\bar L)\); this is
admissible since \(\bar L\ge\Lstarr\ge d/n\).  If \(\Lstarr<d/n\), use
the same family with \(\gamma=0\): mass \(\tau_0=\bar L(r+1)/d\) on
each \(x_i\) with the label uniform on the \(r+1\) candidates
(independent of \(\sigma\)), and the remaining mass on
\((x_0,y_0)\); every \(r\)-list then has side-1 risk at least
\(d\tau_0\cdot\tfrac1{r+1}=\bar L\), with equality for the same cube
tuples, which calibrates the oracle list risk without any hardness
claim, and the case \(\Lstarr=0\) is the degenerate instance
\(\tau_0=0\).  In all cases define, for
\((\sigma,v)\in[r+1]^d\times\cW\),
\[
 \cD_{\sigma,v}=(1-\nu)\,\widetilde\cD_\sigma+\nu\,\cU_v,
\tag{C.13}\label{eq:list-family}
\]
with \(\cU_v\) uniform on \(\{(z_i,v_i):i\in[D]\}\).  Because the two
supports are disjoint and the tuple
\(\bigl((g_{\sigma+1},w_v),\ldots,(g_{\sigma+r},w_v)\bigr)\), with
\(w_v\in\cW\) the member realizing \(v\), has side-2 list error zero,
\[
 \inf_{\vec h\in\cH^{r}}
 \ListRisk_{\cD_{\sigma,v}}(\Lambda_{\vec h})
 =(1-\nu)\bar L=\Lstarr
\tag{C.14}\label{eq:list-exact}
\]
exactly, for every \((\sigma,v)\); by
\cref{lem:list-assouad}(i) the infimum in \eqref{eq:list-exact} extends
to all lists of size \(r\) from any class, which is the Bayes-exactness
claim.  The excess of any \(r\)-list \(\lambda\) splits as
\[
 \ListRisk_{\cD_{\sigma,v}}(\lambda)-\Lstarr
 =(1-\nu)\bigl(\ListRisk_{\widetilde\cD_\sigma}(\lambda)-\bar L\bigr)
 +\frac{\nu}{D}\sum_{i=1}^{D}\ind\{v_i\notin\lambda(z_i)\},
\tag{C.15}\label{eq:list-split}
\]
with a nonnegative first term.

\emph{Side-2 term.}  \Cref{lem:list-ds-lower} applies with
\(\cQ=\widetilde\cD_\sigma\) for each fixed \(\sigma\), by
\cref{lem:band}(a) for \(\cW\).  Averaging over \(\sigma\) as well, and
repeating the two-regime computation of \eqref{eq:lb-side2} verbatim,
\[
 \E_{\sigma,v}\,\E
 \Biggl[\frac{\nu}{D}\sum_{i=1}^{D}
 \ind\{v_i\notin A(S)(z_i)\}\Biggr]
 \ \ge\ \frac{\nu}{2}\Bigl(1-\frac\nu D\Bigr)^{n}
 \ \ge\ \frac1{15}\min\Bigl\{\frac Dn,\frac14\Bigr\}.
\tag{C.16}\label{eq:list-side2}
\]

\emph{Side-1 term.}  Suppose \(\Lstarr\ge d/n\).  The argument of
\cref{lem:list-assouad}(ii) applies to the padded family with \(v\)
fixed: the two tested laws differ only in the two swapped atom masses,
now carrying the factor \((1-\nu)\tau\le\tau\), so
\eqref{eq:list-kl}--\eqref{eq:list-test} are unchanged, and the side-2
part of the sample is shared by the two laws and enters as ancillary
information.  The split \eqref{eq:list-split} carries the factor
\((1-\nu)\), and with
\((1-\nu)\sqrt{\bar L}\ge\sqrt{3/4}\,\sqrt{\Lstarr}\ge\tfrac56\sqrt{\Lstarr}\),
\[
 \E_{\sigma,v}\,\E
 \Bigl[(1-\nu)\bigl(\ListRisk_{\widetilde\cD_\sigma}(A(S))-\bar
 L\bigr)\Bigr]
 \ \ge\ (1-\nu)\cdot\frac1{40}\sqrt{\frac{\bar Ld}{n}}
 \ \ge\ \frac1{48}\sqrt{\frac{\Lstarr d}{n}} .
\tag{C.17}\label{eq:list-side1}
\]
For \(\Lstarr<d/n\) no side-1 claim is made and the first term of
\eqref{eq:list-lower} is zero.

\emph{Conclusion.}  Summing \eqref{eq:list-side1} and
\eqref{eq:list-side2} inside \eqref{eq:list-split}, some
\((\sigma,v)\) attains at least the average, which is
\eqref{eq:list-lower}; the family \eqref{eq:list-family} depends only
on \((\cH,\Lstarr,n)\), so the same averaging covers randomized
learners.
\end{proof}

\begin{proof}[Proof of \cref{cor:list-realizable}]
Let the domain be \(\{x_0,z_1,\ldots,z_D\}\), let \(y_0\) be a label
outside \([2r]\), and let
\(\cH=\{w_v:v\in[2r]^D\}\) with \(w_v(z_i)=v_i\) and
\(w_v(x_0)=y_0\).  Every \(v\) has exactly \(2r-1\) neighbors in
every coordinate, so \(V=[2r]^D\) witnesses \(\dDSk{r}(\cH)\ge D\);
conversely a shattered sequence is injective and cannot contain
\(x_0\), on which all members agree, so \(\dDSk{r}(\cH)=D\).  For
\(v\in[2r]^D\) set \(\nu=\min\{D/n,\tfrac14\}\) and
\(\cD_v=(1-\nu)\,\delta_{(x_0,y_0)}+\nu\,\cU_v\); each \(\cD_v\) is
realizable by \(w_v\).  \Cref{lem:list-ds-lower} applies with
\(\cQ=\delta_{(x_0,y_0)}\), and
\(\ListRisk_{\cD_v}(\lambda)\ge
\frac\nu D\sum_i\ind\{v_i\notin\lambda(z_i)\}\) pointwise, so for
\(v\) uniform every learner has
\(\E_v\,\E\bigl[\ListRisk_{\cD_v}(A(S))\bigr]
\ge\frac\nu2(1-\nu/D)^n\), which the numeric chain of
\eqref{eq:lb-side2} bounds below by
\(\frac1{15}\min\{D/n,\tfrac14\}\).  Some \(v\) attains the
average, and the family does not depend on the learner, which covers
randomized learners.
\end{proof}

\begin{remark}[Interface for the general range]
\label{rem:list-interface}
Call \(V\subseteq[K_2]^D\) an \emph{\((r,D)\)-separation family} if
every element has at least \(2r-1\) neighbors in every coordinate and
the induced class has \(r\)-Natarajan dimension at most \(1\).
Replacing \(\cW\) by such a family in the proof above extends
\cref{thm:list-lower} to every \(D\ge1\), with
\(d\le\dNk{r}(\cH)\le d+1\).  \Cref{lem:band} provides the case
\(D=2\) for every \(r\), and \cite{brukhim2022characterization} the
case \(r=1\) for every \(D\); constructing \((r,D)\)-separation
families for all \((r,D)\), plausibly by the methods of
\cite{charikar2023list}, is left open.
\end{remark}

\section{Machine-checked formalization}
\label{app:lean}

Both directions of the paper have been formalized in Lean~4 over Mathlib:
the relative compression theorem together with the three-block assembly
behind \cref{thm:main}, and the lower bound \cref{thm:lower} in full.  The
development is about $4{,}500$ lines, contains no \texttt{sorry}, declares
no axioms of its own, and every headline statement passes the kernel's
axiom audit depending only on \texttt{propext},
\texttt{Classical.choice}, and \texttt{Quot.sound}.  This section records
what is checked, in which form, and where the formal statements differ
from the prose.  The list extension of \cref{sec:list} and
\cref{app:list} postdates the formalization and is not covered by it.

\paragraph{Upper bound.}
\Cref{thm:relative-compression} is checked in full generality and in
measure-theoretic form: an arbitrary probability space, product samples
via \texttt{Measure.pi}, losses valued in $[0,1]$ (slightly more general
than $\{0,1\}$), and a possibly non-measurable bad event handled by outer
measure and a measurable cover, so no stability or measurability of the
selection is assumed beyond \cref{asm:measurable}.  The constants are the
ones displayed in \eqref{eq:relative-bound}: the theorem
\texttt{relative\_compression} concludes
$L_{\cD}(A(S))\le L_{\cD}(h)+12\sqrt{L_{\cD}(h)\Gamma}+20\Gamma$ with
$\Gamma$ literally the expression \eqref{eq:gamma}.  The one-sided
Bernoulli--Bernstein inequality (\cref{lem:bernoulli-bernstein}) is proved
from first principles rather than imported, and the selective-menu
calculus of \cref{lem:menu-calculus} is checked in integrated form.  The
assembly of \cref{thm:main} is checked as two theorems:
\texttt{rmc\_master}, the fixed-comparator bound
$L_{\cD}(\RMC)\le L_{\cD}(h^\ast)+12\sqrt{L_{\cD}(h^\ast)\Gamma}
+20\Gamma+a_1+a_2$ on the explicit three-block product sample space, and
\texttt{rmc\_main}, which replaces the fixed comparator by
$\inf_{h\in\cH}L_{\cD}(h)$ through the $\varepsilon\downarrow0$ argument
(an antitone sequence of good events and continuity from above).  The
imported modules of \cref{app:toolkit-statements} enter
\texttt{rmc\_master} as per-block hypotheses in exactly the form stated
there: the correct-region guarantee of \cref{lem:compression-cover} as the
block-1 hypothesis, the conditional menu coverage of \cref{lem:mw-menu} as
the block-2 hypothesis, and the ordered-compression structure of
\cref{lem:fixed-menu-compression,lem:menu-finalizer} as the structural
description of block 3.  What is \emph{not} formalized is the content of
the imported modules themselves and the final bookkeeping from
$(a_1,a_2,\Gamma)$ to the displayed rate \eqref{eq:main-bound}, i.e.\ the
instantiation of \cref{rem:constants}.

\paragraph{Lower bound.}
\Cref{thm:lower} is checked end to end.  Probability is formalized by
finitely supported distributions (weight functions on finite types, with
expectations as finite sums), which is all the construction of
\cref{sec:lower} uses; samples are product laws on the finite atom space,
and learners are arbitrary functions from $n$-samples to classifiers, with
a separate statement for randomized learners averaging over a finite seed
space.  The formal theorem, \texttt{lower\_bound\_main}, produces for
every $n,\dN,\dDS\ge1$, every $\Lstar\in[0,1/8]$, and every learner a
member of the explicit family
\eqref{eq:lb-family} whose excess risk is at least
\[
 \frac1{12}\,\ind\{\Lstar\ge\dN/n\}\sqrt{\frac{\Lstar\dN}{n}}
 +\frac1{15}\min\Bigl\{\frac{\dDS}{n},\frac14\Bigr\},
\]
with the same absolute constants as \eqref{eq:lower-bound}; calibration is
checked in a slightly stronger form than stated, namely
$\Lstar\le\Risk_{\cD_{\sigma,v}}(g)$ for \emph{every} classifier $g$
together with a member of $\cH$ attaining $\Lstar$ exactly.  The four
dimension claims $\dN\le\Nat(\cH)\le\dN+1$ and $\DS(\cH)=\dN+\dDS$ are
checked in both directions (shattered witnesses below, caps above, via
the dimension calculus of \cref{lem:dim-calculus}).  The only import,
\cref{lem:separation}, is isolated as an interface structure carrying
exactly the data used in the proof: a pseudo-cube on $D$ points whose
class has Natarajan dimension at most $1$.  The six-cycle of
\cref{ex:sixcycle} is verified by the kernel's decision procedure and
discharges the interface at $D=2$, so at that parameter the formal
theorem has no unproved input at all.

\paragraph{Four findings.}
The formal proof routes were chosen for economy inside the proof
assistant, and none of the divergences from the prose affects a
statement.  Four of them, however, carry content of their own.

\begin{enumerate}
\item[(1)] \emph{Calibration is Bayes-exact.}  The formal statement
proves $\Lstar\le\Risk_{\cD_{\sigma,v}}(g)$ for \emph{every} classifier
$g$, not only for members of $\cH$, with a member of $\cH$ attaining
$\Lstar$.  Thus $\Lstar$ is the Bayes risk of every member of the family
\eqref{eq:lb-family}, and $\cH$ contains a Bayes-optimal rule for each;
the two lower-bound terms are paid even by a learner told that a Bayes
rule lies in $\cH$, so the bound is in no sense an artifact of comparing
to the class.  The strengthening is free: the risk identity
\eqref{eq:lb-risk-identity} bounds every classifier pointwise
($r_i\ge\tfrac12-\gamma$), and the side-2 error is nonnegative.

\item[(2)] \emph{The posterior in \cref{lem:ds-lower} can be
eliminated.}  Formally, all members of the family are coupled through a
common atom law: an atom sequence is drawn once, and the sample seen by
the learner is a deterministic function of the atom sequence and $v$,
the $x$-part being $v$-independent by construction.  Fixing the atom
sequence and grouping $V$ by the restriction of $v$ to the seen
coordinates, the sample is constant on each group, hence so is the
learner's prediction at an unseen $z_j$; by \cref{lem:fiber} that
prediction agrees with $v_j$ on at most half of each group.  Summing
over groups gives \eqref{eq:lb-ds} as a deterministic double count: the
posterior computation of the prose proof (the conditional law of $v$ is
uniform on the cloud) disappears, and randomness survives only in
$\E|J|=D(1-\nu/D)^n$.

\item[(3)] \emph{The two-point step needs no entropy.}  The testing
step \eqref{eq:lb-kl}--\eqref{eq:lb-tv} is carried out with the
Hellinger affinity $\rho(P,Q)\defeq\sum_z\sqrt{P(z)Q(z)}$ in place of
the KL chain rule and Pinsker's inequality.  For the two swapped-pair
laws, $1-\rho\le4(1-\nu)\tau\gamma^2$ by direct computation; affinity
tensorizes exactly under products, and Cauchy--Schwarz gives
\[
 \TV\le\sqrt{1-\rho^{2n}}
 \le\sqrt{2n(1-\rho)}
 \le\sqrt{8n\tau\gamma^2},
\]
the bound of \eqref{eq:lb-tv} with the same constant.  Combined with
the finite-support Neyman--Pearson bound (for any test,
$\mathrm{err}_P+\mathrm{err}_Q\ge1-\TV(P,Q)$, one line), the chain
\eqref{eq:lb-kl}--\eqref{eq:lb-test} uses no logarithms and no imported
inequalities, so the probabilistic core of \cref{app:lower} is
self-contained in the literal sense.

\item[(4)] \emph{Counting replaces connectivity.}  The dimension caps
for the cube class follow from cardinality alone: shattering $k$ points
forces $2^k\le|\cH_{\mathrm{cube}}|=2^{\dN}$, for DS via the
pseudo-cube cardinality bound $2^k\le|V|$.  Together with the explicit
shattered witnesses this recovers
$\Nat(\cH_{\mathrm{cube}})=\DS(\cH_{\mathrm{cube}})=\dN$ and all four
dimension claims of \cref{thm:lower}.  The connectivity observation in
the prose proof, that the only pseudo-cube inside a sign cube is the
full cube, is a stronger structural statement than the argument needs,
and it is the one step of the appendix the formalization simply never
had to prove.
\end{enumerate}

\paragraph{Remaining deviations.}
The two small-$\Lstar$ calibration cases of \eqref{eq:lb-exact} are
merged into a single choice that spreads the mass $2\bar L$ uniformly
over the $\dN$ pair points with fair labels; the side-1 infimum is
again exactly $\bar L$, and the case $\Lstar=0$ is the degenerate
instance.  Randomized learners are modeled by a finite seed space.  The
numeric facts of \eqref{eq:lb-side2} and \eqref{eq:lb-side1}
($1-t\ge e^{-2t}$ on $[0,\tfrac12]$, $e^{2}\le7.5$, and the
$\sqrt{3/4}$ step, formalized as the rational certificate
$25/36\le1-\nu$) are checked as stated.

\paragraph{Audit summary.}
\Cref{tab:lean-map} lists the correspondence between the prose and the
formal development.  Everything named there is proved; the imported
modules of \cref{app:toolkit-statements} enter the assembly as
hypotheses in their stated forms, and \cref{lem:separation} is an
interface discharged at $D=2$ by \cref{ex:sixcycle}.  Every listed
declaration depends on no axioms beyond \texttt{propext},
\texttt{Classical.choice}, and \texttt{Quot.sound}, and the repository
builds with zero \texttt{sorry}.

\begin{table}[t]
  \centering
  \caption{Paper-to-Lean correspondence.  File names are relative to
  the repository root; the lower-bound files live under
  \texttt{Lower/}.}
  \label{tab:lean-map}
  \footnotesize
  \begin{tabular}{@{}>{\raggedright\arraybackslash}p{0.34\textwidth}>{\raggedright\arraybackslash}p{0.58\textwidth}@{}}
    \toprule
    Paper & Lean declaration (file) \\
    \midrule
    \cref{thm:relative-compression}; $\Gamma$ of \eqref{eq:gamma}
      & \texttt{relative\_compression}, \texttt{Gamma} (\texttt{Main.lean}) \\
    \cref{lem:bernoulli-bernstein}
      & \texttt{bernstein\_upper}, \texttt{bernstein\_lower}
        (\texttt{Bernstein.lean}) \\
    \cref{lem:menu-calculus}
      & \texttt{integral\_risk\_algebra},
        \texttt{integral\_alpha\_transfer},
        \texttt{integral\_menuLoss\_le} (\texttt{Menu.lean}) \\
    \cref{thm:main}, fixed comparator; infimum form
      & \texttt{rmc\_master}; \texttt{rmc\_main}
        (\texttt{Assembly.lean}) \\
    \midrule
    \cref{thm:lower}, deterministic; randomized
      & \texttt{lower\_bound\_main};
        \texttt{lower\_bound\_randomized} (\texttt{Main.lean}) \\
    family-summed form \eqref{eq:lb-split}
      & \texttt{lower\_bound\_core}, \texttt{lower\_bound\_avg}
        (\texttt{Main.lean}) \\
    \cref{lem:assouad-lower}
      & \texttt{assouad\_bound}; \texttt{risk\_ge},
        \texttt{risk\_gStar} (\texttt{Engines.lean},
        \texttt{Master.lean}) \\
    \eqref{eq:lb-kl}--\eqref{eq:lb-test}
      & \texttt{tv\_flip\_le}, \texttt{np\_pair}; affinity toolkit
        (\texttt{Engines.lean}, \texttt{Dist.lean}) \\
    \cref{lem:ds-lower}
      & \texttt{loo\_bound}; \texttt{cloud\_bound},
        \texttt{unseen\_prob} (\texttt{Engines.lean}) \\
    \cref{lem:fiber}; \cref{lem:dim-calculus}
      & \texttt{fiber\_lemma}; \texttt{natShatters\_sum\_le},
        \texttt{dsShatters\_sum\_le} (\texttt{Comb.lean}) \\
    dimension claims of \cref{thm:lower}
      & \texttt{bigClass\_natShatters}, \texttt{bigClass\_nat\_le},
        \texttt{bigClass\_dsShatters}, \texttt{bigClass\_ds\_le}
        (\texttt{Main.lean}) \\
    \cref{lem:separation}; \cref{ex:sixcycle}
      & \texttt{SeparationFamily}; \texttt{sixCycle}
        (\texttt{Main.lean}) \\
    \eqref{eq:lb-exact} parameters and numerics
      & \texttt{regimeA\_params}, \texttt{regimeB\_params},
        \texttt{second\_term\_bound} (\texttt{Main.lean}) \\
    axiom audit
      & \texttt{Axioms.lean} \\
    \bottomrule
  \end{tabular}
\end{table}

\paragraph{Library status.}
The development sits on Mathlib, but its learning-theoretic layer is
built from scratch: at the time of writing, Mathlib contains no
Bernstein-type inequality (the one-sided Bernoulli bounds behind
\cref{lem:bernoulli-bernstein} are self-contained and could be
contributed), no total-variation or Hellinger-affinity toolkit for
finitely supported distributions, and no shattering theory in any of
its guises (VC, Natarajan, or DS).  We are not aware of a prior
machine-checked minimax lower bound for a learning problem of this
kind, and the atom-space coupling and interface patterns described
above are offered as reusable templates for formalizing such
arguments.

\end{document}